\documentclass{article} %
\usepackage{iclr2027_conference,times}

\usepackage{amsmath,amssymb,amsfonts,mathtools}
\usepackage{amsthm}
\usepackage{bm}
\usepackage{booktabs}
\usepackage{multirow}
\usepackage{graphicx}
\usepackage{wrapfig}
\usepackage{xcolor}
\usepackage{enumitem}
\usepackage{algorithm}
\usepackage{algorithmic}
\usepackage{xspace}

\usepackage{amsmath,amsfonts,bm}

\def\eqref#1{equation~\ref{#1}}

\def\1{\bm{1}}

\DeclareMathAlphabet{\mathsfit}{\encodingdefault}{\sfdefault}{m}{sl}
\SetMathAlphabet{\mathsfit}{bold}{\encodingdefault}{\sfdefault}{bx}{n}

\usepackage{hyperref}
\usepackage{url}

\usepackage{booktabs}
\usepackage[table]{xcolor}

\definecolor{LossBlue}{HTML}{2171B5}
\definecolor{DeltaRed}{HTML}{CB181D}

\newtheorem{theorem}{Theorem}
\newtheorem{proposition}[theorem]{Proposition}

\newtheorem{corollary}[theorem]{Corollary}

\providecommand{\bulkboost}{\textsc{BulkBoost}\xspace}
\providecommand{\algostatic}{\textsc{BB-Static}\xspace}
\providecommand{\algoadaptive}{\textsc{BB-Adaptive}\xspace}
\providecommand{\freon}{\textsc{Freon}\xspace}
\providecommand{\deva}{\textsc{DeVA}\xspace}
\providecommand{\spectra}{\textsc{Spectra}\xspace}
\providecommand{\muon}{\textsc{Muon}\xspace}
\providecommand{\adamw}{\textsc{AdamW}\xspace}

\providecommand{\dmodel}{d_{\mathrm{model}}}

\providecommand{\FineMath}{FineMath-4}
\providecommand{\FineMathN}{FineMath-3}
\providecommand{\CodeParrot}{CodeParrot}
\providecommand{\WikiDe}{Wiki-de}
\providecommand{\WikiZh}{Wiki-zh}
\providecommand{\WikiAr}{Wiki-ar}

\usepackage{booktabs}
\usepackage[table]{xcolor}
\usepackage{pgf}

\definecolor{LossBlue}{HTML}{2171B5}

\newcommand{\heatcell}[3]{%
  \pgfmathparse{8 + 47*max(0,min(1,((#3)-(#1))/((#3)-(#2))))}%
  \edef\heatcolor{\noexpand\cellcolor{LossBlue!\pgfmathresult}}%
  \heatcolor #1%
}

\title{Does Muon Need Fine-Grained Spectral Shaping?}

\author{
Meher Chaitanya, Tianyi Zhou \& Aristides Gionis \\
\;KTH Royal Institute of Technology \\
\;Stockholm, Sweden \\
\;\texttt{\{mcpi,tzho,argioni\}@kth.se}
}

\iclrfinalcopy %
\begin{document}

\maketitle
\lhead{Preprint}
\begin{abstract}

\muon{} combines current and past gradients into matrix momentum. For $M=U\Sigma V^\top$, the idealized polar update $Q=UV^\top$ gives every singular direction the same weight. We refer to this as the flat profile. Several recent optimizers replace this flat profile with fine-grained spectral maps that give each direction its own gain. \emph{We ask how much of this spectral detail a \muon{} update needs}. Our spectral diagnostics show that approximately $94$--$97\%$ of measured singular modes lie below an estimated noise edge, yet collectively align positively with
a reference gradient.

We introduce \textsc{BulkBoost}, a two-band spectral reweighting framework with fixed-rank and noise-calibrated variants. The latter uses split-minibatch gradient differences to calibrate a Marchenko--Pastur reference edge for \muon{}'s Nesterov input, separating the \emph{bulk} below the edge from the \emph{spikes} above it. Both variants increase the bulk's relative weight through \emph{one shared gain} while preserving the Frobenius norm of each matrix's unreweighted direction. For a fixed partition, our theory gives the first-order condition under which moving weight toward the bulk lowers the loss. It also quantifies the fraction of the maximal first-order improvement rate, over all per-mode reallocations, that two bands can capture. A polar-plus-low-rank representation enables an approximation without a full SVD, tracking only the spike subspace. Across $30$ continued-pretraining settings spanning Pythia-$14$M to $410$M and
six corpora, two-band reweighting is competitive with the fine-grained
power-law profile of \textsc{Freon}~\citep{shumaylov2026freon} and outperforms \textsc{Spectra}~\citep{huang2026spectra}. Measured against \muon{}'s flat profile, \textsc{Freon} reduces final
loss by $0.022\%$ of the pre-adaptation loss on average, whereas the two-band variants achieve reductions of $0.073$--$0.147\%$. These observations suggest that useful departures from the flat profile are surprisingly low-dimensional: a single bulk-to-spike gain captures at least as much benefit as the fine-grained spectral profiles.
\end{abstract}

\section{Introduction}
\label{sec:intro}
 An optimizer for neural-network weight matrices decides two things: how large an update is, and how the update is spread across directions. \muon{}, an optimizer for language-model training, combines current and past gradients into a matrix momentum $M$ \citep{jordan2024muon,liu2025scalable}.
Its SVD, $M=U\Sigma V^\top$, writes $M$ as a sum of orthogonal directions weighted by their singular values.
The idealized polar update $Q=UV^\top$ keeps these directions but sets every singular value to \emph{one}. Each direction therefore receives the same gain. We refer to this as the \emph{flat profile}. Recent optimizers depart from this flat profile through spectral maps $Ug(\Sigma)V^\top$, where the gain varies across singular directions \citep{shumaylov2026freon,dong2026muonp,wu2026dynmuon,
pang2026htmuon,feoktistov2026softsign}. At the same time, \citet{shumaylov2026freon} show that substantially different spectra, including random and inverted ones, can remain effective. These results raise a more basic question than which
detailed profile is best: \emph{does \muon{} need fine-grained spectral shaping, or is a coarse reallocation between two groups of directions enough?}

Two observations motivate this question.
First, a direction's strength in the sampled momentum,
its separation from sampling noise, and the benefit of
increasing its update weight are distinct quantities.
A small singular value can still carry signal, and a
large one need not deserve more weight.
Second, our diagnostics reveal a broad spectral bulk
and a small leading band. Under a Gaussian white-noise reference model, random matrix theory provides an asymptotic upper spectral edge, the \emph{noise edge}
\citep{marchenko1967,baik2005phase,benaychgeorges2012singular}.
We call directions above this reference edge \emph{spikes}
and the remaining directions the \emph{bulk}.
On Pythia-70M \muon{} trajectories, approximately
$93.7$--$97.1\%$ of pooled singular modes lie in the bulk
(Appendix~\ref{app:band-content}). These below-edge singular modes, the \emph{bulk}, are weak individually, yet their aggregate signed alignment with an independently sampled reference gradient is positive in every analyzed matrix.
Thus, the broad spectral bulk makes a non-negligible collective
contribution even though its individual directions are weak.
These measurements motivate a two-band partition.

We introduce \textsc{BulkBoost}, whose exact target reweights \muon{}'s polar direction in two bands. The spike and bulk gains are $1$ and $\rho$ before normalization, and the update is then rescaled to keep the Frobenius norm of the unreweighted direction. At the exact-target level, $\rho=1$ recovers $Q$, while $\rho>1$ increases the bulk's relative weight. Once the partition is chosen, the same $\rho$ controls reallocation across all spectrally updated matrices. We consider two rules for the split. The \emph{fixed-rank rule}, \textsc{BB-Static}, takes
a prescribed fraction $p$ of the leading directions, shared across matrices and held fixed within each run. Our shared configuration uses $p=5\%$, within the pooled spike fraction of $2.9$--$6.3\%$ measured in our diagnostics (Appendix~\ref{app:band-content}).  We also test smaller and larger fractions. The \emph{noise-calibrated rule}, \textsc{BB-Adaptive}, estimates the sampling noise from two minibatch halves, propagates it through the momentum, and places the split at the resulting noise edge for each matrix and step. The edge assumes isotropic Gaussian noise. Since the real gradient noise is not isotropic, this split is an approximation. Our contributions are as follows.

\begin{enumerate}
    \item \textbf{A two-band reweighting rule.}
    \textsc{BulkBoost} reweights \muon{}'s flat update in two bands with one shared gain. The split between the bands comes either from a fixed fraction of directions or from a noise-calibrated edge. An exact polar-plus-low-rank identity lets us avoid a full SVD (Section~\ref{sec:method}).

    \item \textbf{What the two bands capture.}
    For a fixed partition, we characterize to first order when shifting weight
    toward the bulk lowers the loss, and how much of the best per-direction
    improvement two bands can reach (Section~\ref{sec:theory}).

    \item \textbf{Evidence from the spectrum.}
    Our spectral diagnostics establish that most singular directions of \muon{}'s momentum lie below the noise edge. This bulk is weak direction by direction, yet it carries useful signal in aggregate. Giving it more relative weight lowers the loss after a single update in most tested states (Appendices~\ref{app:band-content} and~\ref{app:post-update-loss}).

    \item \textbf{Competitive gains without fine-grained shaping.}
    In continued pretraining settings spanning Pythia-$14$M to $410$M and six corpora, two-band reweighting gives, on average, a slightly larger improvement over \muon{}'s flat profile than the fine-grained direction-specific spectral profiles we compare against. Thus, a single shared gain between spikes and bulk is enough in these settings (Section~\ref{sec:experiments}).
\end{enumerate}

Related work is provided in Appendix~\ref{app:related-work}.

\section{Setup and Spectral Reweighting}
\label{sec:preliminaries}

\textbf{\muon{} and notation.}
Consider a matrix parameter $W_t\in\mathbb R^{m\times n}$ with minibatch
gradient $G_t$ after any global gradient clipping.
Let $d=\min\{m,n\}$ and $[d]=\{1,\ldots,d\}$.
We use $\langle A,B\rangle_F=\operatorname{tr}(A^\top B)$, and $\|A\|_F$ and
$\|A\|_2$ denote the Frobenius and spectral norms.
For $\beta\in[0,1)$ and $B_0=0$, the exponential moving-average (EMA) buffer
and the Nesterov spectral input are, for $t\geq1$,
\begin{equation}
B_t=\beta B_{t-1}+(1-\beta)G_t,
\qquad
M_t=\beta B_t + (1-\beta)G_t.
\label{eq:muon-spectral-input}
\end{equation}
Let $M_t=U_t\Sigma_tV_t^\top$ be its thin singular value decomposition (SVD),
with singular values $\sigma_{t,1}\geq\cdots\geq\sigma_{t,d}$ and paired
singular vectors $u_{t,i},v_{t,i}$.
\muon{}'s idealized direction is $Q_t=\operatorname{polar}(M_t)=U_tV_t^\top$,
which gives unit gain to every singular direction
\citep{jordan2024muon,liu2025scalable}. In practice, \muon{} approximates $Q_t$ with a few Newton--Schulz iterations.
All spectral quantities refer to $M_t$. The exact formulas below assume full-rank inputs. Appendix~\ref{app:rank-conventions} extends them to rank-deficient inputs by restricting to nonzero singular modes.

\textbf{Two-band partitions.}
A partition, which we also call a split, divides the singular directions of
$M_t$ into a spike band $\mathcal I_{\mathrm{sp},t}\subseteq[d]$ and a bulk
band $\mathcal I_{\mathrm{bu},t}=[d]\setminus\mathcal I_{\mathrm{sp},t}$.
Both rules we consider put the leading directions in the spike band, so
$\mathcal I_{\mathrm{sp},t}=\{1,\ldots,r_t\}$.
The rules differ only in how they choose $r_t$.
The fixed-rank rule, \textsc{BB-Static}, sets $r_t=\lceil pd\rceil$ for a
fixed fraction $p\in(0,1)$.
The noise-calibrated rule, \textsc{BB-Adaptive}, compares the singular values
with a threshold $\tau_t$:
\begin{equation}
\mathcal I_{\mathrm{sp},t}
=\{i\in[d]:\sigma_{t,i}>\tau_t\},
\qquad
\mathcal I_{\mathrm{bu},t}
=\{i\in[d]:\sigma_{t,i}\leq\tau_t\}.
\label{eq:mp-partition}
\end{equation}
The threshold is based on the \emph{noise edge}.
For a large Gaussian white-noise matrix with entrywise standard deviation $s$, the top singular value approaches $s(\sqrt m+\sqrt n)$
\citep{marchenko1967,geman1980limit}.
Section~\ref{sec:method} calibrates this edge for the Nesterov input using two
minibatch halves.
Real gradient noise is not isotropic Gaussian, so the edge is an approximate
reference for the split.

\textbf{Norm-matched spectral reweighting.}
A nonzero profile $w_t\in[0,\infty)^d$ specifies unnormalized gains along the
singular directions of $M_t$.
Its exact norm-matched direction is
\begin{equation}
D_t(w_t) =c_t(w_t)U_t\operatorname{diag}(w_t)V_t^\top,
\qquad
c_t(w_t) =\sqrt{\frac{d}{\sum_{i=1}^{d}w_{t,i}^2}}.
\label{eq:norm-matched-profile}
\end{equation}
Thus $\|D_t(w_t)\|_F=\|Q_t\|_F=\sqrt d$, which fixes the per-matrix direction
norm before learning-rate scaling.
A scalar map $g:(0,\infty)\to[0,\infty)$ induces the profile
$w_{t,i}=g(\sigma_{t,i})$, and we write $D_t(g):=D_t(w_t)$ when this profile
is nonzero.
Fine-grained profiles, such as the power law of \textsc{Freon}
\citep{shumaylov2026freon}, have this form: each direction's gain depends on
its own singular value.
\textsc{BulkBoost} instead uses the singular values only to assign bands.
Its two-level profile has one bulk-to-spike gain $\rho\geq0$, shared by all
matrices:
\begin{equation}
w_{t,i}^{(\rho)}
=
\begin{cases}
1, & i\in\mathcal I_{\mathrm{sp},t},\\
\rho, & i\in\mathcal I_{\mathrm{bu},t},
\end{cases}
\qquad
c_{\rho,r_t}
:=c_t\!\left(w_t^{(\rho)}\right)
=\sqrt{\frac{d}{r_t+\rho^2(d-r_t)}}.
\label{eq:two-level-profile}
\end{equation}
For a nonzero profile, the applied spike and bulk gains are $c_{\rho,r_t}$ and
$\rho c_{\rho,r_t}$.
At $\rho=1$, the profile is flat and the direction is $Q_t$.
With both bands nonempty, $\rho>1$ gives the bulk more weight relative to the
spikes, and $0\leq\rho<1$ gives it less.
Because norm matching rescales both bands by $c_{\rho,r_t}$, $\rho$ sets only
the ratio of the two gains.
Multiplying the bulk gains by $\rho$ and dividing the spike gains by $\rho$
therefore give the same update.
For small spike fractions and $\rho>1$, the bulk gain stays close to one, and
most of the change comes from a smaller spike gain.
Appendix~\ref{app:spectral-setup} gives further details on normalization
identities, relative and effective gains, and conventions for rank-deficient
inputs and empty bands.
Section~\ref{sec:method} realizes this direction with a numerical polar
approximation $\widehat Q_t$.

\section{Noise-Edge Reweighted \muon: \textsc{BulkBoost}}
\label{sec:method}
This section describes how \textsc{BulkBoost} computes its update.

\subsection{Noise-Calibrated Partitioning}
\label{sec:split-edge}
Here we describe how \textsc{BB-Adaptive} places its threshold. At fixed parameters, let $G_t^{(1)},G_t^{(2)}$ be the raw mean gradients of two
independent, equal-sized minibatch halves from the same sampling distribution.
Define
\begin{equation}
G_t^{\mathrm{raw}}
=\frac{G_t^{(1)}+G_t^{(2)}}{2},
\qquad
Z_t^{\mathrm{split}}
=\frac{G_t^{(1)}-G_t^{(2)}}{2}.
\label{eq:split-gradients}
\end{equation}
The half-batch difference $Z_t^{\mathrm{split}}$ has zero mean.
After vectorization, it has the same covariance as the sampling noise in the
raw full-minibatch mean.

Let $\chi_t\in(0,1]$ be the full-minibatch clipping scalar computed across all
model parameters, with $\chi_t=1$ without clipping.
We use $G_t=\chi_tG_t^{\mathrm{raw}}$ in \eqref{eq:muon-spectral-input} and do
not clip the halves separately.
We apply $\chi_t^2$ to the raw variance estimate before smoothing:
\begin{equation}
\widehat s_t^{\,2}
=\frac{\|Z_t^{\mathrm{split}}\|_F^2}{mn},
\qquad
\overline s_t^{\,2}
=\gamma\overline s_{t-1}^{\,2}
+(1-\gamma)\chi_t^2\widehat s_t^{\,2}.
\label{eq:split-variance-estimate}
\end{equation}
For $\gamma\in[0,1)$, each matrix's EMA starts at
$\overline s_1^{\,2}=\chi_1^2\widehat s_1^{\,2}$, and the recursion applies for
$t\geq2$. The threshold must match the noise in the Nesterov input $M_t$, which differs
from the noise in the gradient or in the EMA buffer.
From \eqref{eq:muon-spectral-input}, $M_t=(1-\beta^2)G_t+\beta^2B_{t-1}$, so
the current gradient enters both directly and through the updated buffer.
Under a reference model with centered, temporally independent, isotropic input
noises of constant entrywise variance $s^2$ and $B_0=0$, the propagated noise
variance in $M_t$ is $\nu_t^2(\beta)s^2$, where
\begin{equation}
\nu_t^2(\beta)
=(1-\beta)^2\left[
(1+\beta)^2+
\frac{\beta^4(1-\beta^{2(t-1)})}{1-\beta^2}
\right],
\qquad t\geq1.
\label{eq:nesterov-noise-factor}
\end{equation}
Using its nonnegative square root $\nu_t(\beta)$, we set the threshold to the
estimated noise edge of $M_t$:
\begin{equation}
\tau_t=\widehat\tau_t
=\kappa_{\mathrm{edge}}\nu_t(\beta)
\sqrt{\overline s_t^{\,2}}(\sqrt m+\sqrt n).
\label{eq:estimated-noise-edge}
\end{equation}
with a fixed multiplier $\kappa_{\mathrm{edge}}>0$
($\kappa_{\mathrm{edge}}=1$ in all experiments) and the square root
taken after variance smoothing.
Proposition~\ref{prop:edge-calibration} proves the half-batch and
Nesterov noise formulas, and Appendix~\ref{app:method-calibration}
states the assumptions behind the threshold.

\subsection{Structure and SVD-Free Realization}
\label{sec:svd-free}
\textbf{Exact target.}
For either rule, the oracle evaluates $D_t^{\mathrm{svd}}=D_t(w_t^{(\rho)})$
using \eqref{eq:norm-matched-profile} and \eqref{eq:two-level-profile}.
For the spike singular-vector blocks $U_{t,\mathrm{sp}},V_{t,\mathrm{sp}}$,
define $P_{t,\mathrm{sp}}=U_{t,\mathrm{sp}}U_{t,\mathrm{sp}}^\top$ and
$R_{t,\mathrm{sp}}=V_{t,\mathrm{sp}}V_{t,\mathrm{sp}}^\top$.
Since $P_{t,\mathrm{sp}}Q_t=Q_tR_{t,\mathrm{sp}}
=U_{t,\mathrm{sp}}V_{t,\mathrm{sp}}^\top$,
\begin{equation}
D_t^{\mathrm{svd}}=
\begin{cases}
c_{\rho,r_t}[\rho Q_t-(\rho-1)P_{t,\mathrm{sp}}Q_t], & m\leq n,\\
c_{\rho,r_t}[\rho Q_t-(\rho-1)Q_tR_{t,\mathrm{sp}}], & n<m.
\end{cases}
\label{eq:oracle-polar-correction}
\end{equation}
The flat update $Q_t$ carries the bulk, and a correction of rank at most $r_t$
adjusts the spikes.
Theorem~\ref{thm:exact-decomposition} extends this identity to
direction-dependent spike gains.
Both projector expressions are exact for either orientation of the matrix.
We keep both cases as the tracking below works on the smaller side.

\textbf{Spike-subspace tracking.}
Both variants track the spike subspace on the smaller Gram side,
\begin{equation}
K_t=
\begin{cases}
M_tM_t^\top, & m\leq n,\\
M_t^\top M_t, & n<m.
\end{cases}
\label{eq:gram-operator}
\end{equation}
We evaluate $K_t$ through products with $M_t$ and $M_t^\top$, without forming
the Gram matrix.
\textsc{BB-Static} keeps a block of fixed width $\lceil pd\rceil$ and refines
it with $q$ block-power iterations at each step.
\textsc{BB-Adaptive} must also find how many directions lie above the
threshold.
It augments the previously retained basis $X_{t-1}$ with fresh standard
Gaussian probes $\Omega_t$ and sets $Y_0=\operatorname{orth}([X_{t-1},\Omega_t])$,
where $\operatorname{orth}$ returns an orthonormal basis of the column space.
After $q$ updates $Y_{j+1}=\operatorname{orth}(K_tY_j)$, the eigenvalues of
$Y_q^\top K_tY_q$ estimate squared singular values, and multiplying its
eigenvectors by $Y_q$ gives the Ritz directions.
We keep the Ritz directions whose eigenvalues exceed $\widehat\tau_t^{\,2}$,
in decreasing order, up to an absolute rank cap $r_{\max}$.
Reorthonormalizing them gives $X_t\in\mathbb R^{d\times\widehat r_t}$.
Finite iterations, the size of the trial space, and a binding rank cap can make
the retained rank $\widehat r_t$ differ from the exact rank $r_t$.
Appendix~\ref{app:spike-tracking} gives initialization, probe generation, and
selection details.

\textbf{Numerically norm-matched update.}
We compute a numerical polar factor
$\widehat Q_t=\operatorname{PolarExpress}(M_t)$ \citep{amsel2025polar} and form
\begin{equation}
\widetilde D_t=
\begin{cases}
\rho\widehat Q_t
-(\rho-1)X_t(X_t^\top\widehat Q_t), & m\leq n,\\
\rho\widehat Q_t
-(\rho-1)(\widehat Q_tX_t)X_t^\top, & n<m.
\end{cases}
\label{eq:svdfree-raw-direction}
\end{equation}
For nonzero $\widehat Q_t$ and $\widetilde D_t$, we set
\begin{equation}
\widehat D_t=\widehat c_{\rho,t}\widetilde D_t,
\qquad
\widehat c_{\rho,t}
=\frac{\|\widehat Q_t\|_F}{\|\widetilde D_t\|_F}.
\label{eq:numerical-norm-matching}
\end{equation}
In exact arithmetic, $\|\widehat D_t\|_F=\|\widehat Q_t\|_F$, and $\rho=1$
returns the unreweighted direction $\widehat Q_t$.
A zero $\widehat Q_t$ gives a zero update.
Otherwise, an orthonormal $X_t$ ensures a nonzero raw direction for $\rho>0$,
while $\rho=0$ requires a nonzero projected spike component.
Since $\|\widehat Q_t\|_F$ need not equal $\sqrt d$, $\widehat c_{\rho,t}$ can
differ from $c_{\rho,\widehat r_t}$.
Norm matching fixes the magnitude only, it does not correct polar or subspace
approximation errors (Appendix~\ref{app:numerical-norm-matching}).
Algorithm~\ref{alg:ser-muon} in Appendix~\ref{app:algorithm} summarizes the
realization.
Section~\ref{sec:theory} characterizes when this reallocation lowers the loss
to first order.

\section{Theory of Two-Band Reweighting}
\label{sec:theory}

We first give a first-order test for when two-band reweighting lowers
the loss, and quantify how much of the best per-direction improvement
two bands can capture.
We then give the exact structure behind the SVD-free realization.

\subsection{What Two Bands Can Capture}
\label{sec:theory-continuation}

\textbf{A controlled update comparison.}
Starting from the same training state, we vary only the applied gains
of one matrix's current update.
Fix step $t$, the singular frame $U_t,V_t$ of $M_t$, and a
continuation policy $\pi$, which specifies how training continues.
Let $\mathcal L^\pi_{t,K}(w)$ be the expected evaluation loss, given
the current state, after using $U_t\operatorname{diag}(w)V_t^\top$ as
the current direction and then following $\pi$ for $K\geq0$ more
steps.
The learning rate, the current updates of all other parameters, and
the rest of the current step are fixed.
Later gradients and optimizer states evolve under $\pi$.
At $K=0$, the loss is evaluated right after the current update.
Here $w$ denotes the applied gains, that is, the normalized gains
$c_t(w_t)w_t$ from \eqref{eq:norm-matched-profile}.
Norm matching keeps $w$ on
$\mathcal S_d=\{w\in[0,\infty)^d:\|w\|_2^2=d\}$, and the flat profile
$w=\mathbf 1_d$ gives \muon{}'s idealized update $Q_t$.

\textbf{Continuation utility.}
We take the loss to be defined on an open neighborhood of
$\mathbf 1_d$ in the positive orthant.
Set
\begin{equation}
u_i =-\left. \frac{\partial\mathcal L^\pi_{t,K}(w)}{\partial w_i}
\right|_{w=\mathbf 1_d},
\qquad
\overline u=\frac{1}{d}\sum_{i=1}^d u_i,
\qquad
u_\perp=u-\overline u\mathbf 1_d.
\label{eq:continuation-utility}
\end{equation}
The \emph{continuation utility} $u_i$ is the rate at which the loss
falls when the applied gain of direction $i$ increases and the other
gains stay fixed.
Norm matching forces compensating changes across directions, so
first-order improvements depend only on differences between
utilities, which $u_\perp$ captures.

Fix a partition with
$r=|\mathcal I_{\mathrm{sp}}|$,
$k=|\mathcal I_{\mathrm{bu}}|=d-r$,
and $1\leq r\leq d-1$, dropping the time index.
Let
$\overline u_{\mathrm{sp}}
=r^{-1}\sum_{i\in\mathcal I_{\mathrm{sp}}}u_i$
and
$\overline u_{\mathrm{bu}}
=k^{-1}\sum_{i\in\mathcal I_{\mathrm{bu}}}u_i$
be the band means, with gap
$\Delta=\overline u_{\mathrm{bu}}-\overline u_{\mathrm{sp}}$.
For this partition, varying $\rho$ gives the norm-matched two-level
gains
$w(\rho)=c_{\rho,r}
(\mathbf 1_{\mathrm{sp}}+\rho\mathbf 1_{\mathrm{bu}})$,
where $\mathbf 1_{\mathrm{sp}}$ and $\mathbf 1_{\mathrm{bu}}$ indicate
band membership and $c_{\rho,r}$ is defined in
\eqref{eq:two-level-profile}.

\begin{theorem}[Local spectral improvement]
\label{thm:continuation-decomposition}
Assume $\mathcal L^\pi_{t,K}$ is differentiable at
$\mathbf 1_d$.

\textbf{(i) Which allocation improves loss?}
Along the two-level family,
$-\left.\frac{d}{d\rho}
\mathcal L^\pi_{t,K}(w(\rho))\right|_{\rho=1}
=\frac{rk}{d}\Delta$.
Increasing $\rho$ from $1$ gives a strictly positive first-order
improvement if and only if $\Delta>0$.
If $\Delta=0$, the first-order test is inconclusive.

\textbf{(ii) How much improvement can two bands capture?}
At the flat profile, norm matching requires small gain changes to sum
to zero, so they lie in
$\mathcal T=\{\delta\in\mathbb R^d:\mathbf 1_d^\top\delta=0\}$.
If every gain may change separately, the best improvement rate per
unit change is
$\max_{\delta\in\mathcal T,\ \|\delta\|_2=1}
u^\top\delta=\|u_\perp\|_2$.
Thus the flat profile is first-order stationary in this fixed-frame,
norm-matched class if and only if $u_1=\cdots=u_d$.
If the change must be constant within each band, the best rate is
$\sqrt{rk/d}\,|\Delta|$.
For $u_\perp\neq0$, two bands therefore reach the fraction
$R_{\mathrm{bt}}
=\frac{\sqrt{rk/d}\,|\Delta|}{\|u_\perp\|_2}
\in[0,1]$ of the best rate.
Both rates are per unit Euclidean change in $w$, and
$R_{\mathrm{bt}}$ allows $\rho$ to move in either direction.
\end{theorem}

Two bands can use only the difference between the band means.
Near the flat profile, any norm-matched reweighting of the same
singular directions, however fine-grained, moves $w$ along some
$\delta\in\mathcal T$.
Its first-order rate is therefore at most $\|u_\perp\|_2$.
Finer shaping can thus add at most the share $1-R_{\mathrm{bt}}$,
which comes from utility differences within the bands.
The split rule, whether the noise edge or a fixed fraction, only
decides which directions fall in each band.
The sign of $\Delta$ and the size of $R_{\mathrm{bt}}$ depend on the
utilities.

When $\Delta>0$, small enough increases in $\rho$ lower the loss.
The theorem does not say how large $\rho$ may be, so it does not
justify a specific finite gain such as $\rho=4$.
Here, first order refers to small changes in the gains; the parameter
update itself is a finite step.
With one $\rho$ shared by all matrices, each norm-matched separately,
the rates from Theorem~\ref{thm:continuation-decomposition}(i) add up
across matrices under joint differentiability
(Appendix~\ref{app:post-update-loss}).
Appendices~\ref{app:post-update-loss} and~\ref{app:omitted-ablation}
test finite values of $\rho$.
Appendix~\ref{app:proof-continuation-decomposition} gives the proof.

\subsection{Exact Polar-Plus-Low-Rank Structure}
\label{sec:theory-decomposition}

A shared bulk gain also gives the update a simple form.
The polar factor carries the bulk, and only a correction in the spike
subspace remains.

\begin{theorem}[Exact polar-plus-low-rank decomposition]
\label{thm:exact-decomposition}
Let $M\in\mathbb R^{m\times n}$ have full rank
$d=\min\{m,n\}$, thin SVD $M=U\Sigma V^\top$,
and polar factor $Q=UV^\top$.
For $r$ selected singular directions, let
$U_{\mathrm{sp}},V_{\mathrm{sp}}$ contain their
paired vectors.
Assign gains $a\in[0,\infty)^r$ to these directions
and a common gain $\rho\geq0$ to the others.
The resulting unnormalized spectral update satisfies
\begin{equation}
\widetilde D(a,\rho)
=\rho Q
+U_{\mathrm{sp}}
\operatorname{diag}(a-\rho\mathbf 1_r)
V_{\mathrm{sp}}^\top,
\label{eq:theory-polar-decomposition}
\end{equation}
and
$\operatorname{rank}(\widetilde D(a,\rho)-\rho Q)
\leq r$.
\end{theorem}

With $a=\mathbf 1_r$, write
$\widetilde D_\rho=\widetilde D(\mathbf 1_r,\rho)$.
Normalizing it gives the projector form
\eqref{eq:oracle-polar-correction}, which needs no bulk singular
vectors.
The squared Frobenius norm of $\widetilde D(a,\rho)$ is
$\|a\|_2^2+\rho^2(d-r)$, so for $a=\mathbf 1_r$ the normalizer
$c_{\rho,r}$ restores $\|Q\|_F=\sqrt d$
(Corollary~\ref{cor:exact-normalization}).
The correction is low rank. Appendices~\ref{app:proof-exact-decomposition}
and~\ref{app:proof-exact-normalization} give the proofs.

\section{Experiments}
\label{sec:experiments}
This section tests whether two bands are enough in practice.
The diagnostics in Appendix~\ref{app:band-content} motivate this test.
On Pythia-70M \muon{} trajectories, 94--97\% of singular directions
lie below the noise edge, and this bulk is weak per direction but
positively aligned with a reference gradient in aggregate.
We compare both \bulkboost{} variants, \algostatic{} and
\algoadaptive{}, with \muon{}'s flat profile and with fine-grained
spectral profiles in continued pretraining.
We also compare the two variants with each other.
Extended results appear in Appendix~\ref{app:experiment-result}.

\begin{figure}[t]
\centering
\includegraphics[width=\linewidth]{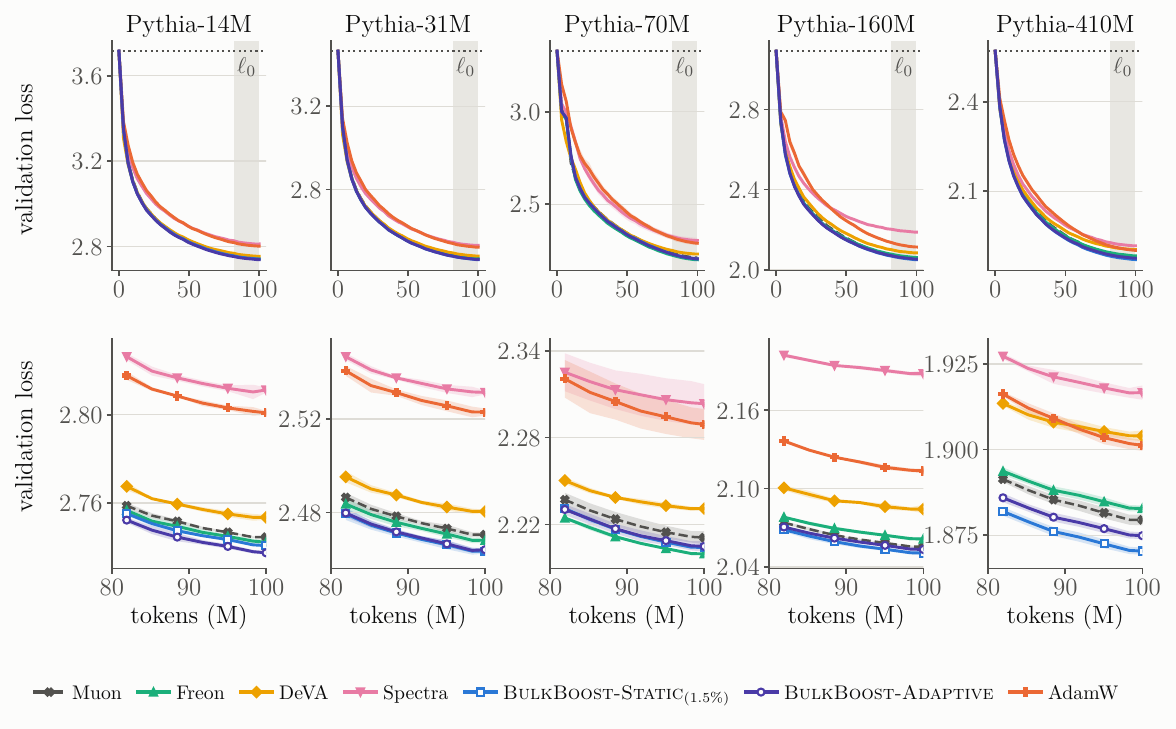}
\caption{Held-out loss versus training tokens on \WikiZh{}, one column per model scale. Top: the full run, with the pre-adaptation loss $\ell_0$ dotted and the zoomed interval shaded. Bottom: the shaded interval at a finer vertical scale. The \algostatic{} curve uses $p=1.5\%$, the leading fraction of \spectra{}, so the two can be compared directly. Appendix~\ref{app:curves} shows the other corpora.
}
\label{fig:grid-wiki_zh}
\end{figure}

\textbf{\bulkboost{} configurations.}
Unless varied in an ablation, both variants use $\rho=4$,
$\widehat Q_t=\operatorname{PolarExpress}(M_t)$~\citep{amsel2025polar},
$q=6$ warm-started block iterations for spike tracking, and the norm
matching in \eqref{eq:numerical-norm-matching}.
Like Newton--Schulz, Polar Express uses only matrix products.
It approximates the polar factor more closely, which the correction
in \eqref{eq:svdfree-raw-direction} needs.
\muon{} and both variants use Nesterov momentum $\beta=0.95$.
\algoadaptive{} splits at $\widehat\tau_t$ (Section~\ref{sec:split-edge})
with $\gamma=0.95$, $\kappa_{\mathrm{edge}}=1$, and the rank cap of
Appendix~\ref{app:spike-tracking}.
\algostatic{} keeps $r_t=\lceil pd\rceil$ spike directions, with one
fraction $p$ for all matrices and the whole run.
We test $p\in\{0.015,0.02,0.05,0.15,0.20\}$ (Table~\ref{tab:fracsweep}).
Our shared configuration, \algostatic{}$_{(5\%)}$, uses $p=5\%$,
which lies within the 2.9--6.3\% spike fraction measured in
Appendix~\ref{app:band-content}.
\algostatic{}$^{\dagger}$ instead reports, for each setting, the
fraction with the lowest mean final held-out loss.
Because this choice uses the reported loss, \algostatic{}$^{\dagger}$
shows what per-setting tuning can reach.

\textbf{Compared methods.}
We compare both \bulkboost{} variants with five baselines.
\muon{}~\citep{jordan2024muon} approximates the flat profile with
Newton--Schulz iterations and is the reference for all relative
comparisons.
Two baselines use fine-grained spectral profiles.
\freon{}~\citep{shumaylov2026freon} uses a power-law profile with
$c=0.667$; $c=\tfrac12$ recovers the flat profile.
\spectra{}~\citep{huang2026spectra} rescales the leading singular
values to the root-mean-square level of the tail.
It keeps the tail's unequal singular values, whereas \bulkboost{}
gives the whole bulk one gain.
We use its recommended leading fraction of 1.5\%. We also run \algostatic{} at $p=1.5\%$, so the two methods can be compared with the same number of spike directions (Figure~\ref{fig:grid-wiki_zh}).
The other two baselines are \adamw{}~\citep{loshchilov2017decoupled}
and \deva{}~\citep{song2026deva}, which separates variance adaptation
from scale-invariant matrix updates.

\textbf{Models and datasets.}
We continue pretraining \emph{converged} Pythia
checkpoints~\citep{biderman2023pythia} at five scales, 14M, 31M, 70M, 160M, and 410M, on six corpora: math (\FineMath{}, \FineMathN{}), Code (\CodeParrot{}), and multilingual text (\WikiDe{}, \WikiZh{},
\WikiAr{}). See Appendix~\ref{app:corpus-adaptation}.

\textbf{Token sampling and evaluation.}
Each corpus is tokenized once with Pythia's GPT-NeoX tokenizer, so all scales and methods reuse
the same tokenized data. We reserve the first $1$M tokens for held-out evaluation and train on
the remainder. Training samples uniformly random contiguous $1024$-token windows with
a run-specific seed. A run takes $3,051$ steps of $32$ sequences ($32,768$ tokens per step), or
about $100$M tokens, with a cosine learning-rate schedule. Every $100$ steps, we evaluate on $128$ deterministic, evenly spaced held-out windows, shared across steps and methods. The main comparisons use three seeds.

\textbf{Pre-adaptation loss.}
For each model--corpus pair, let $\ell_0$ and $\ell_{\mathrm{final}}$
denote the held-out losses before and after continued pretraining,
and define $\Delta=\ell_{\mathrm{final}}-\ell_0$.
All methods start from the same pretrained checkpoint and are
evaluated on the same held-out windows.
We report loss differences between methods as a percentage of~$\ell_0$.
Table~\ref{tab:datasets} in Appendix~\ref{app:corpus-adaptation}
reports $\ell_0$ and $\Delta$ for \muon{}.

\textbf{Learning-rate calibration.}
We tune one learning rate per model scale on \FineMath{} and reuse it for every method and all six corpora. To choose it, we run a short learning-rate sweep over multiples of a reference rate, which is $2\times10^{-3}$ for Pythia-14M, 31M, and 70M and $10^{-3}$ for 160M and 410M. The multipliers are $\{0.01,0.05,0.1,0.25,0.5,1,2,4\}$. Each sweep run trains one method for 20M tokens with one seed, and we record its final validation loss. The sweep covers every method except \adamw{}~\footnote{Because Pythia models are trained with \adamw{} and their learning rate decays to $0.1\times$ the base rate at convergence, we omit \adamw{} from the sweep.}. In it, \algostatic{} uses $p=20\%$, and both \bulkboost{} variants use $\rho=4$.

At each scale, we then fix one shared multiplier: the one that is best for most swept methods. It is $0.1\times$ for 14M, 31M, and 160M, $0.25\times$ for 70M, and $0.05\times$ for 410M. The resulting peak learning rates are $2\times10^{-4}$ for 14M and 31M, $5\times10^{-4}$ for 70M, $10^{-4}$ for 160M, and $5\times10^{-5}$ for 410M. For every swept method, the shared multiplier is also its own best, with two exceptions. At 70M, \muon{} does best at $0.1\times$ ($2\times10^{-4}$), with a
sweep loss of $2.779$ against $2.783$ at the shared $0.25\times$. Its loss then jumps to $2.919$ at the next multiplier, $0.5\times$, so the shared rate stays in \muon{}'s near-optimal range.
At 160M, \spectra{} does best at $0.05\times$ ($5\times10^{-5}$), with a sweep loss of $2.552$ against $2.553$ at the shared $0.1\times$. Figure~\ref{fig:lrsweep} and Table~\ref{tab:lrsweep} in
Appendix~\ref{app:omitted-ablation} report the full sweep.

\textbf{Shared rate in the 100M-token runs.} Every 100M-token run uses the shared rate of its scale, with no further tuning. This includes \adamw{}, which the sweep does not cover. It also includes all five \algostatic{} fractions in Table~\ref{tab:fracsweep}, so \algostatic{}$_{(5\%)}$ and
\algostatic{}$^{\dagger}$ use it too. Tables~\ref{tab:main} and~\ref{tab:tokensaved} and
Figure~\ref{fig:grid-wiki_zh} report these runs.

\subsection {Comparison of all optimization methods}
\label{sec:exp:arms}
Table~\ref{tab:main} reports final validation losses in all 30
model--corpus settings.
Table~\ref{tab:tokensaved} compares each method with \muon{} and also includes the shared configuration \algostatic{}$_{(5\%)}$.
Final loss decreases with model size for every method and corpus. Pre-adaptation losses $\ell_0$ follow the same trend (Table~\ref{tab:datasets}).

\begin{table}[t]
\centering\footnotesize
\setlength{\tabcolsep}{2.8pt}
\scalebox{0.9}{
\begin{tabular}{@{}llrrrrrr@{}}
\toprule
Model & Method & \FineMath{} & \FineMathN{} & \CodeParrot{} & \WikiDe{} & \WikiZh{} & \WikiAr{} \\
\midrule
\multirow{7}{*}{14M} & \adamw{} & $3.102{\scriptstyle\,\pm\,0.003}$ & $2.650{\scriptstyle\,\pm\,0.002}$ & $1.929{\scriptstyle\,\pm\,0.003}$ & $3.324{\scriptstyle\,\pm\,0.001}$ & $2.801{\scriptstyle\,\pm\,0.001}$ & $2.055{\scriptstyle\,\pm\,0.001}$ \\
 & \muon{} & $3.086{\scriptstyle\,\pm\,0.002}$ & $2.640{\scriptstyle\,\pm\,0.002}$ & $1.902{\scriptstyle\,\pm\,0.004}$ & $3.283{\scriptstyle\,\pm\,0.001}$ & $2.745{\scriptstyle\,\pm\,0.001}$ & $2.006{\scriptstyle\,\pm\,0.000}$ \\
 & \freon{} & $\mathbf{3.083}{\scriptstyle\,\pm\,0.002}$ & $\mathbf{2.638}{\scriptstyle\,\pm\,0.002}$ & $\underline{1.899}{\scriptstyle\,\pm\,0.004}$ & $3.279{\scriptstyle\,\pm\,0.001}$ & $2.742{\scriptstyle\,\pm\,0.001}$ & $2.008{\scriptstyle\,\pm\,0.001}$ \\
 & \deva{} & $3.110{\scriptstyle\,\pm\,0.004}$ & $2.661{\scriptstyle\,\pm\,0.002}$ & $1.917{\scriptstyle\,\pm\,0.005}$ & $3.286{\scriptstyle\,\pm\,0.001}$ & $2.753{\scriptstyle\,\pm\,0.001}$ & $2.013{\scriptstyle\,\pm\,0.001}$ \\
 & \spectra{} & $3.138{\scriptstyle\,\pm\,0.004}$ & $2.682{\scriptstyle\,\pm\,0.002}$ & $1.955{\scriptstyle\,\pm\,0.004}$ & $3.328{\scriptstyle\,\pm\,0.003}$ & $2.811{\scriptstyle\,\pm\,0.001}$ & $2.058{\scriptstyle\,\pm\,0.002}$ \\
 & \algostatic{}$^{\dagger}$ & $\mathbf{3.083}{\scriptstyle\,\pm\,0.002}$ & $\mathbf{2.638}{\scriptstyle\,\pm\,0.002}$ & $\mathbf{1.898}{\scriptstyle\,\pm\,0.003}$ & $\underline{3.278}{\scriptstyle\,\pm\,0.001}$ & $\underline{2.738}{\scriptstyle\,\pm\,0.001}$ & $\mathbf{1.998}{\scriptstyle\,\pm\,0.001}$ \\
 & \algoadaptive{} & $\underline{3.085}{\scriptstyle\,\pm\,0.002}$ & $\underline{2.639}{\scriptstyle\,\pm\,0.001}$ & $\underline{1.899}{\scriptstyle\,\pm\,0.003}$ & $\mathbf{3.277}{\scriptstyle\,\pm\,0.002}$ & $\mathbf{2.737}{\scriptstyle\,\pm\,0.001}$ & $\underline{2.000}{\scriptstyle\,\pm\,0.001}$ \\
\midrule
\multirow{7}{*}{31M} & \adamw{} & $2.794{\scriptstyle\,\pm\,0.001}$ & $2.443{\scriptstyle\,\pm\,0.002}$ & $1.674{\scriptstyle\,\pm\,0.002}$ & $2.998{\scriptstyle\,\pm\,0.001}$ & $2.523{\scriptstyle\,\pm\,0.002}$ & $1.897{\scriptstyle\,\pm\,0.000}$ \\
 & \muon{} & $2.776{\scriptstyle\,\pm\,0.001}$ & $2.431{\scriptstyle\,\pm\,0.002}$ & $1.649{\scriptstyle\,\pm\,0.001}$ & $2.959{\scriptstyle\,\pm\,0.002}$ & $2.471{\scriptstyle\,\pm\,0.001}$ & $1.855{\scriptstyle\,\pm\,0.001}$ \\
 & \freon{} & $\mathbf{2.771}{\scriptstyle\,\pm\,0.001}$ & $\underline{2.429}{\scriptstyle\,\pm\,0.001}$ & $\underline{1.648}{\scriptstyle\,\pm\,0.002}$ & $\underline{2.956}{\scriptstyle\,\pm\,0.001}$ & $2.468{\scriptstyle\,\pm\,0.001}$ & $1.861{\scriptstyle\,\pm\,0.001}$ \\
 & \deva{} & $2.800{\scriptstyle\,\pm\,0.001}$ & $2.462{\scriptstyle\,\pm\,0.003}$ & $1.668{\scriptstyle\,\pm\,0.003}$ & $2.966{\scriptstyle\,\pm\,0.003}$ & $2.480{\scriptstyle\,\pm\,0.001}$ & $1.862{\scriptstyle\,\pm\,0.002}$ \\
 & \spectra{} & $2.830{\scriptstyle\,\pm\,0.002}$ & $2.482{\scriptstyle\,\pm\,0.004}$ & $1.697{\scriptstyle\,\pm\,0.003}$ & $3.009{\scriptstyle\,\pm\,0.005}$ & $2.531{\scriptstyle\,\pm\,0.001}$ & $1.899{\scriptstyle\,\pm\,0.002}$ \\
 & \algostatic{}$^{\dagger}$ & $\underline{2.772}{\scriptstyle\,\pm\,0.001}$ & $\mathbf{2.428}{\scriptstyle\,\pm\,0.002}$ & $\underline{1.648}{\scriptstyle\,\pm\,0.002}$ & $\mathbf{2.953}{\scriptstyle\,\pm\,0.001}$ & $\mathbf{2.463}{\scriptstyle\,\pm\,0.001}$ & $\mathbf{1.848}{\scriptstyle\,\pm\,0.001}$ \\
 & \algoadaptive{} & $2.773{\scriptstyle\,\pm\,0.001}$ & $2.430{\scriptstyle\,\pm\,0.003}$ & $\mathbf{1.647}{\scriptstyle\,\pm\,0.003}$ & $\mathbf{2.953}{\scriptstyle\,\pm\,0.001}$ & $\underline{2.464}{\scriptstyle\,\pm\,0.001}$ & $\underline{1.849}{\scriptstyle\,\pm\,0.000}$ \\
\midrule
\multirow{7}{*}{70M} & \adamw{} & $2.617{\scriptstyle\,\pm\,0.002}$ & $2.343{\scriptstyle\,\pm\,0.004}$ & $1.547{\scriptstyle\,\pm\,0.005}$ & $2.738{\scriptstyle\,\pm\,0.007}$ & $2.289{\scriptstyle\,\pm\,0.011}$ & $1.838{\scriptstyle\,\pm\,0.014}$ \\
 & \muon{} & $2.551{\scriptstyle\,\pm\,0.002}$ & $2.297{\scriptstyle\,\pm\,0.004}$ & $1.484{\scriptstyle\,\pm\,0.003}$ & $2.680{\scriptstyle\,\pm\,0.003}$ & $2.211{\scriptstyle\,\pm\,0.005}$ & $1.702{\scriptstyle\,\pm\,0.000}$ \\
 & \freon{} & $\mathbf{2.536}{\scriptstyle\,\pm\,0.001}$ & $\underline{2.288}{\scriptstyle\,\pm\,0.006}$ & $\underline{1.471}{\scriptstyle\,\pm\,0.003}$ & $\mathbf{2.668}{\scriptstyle\,\pm\,0.002}$ & $\mathbf{2.200}{\scriptstyle\,\pm\,0.000}$ & $\underline{1.694}{\scriptstyle\,\pm\,0.001}$ \\
 & \deva{} & $2.565{\scriptstyle\,\pm\,0.001}$ & $2.310{\scriptstyle\,\pm\,0.002}$ & $1.493{\scriptstyle\,\pm\,0.004}$ & $2.707{\scriptstyle\,\pm\,0.001}$ & $2.231{\scriptstyle\,\pm\,0.001}$ & $1.723{\scriptstyle\,\pm\,0.003}$ \\
 & \spectra{} & $2.617{\scriptstyle\,\pm\,0.003}$ & $2.339{\scriptstyle\,\pm\,0.001}$ & $1.545{\scriptstyle\,\pm\,0.001}$ & $2.768{\scriptstyle\,\pm\,0.012}$ & $2.303{\scriptstyle\,\pm\,0.014}$ & $1.813{\scriptstyle\,\pm\,0.021}$ \\
 & \algostatic{}$^{\dagger}$ & $\underline{2.537}{\scriptstyle\,\pm\,0.004}$ & $\mathbf{2.287}{\scriptstyle\,\pm\,0.007}$ & $\mathbf{1.469}{\scriptstyle\,\pm\,0.004}$ & $\underline{2.670}{\scriptstyle\,\pm\,0.002}$ & $\underline{2.203}{\scriptstyle\,\pm\,0.003}$ & $\mathbf{1.690}{\scriptstyle\,\pm\,0.001}$ \\
 & \algoadaptive{} & $2.545{\scriptstyle\,\pm\,0.001}$ & $2.291{\scriptstyle\,\pm\,0.003}$ & $1.473{\scriptstyle\,\pm\,0.003}$ & $2.674{\scriptstyle\,\pm\,0.002}$ & $2.205{\scriptstyle\,\pm\,0.004}$ & $\mathbf{1.690}{\scriptstyle\,\pm\,0.000}$ \\
\midrule
\multirow{7}{*}{160M} & \adamw{} & $2.292{\scriptstyle\,\pm\,0.002}$ & $\underline{2.072}{\scriptstyle\,\pm\,0.001}$ & $1.285{\scriptstyle\,\pm\,0.003}$ & $2.481{\scriptstyle\,\pm\,0.002}$ & $2.114{\scriptstyle\,\pm\,0.002}$ & $1.664{\scriptstyle\,\pm\,0.001}$ \\
 & \muon{} & $\mathbf{2.272}{\scriptstyle\,\pm\,0.003}$ & $\mathbf{2.066}{\scriptstyle\,\pm\,0.002}$ & $\mathbf{1.270}{\scriptstyle\,\pm\,0.003}$ & $\mathbf{2.452}{\scriptstyle\,\pm\,0.002}$ & $2.055{\scriptstyle\,\pm\,0.001}$ & $1.608{\scriptstyle\,\pm\,0.000}$ \\
 & \freon{} & $2.282{\scriptstyle\,\pm\,0.001}$ & $2.077{\scriptstyle\,\pm\,0.002}$ & $1.282{\scriptstyle\,\pm\,0.003}$ & $2.461{\scriptstyle\,\pm\,0.002}$ & $2.061{\scriptstyle\,\pm\,0.001}$ & $1.613{\scriptstyle\,\pm\,0.001}$ \\
 & \deva{} & $2.297{\scriptstyle\,\pm\,0.003}$ & $2.092{\scriptstyle\,\pm\,0.002}$ & $1.287{\scriptstyle\,\pm\,0.004}$ & $2.472{\scriptstyle\,\pm\,0.002}$ & $2.084{\scriptstyle\,\pm\,0.003}$ & $1.639{\scriptstyle\,\pm\,0.001}$ \\
 & \spectra{} & $2.353{\scriptstyle\,\pm\,0.003}$ & $2.125{\scriptstyle\,\pm\,0.006}$ & $1.325{\scriptstyle\,\pm\,0.004}$ & $2.590{\scriptstyle\,\pm\,0.003}$ & $2.188{\scriptstyle\,\pm\,0.002}$ & $1.710{\scriptstyle\,\pm\,0.001}$ \\
 & \algostatic{}$^{\dagger}$ & $\underline{2.281}{\scriptstyle\,\pm\,0.003}$ & $2.077{\scriptstyle\,\pm\,0.003}$ & $\underline{1.280}{\scriptstyle\,\pm\,0.004}$ & $\underline{2.457}{\scriptstyle\,\pm\,0.003}$ & $\mathbf{2.051}{\scriptstyle\,\pm\,0.001}$ & $\mathbf{1.597}{\scriptstyle\,\pm\,0.002}$ \\
 & \algoadaptive{} & $\underline{2.281}{\scriptstyle\,\pm\,0.004}$ & $2.077{\scriptstyle\,\pm\,0.002}$ & $1.281{\scriptstyle\,\pm\,0.004}$ & $2.460{\scriptstyle\,\pm\,0.002}$ & $\underline{2.053}{\scriptstyle\,\pm\,0.001}$ & $\underline{1.604}{\scriptstyle\,\pm\,0.001}$ \\
\midrule
\multirow{7}{*}{410M} & \adamw{} & $2.058{\scriptstyle\,\pm\,0.005}$ & $1.895{\scriptstyle\,\pm\,0.002}$ & $\underline{1.119}{\scriptstyle\,\pm\,0.001}$ & $2.198{\scriptstyle\,\pm\,0.001}$ & $1.901{\scriptstyle\,\pm\,0.002}$ & $1.469{\scriptstyle\,\pm\,0.000}$ \\
 & \muon{} & $2.046{\scriptstyle\,\pm\,0.005}$ & $1.891{\scriptstyle\,\pm\,0.002}$ & $\mathbf{1.111}{\scriptstyle\,\pm\,0.002}$ & $2.181{\scriptstyle\,\pm\,0.001}$ & $1.879{\scriptstyle\,\pm\,0.002}$ & $\underline{1.452}{\scriptstyle\,\pm\,0.001}$ \\
 & \freon{} & $\underline{2.043}{\scriptstyle\,\pm\,0.005}$ & $\mathbf{1.889}{\scriptstyle\,\pm\,0.001}$ & $\mathbf{1.111}{\scriptstyle\,\pm\,0.002}$ & $2.179{\scriptstyle\,\pm\,0.001}$ & $1.883{\scriptstyle\,\pm\,0.001}$ & $1.460{\scriptstyle\,\pm\,0.001}$ \\
 & \deva{} & $2.087{\scriptstyle\,\pm\,0.005}$ & $1.929{\scriptstyle\,\pm\,0.004}$ & $1.143{\scriptstyle\,\pm\,0.002}$ & $2.209{\scriptstyle\,\pm\,0.003}$ & $1.904{\scriptstyle\,\pm\,0.002}$ & $1.467{\scriptstyle\,\pm\,0.001}$ \\
 & \spectra{} & $2.099{\scriptstyle\,\pm\,0.004}$ & $1.933{\scriptstyle\,\pm\,0.006}$ & $1.151{\scriptstyle\,\pm\,0.001}$ & $2.224{\scriptstyle\,\pm\,0.002}$ & $1.917{\scriptstyle\,\pm\,0.002}$ & $1.472{\scriptstyle\,\pm\,0.001}$ \\
 & \algostatic{}$^{\dagger}$ & $\underline{2.043}{\scriptstyle\,\pm\,0.005}$ & $\underline{1.890}{\scriptstyle\,\pm\,0.002}$ & $\mathbf{1.111}{\scriptstyle\,\pm\,0.002}$ & $\mathbf{2.174}{\scriptstyle\,\pm\,0.001}$ & $\mathbf{1.870}{\scriptstyle\,\pm\,0.001}$ & $\mathbf{1.446}{\scriptstyle\,\pm\,0.001}$ \\
 & \algoadaptive{} & $\mathbf{2.042}{\scriptstyle\,\pm\,0.005}$ & $\underline{1.890}{\scriptstyle\,\pm\,0.002}$ & $\mathbf{1.111}{\scriptstyle\,\pm\,0.002}$ & $\underline{2.176}{\scriptstyle\,\pm\,0.001}$ & $\underline{1.875}{\scriptstyle\,\pm\,0.001}$ & $1.453{\scriptstyle\,\pm\,0.002}$ \\
\bottomrule
\end{tabular}
}

\caption{Final validation loss after about 100M tokens, reported as
mean $\pm$ standard deviation over three seeds.
All runs use the shared learning rate of their scale.
For each model--corpus setting, \algostatic{}$^{\dagger}$ reports the
fraction with the lowest mean loss among the five in
Table~\ref{tab:fracsweep}.
\textbf{Bold} and \underline{Underline} mark the lowest and second-lowest means, respectively in each setting, including ties. Appendix~\ref{app:corpus-adaptation} reports the pre-adaptation losses
$\ell_0$.}
\label{tab:main}
\end{table}

\begin{table}[t]
\centering\small
\begin{tabular}{@{}lccccccc@{}}
\toprule
 & \multicolumn{2}{c}{vs.\ \muon{}} & \multicolumn{2}{c}{best of all} & $\Delta$ vs.\ \muon{} & \multicolumn{2}{c}{tokens saved (\%)} \\
\cmidrule(lr){2-3}\cmidrule(lr){4-5}\cmidrule(lr){7-8}
Method & wins & tied & wins & tied & (pp of $\ell_0$) & mean & median \\
\midrule
\muon{} (reference) & -- & -- & 4 & 1 & -- & -- & -- \\
\adamw{} & 0 & 0 & 0 & 0 & $+1.299 \pm 0.010$ & -- & -- \\
\freon{} & 19 & 1 & $\underline{5}$ & 3 & $-0.022 \pm 0.010$ & $7.78$ & $7.71$ \\
\deva{} & 0 & 0 & 0 & 0 & $+0.811 \pm 0.015$ & -- & -- \\
\spectra{} & 0 & 0 & 0 & 0 & $+2.321 \pm 0.047$ & -- & -- \\
\algostatic{}$_{(5\%)}$ & 23 & 2 & 0 & 3 & $-0.073 \pm 0.010$ & $7.81$ & $7.80$ \\
\algoadaptive{} & $\underline{24}$ & 1 & 4 & 3 & $\underline{-0.077} \pm 0.010$ & $\underline{7.99}$ & $\underline{8.01}$ \\
\algostatic{}$^{\dagger}$ & $\mathbf{25}$ & 1 & $\mathbf{12}$ & 5 & $\mathbf{-0.147} \pm 0.021$ & $\mathbf{10.00}$ & $\mathbf{10.28}$ \\
\bottomrule
\end{tabular}

\caption{Summary over the 30 model--corpus settings after about 100M
tokens.
\algostatic{}$^{\dagger}$ uses the fraction selection of
Table~\ref{tab:main}.
Counts compare three-seed mean final losses.
 \emph{vs.\ \muon{}}: settings with lower (\emph{wins}) or equal
(\emph{tied}) loss than \muon{}.
\emph{Best of all}: settings with the lowest loss of all rows, alone
(\emph{wins}) or shared (\emph{tied}).
\algostatic{}$_{(5\%)}$ can at best tie \algostatic{}$^{\dagger}$,
which includes its runs.
\emph{$\Delta$ vs.\ \muon{}}: the loss difference from \muon{} as a
percentage of $\ell_0$, averaged over settings within each seed.
We report the mean $\pm$ std dev of the three seed averages. Negative values favor the method.
\emph{Tokens saved}: the mean and median share of the 100M-token
budget saved in reaching \muon{}'s final loss, over the winning
settings and the seeds that end below \muon{}
(Table~\ref{tab:tokensaved-detail}).}
\label{tab:tokensaved}
\end{table}

\textbf{Two bands match or beat fine-grained profiles.}
Measured against \muon{}'s flat profile, \algostatic{}$_{(5\%)}$,
\algoadaptive{}, and \algostatic{}$^{\dagger}$ lower the final loss in $23$, $24$, and $25$ of the $30$ settings, by $0.073\%$, $0.077\%$, and $0.147\%$ of $\ell_0$ on average (Table~\ref{tab:tokensaved}). The fine-grained \freon{} does so in $19$ settings, by $0.022\%$.
 \algostatic{}$_{(5\%)}$ beats \freon{} in $13$ settings and loses in $13$, and \algoadaptive{} beats it in $17$ and loses in $10$. \spectra{} and \deva{} trail \muon{} in all $30$ settings, and every \bulkboost{} configuration beats both. This includes \algostatic{} at \spectra{}'s leading fraction of $1.5\%$
(Table~\ref{tab:fracsweep}). The original \deva{} and \spectra{} studies pretrain from scratch~\citep{song2026deva,huang2026spectra}, which differs from our continued pretraining setup. These results suggest that a single bulk-to-spike gain captures at least as much benefit as the fine-grained spectral profiles.

\textbf{Static and adaptive splits.}
\algostatic{}$^{\dagger}$ beats \algoadaptive{} in $20$ settings and
loses in $4$, but it selects its fraction on the reported loss.
Without any fraction search, \algoadaptive{} matches
\algostatic{}$_{(5\%)}$, with $10$ wins, $6$ losses, and $14$ ties.

\textbf{Token efficiency and runtime trade-offs.}
Among settings where it outperforms \muon{},
\algoadaptive{} reaches \muon{}'s final loss using
$7.99\%$ fewer tokens on average, with a median
saving of $8.01\%$ (Table~\ref{tab:tokensaved}).
The extensively tuned \algostatic{}$^{\dagger}$ increases
these savings to $10\%$ and $10.28\%$,
respectively.
These token savings come with additional per-step computation.
At 70M--410M, measured per-step runtimes are
$1.36$--$1.50\times$ \muon{}'s for \algoadaptive{},
compared with $1.19$--$1.27\times$ for the measured
static configurations
(Appendix~\ref{app:experiment-result}).
Adaptive selection avoids fraction search, while static
selection offers cheaper execution once a suitable
fraction is identified.
Neither \deva{} nor \spectra{} achieves lower mean final
validation loss than \muon{} in any of the $30$ settings
(Table~\ref{tab:main}); their token-saving averages over
wins are therefore undefined
(``--'' in Table~\ref{tab:tokensaved}).
\deva{} outperforms \adamw{} in $18$ settings and
trails it in $12$. Table~\ref{tab:tokensaved-detail} in
Appendix~\ref{app:omitted-ablation} identifies the winning
model--corpus settings and reports their token savings
and variation across seeds.

\textbf{Ablations.}
Appendix~\ref{app:omitted-ablation} compares the static and adaptive
splits, sweeps the gain $\rho$, and shows that a third band does not
improve on two.  In exact-SVD single-update tests on Pythia-70M, $\rho=1.25$ and $1.5$ lower the loss relative to the flat update ($\rho=1$) in five of six settings (Appendix~\ref{app:post-update-loss}).

\textbf{Conclusion.} Our diagnostics reveal that \muon{}'s momentum has a broad spectral
bulk whose individual directions are weak but collectively aligned with a reference gradient. \bulkboost{} exploits this structure with only two bands and a single shared bulk-to-spike gain.
Our theory shows when this coarse reallocation is locally beneficial
and what finer shaping can add through within-band variation.
An exact polar-plus-low-rank identity enables an SVD-free realization. Across $30$ continued-pretraining settings, the two-band variants match or exceed the finer spectral baselines in aggregate. Single-update experiments show that bulk emphasis lowers the loss relative to the exact flat update in five of six states. Together, these results suggest that useful departures from \muon{}'s flat spectrum can be surprisingly low-dimensional: \emph{a single bulk-to-spike gain can capture at least
as much benefit as the fine-grained spectral profiles}.

\textbf{Limitations.} The strongest \algostatic{} results select the fraction separately for each setting. A shared fraction avoids this search, but the best fraction varies across settings (Table~\ref{tab:fracsweep}). \algoadaptive{} instead sets the split online from the noise edge. It needs no fraction search, but it costs more per step. The noise edge comes from a white Gaussian noise reference model, so it separates signal from noise only approximately. Our evaluation is restricted to Pythia continued pretraining under the reported tuning protocol. Extending the evaluation to other architectures and training regimes remains future work.

\section*{AI Use Statement}
Generative AI tools were used during the development of this manuscript. Their use included drafting and debugging portions of experimental code, assisting with the interpretation and presentation of experimental results, and drafting and editing portions of the manuscript. All implementations, experimental protocols, and numerical results included in the final paper were reviewed and verified by the authors. Experiments reported in the paper were executed by the authors on the stated models and datasets. The authors take full responsibility for the correctness and originality of the paper and for parts of the code produced with the assistance of generative AI.

\bibliography{iclr2027_conference.bib}
\bibliographystyle{iclr2027_conference}

\clearpage
\appendix

\section{Related Work}
\label{app:related-work}

We position \textsc{BulkBoost} within matrix-aware optimization,
spectral reweighting, and subspace selection.
Our focus is the allocation of a fixed per-matrix direction
norm between two spectral bands, separating partition selection
--by fixed rank or a noise-calibrated edge--from the relative
gain assigned to the bulk.

\subsection{Matrix-Aware Optimization and Spectral Geometry}

Stochastic spectral descent exploits matrix geometry rather
than treating a parameter matrix as an unstructured vector
\citep{carlson2015stochastic}.
\citet{bernstein2024old} organize several optimizer
constructions through norm-constrained steepest descent.
\muon{} applies an approximate polar map to matrix momentum
\citep{jordan2024muon}; its large-scale realization addresses
update scaling, weight decay, and distributed execution for
language-model training \citep{liu2025scalable}.
Scion develops a broader family of norm-constrained linear
minimization oracles \citep{pethick2025training}, while
PolarGrad places polar-based updates within a matrix-aware
preconditioning framework \citep{lau2025polargrad}.
These perspectives distinguish the geometry defining an update
from its scalar normalization.
Within this setting, \textsc{BulkBoost}'s exact target retains
the singular frame of $M_t$ and reallocates gains between
two bands while preserving the Frobenius norm of its polar
factor $Q_t$.

A related line uses accumulated gradient statistics for
matrix preconditioning.
Shampoo constructs tensor-structured preconditioners
\citep{gupta2018shampoo}; SOAP performs Adam-style adaptation
in Shampoo's eigenbasis \citep{vyas2024soap}.
DeVA separates variance adaptation from scale-invariant
vector and matrix update rules \citep{song2026deva}.
COSMOS combines SOAP-style adaptation in a leading
eigensubspace with \muon{} on its complement
\citep{liu2025cosmos}, illustrating the use of distinct
optimization rules in different spectral subspaces.
\textsc{BulkBoost} instead applies one norm-matched
reweighting rule with a common gain within each band.
Its noise-calibrated variant uses split-minibatch statistics
to select band membership, not to assign direction-specific
second-moment gains; its fixed-rank variant requires no
noise estimate.

\subsection{Spectral Shaping Beyond \muon{}'s Flat Profile}

Several recent studies replace \muon{}'s equal-gain profile
with a nonconstant spectral map.
\citet{shumaylov2026freon} show that random and inverted
spectra can remain effective and introduce Freon through
a family of Schatten-based update rules.
\muon{}$^p$ uses fractional spectral powers
\citep{dong2026muonp}, while DynMuon varies the shaping
profile over training \citep{wu2026dynmuon}.
HTMuon uses heavy-tailed spectral correction
\citep{pang2026htmuon}, and Softsign introduces smooth
alternatives to sign updates, including the matrix optimizer
SoftMuon \citep{feoktistov2026softsign}.
These studies show that exact spectral flattening is
one effective allocation rather than a necessary choice.

Spectral clipping provides another alternative.
The \textsc{Spectra} framework of \citet{jiang2026spectralclipping}
studies clipping before and after preconditioning, with
efficient soft spectral clipping via Newton--Schulz
iterations. \textsc{Musec} replaces full flattening with momentum spectral clipping and implements smooth saturation using coupled
iterations \citep{liu2026musec}. These approaches retain or modify variation along the spectrum. \textsc{BulkBoost} instead asks whether useful reallocation
can be captured by a norm-matched two-band profile,
\emph{without} fine-grained spectral shaping.
Its exact target assigns equal gains within each band,
with fixed-rank and noise-calibrated rules determining
band membership separately from the relative
bulk-to-spike gain.

\subsection{Head--Bulk Allocation and the Closest Related Methods}

The Spectra optimizer of \citet{huang2026spectra} is distinct
from the similarly named \textsc{Spectra} framework of
\citet{jiang2026spectralclipping}.
It tracks a leading momentum subspace, replaces its singular
values with a tail-scale estimate, and retains the residual
tail before RMS-normalizing the update.
Its leading-subspace intervention and warm-started iteration
are closely related to our construction.
The central difference is the reference spectrum:
Spectra retains the raw tail's unequal singular magnitudes,
whereas \textsc{BulkBoost}'s exact target reweights the
polar factor $Q_t$ with equal gains within each band.
Both fixed-rank and noise-calibrated \textsc{BulkBoost}
use this two-level rule, differing in band selection.
Under norm matching, relative head suppression and relative
bulk emphasis are not opposite reallocations.
The meaningful distinctions are the underlying spectral
profile, partition rule, and normalization, rather than
the terminology of suppression versus amplification.

\subsection{Random-Matrix Noise Edges and Spectral Estimation}

The Marchenko--Pastur law describes the limiting spectrum
of appropriate random covariance matrices
\citep{marchenko1967}.
Spiked-model theory characterizes transitions in leading
empirical eigenvalues \citep{baik2005phase}, while rectangular
perturbation theory analyzes singular values and
singular-vector alignment
\citep{benaychgeorges2012singular}.
Optimal singular-value shrinkage studies reconstruction
of an underlying matrix under specified noise and loss
models \citep{gavish2017optimal}.
These results provide statistical reference scales and
estimation criteria, but do not determine which spectral
allocation best reduces evaluation loss after an optimizer
update.

Gradient-noise measurements have also been used to
characterize useful batch sizes and training efficiency
\citep{mccandlish2018empirical}.
The noise-calibrated variant of \textsc{BulkBoost} uses
split-minibatch differences to estimate the raw
full-minibatch noise variance.
It rescales this estimate by the squared global clipping
factor before smoothing and uses the reference-model
Nesterov correction to calibrate an edge for $M_t$
(Section~\ref{sec:split-edge}).
The resulting white-noise reference selects band membership;
the relative gain $\rho$ remains a separate optimization
choice.
The edge is neither an exact signal--noise separator nor
a prescription to discard below-edge modes.
The fixed-rank variant selects its partition without
estimating a noise edge.

\subsection{Polar Computation, Subspace Tracking, and Low-Rank Updates}

Iterative polar decomposition has a classical numerical
foundation \citep{higham1986polar}.
Polar Express develops optimized matrix-sign iterations
and their use in \muon{} \citep{amsel2025polar}.
Randomized range finding and subspace iteration provide
standard tools for estimating leading singular spaces
without a full decomposition \citep{halko2011finding}.
Spectral scaling measurements further highlight the role
of finite-iteration accuracy: momentum singular-value
quantiles vary with layer and model scale, affecting which
tail directions Newton--Schulz iterations resolve
\citep{magakyan2026spectral}.
Our realization combines \muon{}'s numerical polar
approximation $\widehat Q_t$ with warm-started block
iteration to approximate the selected spike subspace. The noise-calibrated tracker augments the retained basis
with fresh probes, whereas fixed-rank tracking uses a
block of prescribed width.
Neither construction explicitly computes individual
bulk singular vectors.  The shared bulk gain determines how these numerical
ingredients combine into the prescribed reweighted target.

Low-rank optimization methods address related computational
objectives.
GaLore projects gradients to reduce optimizer-state memory
\citep{zhao2024galore}, while Dion uses distributed,
amortized low-rank orthonormalized updates with error
feedback \citep{ahn2025dion}.
Dion3 combines algorithmic and systems techniques for
efficient orthogonal updates \citep{amsel2026dion3}.
In \textsc{BulkBoost}, the low-rank object is instead a
correction to a scaled polar update
(Theorem~\ref{thm:exact-decomposition}).
For full-rank $M_t$ and $\rho>0$, the exact two-level target
remains full rank; only its correction is confined to the
selected spike subspace.
The numerical realization matches $\|\widehat Q_t\|_F$,
rather than imposing the exact-polar norm $\sqrt d$.
Its tracking cost depends on the trial-space width and
iteration count as well as the retained rank, so avoiding
an explicit SVD does not by itself establish a runtime
advantage.

\subsection{Curvature, Noise Geometry, and Continuation Utility}

Pre-update gradient alignment captures only part of update
quality.
\citet{wang2026curvature} study \muon{}'s advantage through
directional curvature, distinguishing first-order gain from
the second-order penalty.
Quadratic Spectral Descent incorporates a quadratic surrogate
under a spectral-norm constraint
\citep{zhang2026quadratic}, while
\citet{beneventano2026spectral} analyze spectral dynamics
and noise geometry under \muon{}.
These studies complement spectral-allocation comparisons
by examining how curvature and transformed noise affect
optimization, beyond singular magnitudes alone.

Differentiating evaluation loss through training has a
history in hyperparameter optimization and bilevel learning
\citep{maclaurin2015gradient,franceschi2018bilevel}.
Our continuation utility applies this sensitivity viewpoint
to a fixed-frame spectral intervention: it is the negative
derivative of post-update or later evaluation loss with
respect to a current applied singular-direction gain,
evaluated at \muon{}'s equal-gain allocation.
Subsequent training follows a fixed policy, while gradients
and optimizer states can respond to the intervention.
For either fixed-rank or noise-calibrated band selection,
Theorem~\ref{thm:continuation-decomposition} characterizes
local allocation preferences and the fraction of the maximal
first-order improvement rate accessible to two bands
under norm matching.
This is an analysis of the exact spectral target, not an
online hypergradient procedure for selecting $\rho$.

Recent work on optimizer-memory schedules further
emphasizes the roles of training horizon and memory
timescale in optimizer comparisons
\citep{everett2026memory}.
Our pre-update alignment diagnostics, single-update loss
measurements, and continued-pretraining experiments
examine distinct aspects of spectral allocation:
immediate alignment, the loss response to a current
intervention, and performance over a training run.

\subsection{Controlled Continued-Pretraining Evaluation}

Pythia provides a public model suite designed for analysis
across training and scale \citep{biderman2023pythia}.
Continued pretraining on domains and tasks is an established
adaptation setting \citep{gururangan2020dont}.
We use this setting to compare fixed-rank and noise-calibrated
spectral reweighting from common pretrained checkpoints,
under matched training-token budgets and a common
validation protocol.
This complements the from-scratch language-model training
settings used by several related optimizers and evaluates
both allocation and partition-selection rules during
adaptation.
Our comparisons distinguish validation loss and token
efficiency from per-step computation and fraction-selection
effort (Section~\ref{sec:experiments}).

\section{Supplementary Setup and Normalization}
\label{app:spectral-setup}

This appendix supplements Section~\ref{sec:preliminaries}
with normalization identities, gain diagnostics, and rank
conventions, and derives the numerical norm matching used
in Section~\ref{sec:method}.

\subsection{Extended Notation and Spectral Profiles}
\label{app:spectral-profiles}

\paragraph{Extended notation.}
For $A\in\mathbb R^{m\times n}$, let $d=\min\{m,n\}$ and
write $A=U_A\Sigma_AV_A^\top$ for a thin SVD, where
$U_A\in\mathbb R^{m\times d}$ and
$V_A\in\mathbb R^{n\times d}$ have orthonormal columns,
$\Sigma_A=\operatorname{diag}(\sigma_{A,1},\ldots,\sigma_{A,d})$,
and $\sigma_{A,1}\geq\cdots\geq\sigma_{A,d}\geq0$.
For $\mathcal I\subseteq[d]$, let $U_{A,\mathcal I}$
and $V_{A,\mathcal I}$ collect the selected singular vectors.
Their left and right orthogonal projectors are
\[
P_{A,\mathcal I}=U_{A,\mathcal I}U_{A,\mathcal I}^\top,
\qquad
R_{A,\mathcal I}=V_{A,\mathcal I}V_{A,\mathcal I}^\top.
\]

\paragraph{Noise-scale convention.}
At fixed parameters, write the raw full-minibatch gradient as
$G^{\mathrm{raw}}=\overline G+Z$, where
$\overline G=\mathbb E[G^{\mathrm{raw}}]$ and $\mathbb E[Z]=0$.
In the white-noise reference, $s$ denotes the common entrywise standard
deviation of $Z$.
The gradient entering momentum is $G_t=\chi_tG_t^{\mathrm{raw}}$, where
$\chi_t$ is the global clipping scalar.
\textsc{BB-Adaptive} estimates the raw variance from two minibatch halves and
applies $\chi_t^2$ before smoothing.
Section~\ref{sec:split-edge} combines this running estimate with the
reference-model Nesterov correction to calibrate the noise edge for $M_t$.

\paragraph{Normalization identity.}
For the full-rank spectral input
$M_t=U_t\Sigma_tV_t^\top$, the rank-one singular modes
are Frobenius-orthonormal:
\[
\left\langle
u_{t,i}v_{t,i}^\top,\,
u_{t,j}v_{t,j}^\top
\right\rangle_F
=
(u_{t,i}^\top u_{t,j})(v_{t,i}^\top v_{t,j})
=
\mathbf 1\{i=j\}.
\]
Consequently, for a nonzero unnormalized profile $w_t$
as in \eqref{eq:norm-matched-profile},
\[
\left\|
U_t\operatorname{diag}(w_t)V_t^\top
\right\|_F^2
=
\sum_{i=1}^{d}w_{t,i}^2.
\]
The normalizer $c_t(w_t)$ therefore gives
$\|D_t(w_t)\|_F=\sqrt d=\|Q_t\|_F$.
The applied gain vector is $c_t(w_t)w_t$, denoted by
$w$ in Section~\ref{sec:theory-continuation}.
For either partition, the two-level profile satisfies
\[
\sum_{i=1}^{d}\bigl(w_{t,i}^{(\rho)}\bigr)^2
=
r_t+\rho^2(d-r_t),
\]
yielding $c_t(w_t^{(\rho)})=c_{\rho,r_t}$
whenever the profile is nonzero.
For the noise-calibrated partition, the notation
$w_t^{(\rho,\tau_t)}$ makes the threshold dependence explicit.

For a scalar-induced profile $w_{t,i}=g(\sigma_{t,i})$,
the equivalent matrix-function notation is
\[
g(\Sigma_t)=
\operatorname{diag}(g(\sigma_{t,1}),\ldots,g(\sigma_{t,d})),
\qquad
D_t(g)=c_t(g)U_tg(\Sigma_t)V_t^\top,
\]
where $c_t(g)=c_t(w_t)$ and the induced profile is nonzero.

\subsection{Relative and Effective Gains}
\label{app:gain-interpretation}

\paragraph{Absolute gains under a relative reallocation.}
For a full-rank input with both bands nonempty, let
$q_t=r_t/d\in(0,1)$ denote the spike fraction.
For \textsc{BB-Static}, $q_t=\lceil pd\rceil/d$, which can differ slightly
from $p$ because of rounding.
The applied gains relative to the flat profile are
\begin{equation}
g_{\mathrm{sp}}
=c_{\rho,r_t}
=\frac{1}{\sqrt{q_t+\rho^2(1-q_t)}},
\qquad
g_{\mathrm{bu}}=\rho g_{\mathrm{sp}}.
\label{eq:normalized-gains}
\end{equation}
At fixed $q_t$, increasing $\rho$ above $1$ raises $g_{\mathrm{bu}}$ toward
$1/\sqrt{1-q_t}$ and lowers $g_{\mathrm{sp}}$ toward zero.
For a small spike fraction, the bulk gain therefore stays close to one, while
the spike gain can fall substantially.
With $q_t=0.05$ and $\rho=4$, $g_{\mathrm{bu}}\approx1.02$ and
$g_{\mathrm{sp}}\approx0.26$.
At the largest tested fraction, $q_t=0.2$, the same $\rho$ gives
$g_{\mathrm{bu}}\approx1.11$ and $g_{\mathrm{sp}}\approx0.28$.
Thus $\rho$ sets only the ratio of the two gains, and the absolute bulk gain
changes much less than $\rho$.

More generally, multiplying every unnormalized gain by
the same positive scalar leaves the norm-matched direction
unchanged.
For $\rho>0$, profiles with spike and bulk gains
$(1,\rho)$ and $(1/\rho,1)$ therefore give the same update.
Relative bulk emphasis and relative spike suppression
describe the same norm-matched reallocation.

\paragraph{Effective bulk-to-spike gain.}
For a general spectral profile, define the ratio of
root-mean-square gains across the selected bands by
\begin{equation}
\rho_{\mathrm{eff}}(w_t;\mathcal I_{\mathrm{sp},t})
=
\sqrt{
\frac{
\frac{1}{d-r_t}\sum_{i\in\mathcal I_{\mathrm{bu},t}}w_{t,i}^2
}{
\frac{1}{r_t}\sum_{i\in\mathcal I_{\mathrm{sp},t}}w_{t,i}^2
}}.
\label{eq:effective-gain}
\end{equation}
This statistic is defined when both bands are nonempty
and the denominator is positive.
The common normalizer cancels, so the statistic is the
same for $w_t$ and its applied gains $c_t(w_t)w_t$.
In particular, the exact two-level target satisfies
$\rho_{\mathrm{eff}}=\rho$.
For \textsc{BB-Adaptive}, we abbreviate it as
$\rho_{\mathrm{eff}}(w_t;\tau_t)$.
The partition specifies band membership, whereas
$\rho_{\mathrm{eff}}$ summarizes the relative allocation;
it does not describe gain variation within a band or
measure subspace-recovery accuracy.

\subsection{Rank Conventions and Degenerate Partitions}
\label{app:rank-conventions}

\paragraph{Empty bands.}
For a full-rank input, reallocation between bands requires
$0<r_t<d$.
If $r_t=d$, all unnormalized gains equal one and the
target is $Q_t$.
If $r_t=0$ and $\rho>0$, all unnormalized gains equal
$\rho$, and normalization again gives $Q_t$.
When $r_t=0$ and $\rho=0$, the profile is zero and cannot
be normalized to $\sqrt d$.
This case is excluded by the nonzero-profile condition
in Section~\ref{sec:preliminaries}.

\paragraph{Rank-deficient inputs.}
For a rank-deficient spectral input, restrict the exact
construction to the active indices
$\mathcal I_{+,t}=\{i\in[d]:\sigma_{t,i}>0\}$,
with active rank $d_t^+=|\mathcal I_{+,t}|$.
The partial polar factor is
\[
Q_t=\sum_{i\in\mathcal I_{+,t}}u_{t,i}v_{t,i}^\top,
\qquad
\|Q_t\|_F=\sqrt{d_t^+}.
\]
Intersect each partition from
Section~\ref{sec:preliminaries} with $\mathcal I_{+,t}$,
and let $r_t$ count the selected active spike modes.
In particular, the fixed-rank rule gives
$r_t=\min\{\lceil pd\rceil,d_t^+\}$:
the prescribed fraction still refers to
$d=\min\{m,n\}$, while zero singular modes receive no gains.

For $d_t^+>0$ and a nonzero active gain vector, the exact
normalizers are
\[
c_t(w_t)=
\sqrt{\frac{d_t^+}
{\sum_{i\in\mathcal I_{+,t}}w_{t,i}^2}},
\qquad
c_{\rho,r_t}
=
\sqrt{\frac{d_t^+}
{r_t+\rho^2(d_t^+-r_t)}}.
\]
The gain diagnostics in
Appendix~\ref{app:gain-interpretation} likewise use the
active band sizes.
The matrix dimensions remain $m,n$; replacing $d$ by
$d_t^+$ in these identities does not change the dimensions
used to calibrate the noise edge.
For a zero spectral input, the partial polar factor and
exact update are zero, without evaluating normalization
ratios.

\subsection{Derivation of Numerical Norm Matching}
\label{app:numerical-norm-matching}

The numerical update in
\eqref{eq:svdfree-raw-direction} and
\eqref{eq:numerical-norm-matching}
uses an approximate polar factor $\widehat Q_t$ and
a tracked basis
$X_t\in\mathbb R^{d\times\widehat r_t}$.
We first derive the identities in exact arithmetic,
assuming that $X_t$ has orthonormal columns;
neither an exact polar factor nor exact spike recovery
is required for the projection calculation.

Define
\begin{equation}
\widehat Q_{\mathrm{sp},t}
=
\begin{cases}
X_t(X_t^\top\widehat Q_t), & m\leq n,\\
(\widehat Q_tX_t)X_t^\top, & n<m,
\end{cases}
\qquad
\widehat Q_{\mathrm{bu},t}
=\widehat Q_t-\widehat Q_{\mathrm{sp},t}.
\label{eq:numerical-band-components}
\end{equation}
These components use the tracked subspace and its orthogonal complement,
which can differ from the exact spike and bulk bands.
The raw direction is
\[
\widetilde D_t
=
\widehat Q_{\mathrm{sp},t}
+\rho\widehat Q_{\mathrm{bu},t}
=
\rho\widehat Q_t-(\rho-1)\widehat Q_{\mathrm{sp},t}.
\]

Since $X_tX_t^\top$ is an orthogonal projector,
\[
\left\langle
\widehat Q_{\mathrm{sp},t},
\widehat Q_{\mathrm{bu},t}
\right\rangle_F=0.
\]
Writing
\[
a_t=\|\widehat Q_t\|_F^2,
\qquad
b_t=\|\widehat Q_{\mathrm{sp},t}\|_F^2
\]
therefore gives
$\|\widehat Q_{\mathrm{bu},t}\|_F^2=a_t-b_t$.
For nonzero $\widehat Q_t$ and $\widetilde D_t$,
\begin{equation}
\|\widetilde D_t\|_F^2
=b_t+\rho^2(a_t-b_t),
\qquad
\widehat c_{\rho,t}
=
\sqrt{\frac{a_t}{b_t+\rho^2(a_t-b_t)}}
=
\frac{\|\widehat Q_t\|_F}{\|\widetilde D_t\|_F}.
\label{eq:numerical-normalization-identity}
\end{equation}
Thus $\widehat D_t=\widehat c_{\rho,t}\widetilde D_t$
satisfies $\|\widehat D_t\|_F=\|\widehat Q_t\|_F$.
The projected norm can be evaluated on the smaller side:
\[
b_t=
\begin{cases}
\|X_t^\top\widehat Q_t\|_F^2, & m\leq n,\\
\|\widehat Q_tX_t\|_F^2, & n<m.
\end{cases}
\]

For $a_t>0$ and $\rho>0$, the raw norm is positive because
$0\leq b_t\leq a_t$.
At $\rho=0$, a nonzero projected spike component is required.
A zero $\widehat Q_t$ gives a zero numerical update,
without evaluating the norm ratio.
At $\rho=1$, $\widetilde D_t=\widehat Q_t$, so the update reduces to the
unreweighted numerical polar direction.

\paragraph{Relation to the exact target.}
If $\widehat Q_t=Q_t$ and $M_t$ has full rank, then
\[
Q_tQ_t^\top=I_m
\quad\text{when }m\leq n,
\qquad
Q_t^\top Q_t=I_n
\quad\text{when }n<m.
\]
Consequently, $a_t=d$ and $b_t=\widehat r_t$ for any
orthonormal tracked basis, giving
$\widehat c_{\rho,t}=c_{\rho,\widehat r_t}$
whenever the raw direction is nonzero.
This equality of normalizers does not require an exact
spike subspace.
If $X_t$ also spans the selected spike subspace, then
$\widehat r_t=r_t$ and the numerical construction recovers
the exact target $D_t(w_t^{(\rho)})$.
For an approximate polar factor, $a_t$ need not equal
$d$ and $b_t$ need not equal $\widehat r_t$, so the
normalizers can differ.
For rank-deficient inputs, the direct norm ratio remains
applicable, but the full-rank simplification above does not
apply in general.

\paragraph{Finite-precision computation.}
The direct norm ratio in
\eqref{eq:numerical-norm-matching}
requires only Frobenius norms and rescaling.
Its norm-matching property does not rely on exact
orthogonality of the computed components, unlike the
expression involving $a_t-b_t$.
For a nonzero computed raw direction, exact evaluation of
the ratio matches its norm to $\|\widehat Q_t\|_F$;
finite-precision evaluation is subject to rounding error.
Norm matching controls magnitude without correcting
polar or subspace approximation errors.

\section{Implementation Details and SVD-Free Algorithm}
\label{app:algorithm}

This appendix gives the implementation details behind
Section~\ref{sec:method}.
It covers noise calibration, the exact-SVD target used in oracle
comparisons, spike tracking for both variants, and the full
\textsc{BB-Adaptive} algorithm.
\textsc{BB-Static} and \textsc{BB-Adaptive} share the reweighting and
normalization of Section~\ref{sec:preliminaries}.
They differ only in how they choose and track the spike band.
Appendix~\ref{app:numerical-norm-matching} derives numerical norm
matching, and Proposition~\ref{prop:edge-calibration} derives the
half-batch noise estimate and the Nesterov noise factor.

\subsection{Calibration and Estimator State}
\label{app:method-calibration}

\paragraph{Half-batch noise estimate.}
The two halves in \eqref{eq:split-gradients} are sampled
independently at the same parameters, so they have the same mean.
This mean cancels in the half-batch difference
$Z_t^{\mathrm{split}}$.
After vectorization, $Z_t^{\mathrm{split}}$ has the same covariance
as the sampling noise in $G_t^{\mathrm{raw}}$.
Hence
$\widehat s_t^{\,2}=\|Z_t^{\mathrm{split}}\|_F^2/(mn)$
estimates the average entrywise noise variance of the
full-minibatch gradient.
For isotropic noise, this is the common entrywise variance.
For anisotropic noise, one scalar cannot describe the full
covariance.

\paragraph{Clipping and variance smoothing.}
Let $\chi_t\in(0,1]$ be the global clipping scalar.
It is computed from the raw full-minibatch gradient of all model
parameters, and all matrices share it.
Each matrix uses $G_t=\chi_tG_t^{\mathrm{raw}}$ in the momentum
recursion.
The halves are not clipped separately.
Instead, each matrix's raw variance estimate is scaled by
$\chi_t^2$ before smoothing:
\[
\overline s_t^{\,2}
=
\begin{cases}
\chi_t^2\widehat s_t^{\,2},
& \text{if the variance state is uninitialized},\\[2pt]
\gamma\overline s_{t-1}^{\,2}
+(1-\gamma)\chi_t^2\widehat s_t^{\,2},
& \text{otherwise}.
\end{cases}
\]
The EMA averages variances.
The square root $\sqrt{\overline s_t^{\,2}}$ in
\eqref{eq:estimated-noise-edge} is taken afterwards.
For a matrix first observed at step~$1$, this gives
$\overline s_1^{\,2}=\chi_1^2\widehat s_1^{\,2}$.
Without clipping, $\chi_t=1$.

\paragraph{Nesterov factor.}
Proposition~\ref{prop:edge-calibration} derives the factor
$\nu_t(\beta)$ in \eqref{eq:nesterov-noise-factor} under a
reference model: centered, isotropic noise that is independent over
time and has constant variance, with $B_0=0$.
Writing $M_t=(1-\beta^2)G_t+\beta^2B_{t-1}$ gives an equivalent
two-term form:
\[
\nu_t^2(\beta)
=
(1-\beta^2)^2
+
\beta^4\frac{1-\beta}{1+\beta}
\left(1-\beta^{2(t-1)}\right).
\]
The first term comes from the current gradient, which enters $M_t$
both directly and through the updated buffer.
The second term comes from the past gradients stored in $B_{t-1}$.

\paragraph{Approximations.}
The implementation replaces the reference variance with the
clipping-scaled running estimate $\overline s_t^{\,2}$.
This puts the estimate on the scale of the applied gradient.
Because $\chi_t$ depends on the data, the scaled estimate is in
general biased for the noise in the clipped gradient.
The running estimate also assumes that the noise level changes
slowly.
Finally, the noise edge assumes Gaussian white noise
(Section~\ref{sec:preliminaries}).
The multiplier $\kappa_{\mathrm{edge}}$ can shift the threshold, but
a single scalar cannot correct for anisotropic or correlated noise in
general.
The experiments use $\gamma=0.95$ and $\kappa_{\mathrm{edge}}=1$.

\subsection{Exact Target and Oracle Comparison}
\label{sec:exact-oracle}

For a full-rank input $M_t$ and a nonzero two-level profile, the
exact-SVD oracle evaluates, for either rule,
\begin{equation}
D_t^{\mathrm{svd}}
=
D_t\!\left(w_t^{(\rho)}\right)
=
c_{\rho,r_t}U_t
\operatorname{diag}\!\left(w_t^{(\rho)}\right)
V_t^\top.
\label{eq:exact-oracle-update}
\end{equation}
For \textsc{BB-Adaptive}, $r_t$ counts all singular values above the
threshold $\tau_t=\widehat\tau_t$.
For \textsc{BB-Static}, $r_t=\lceil pd\rceil$.
In both cases, $c_{\rho,r_t}$ is the normalizer in
\eqref{eq:two-level-profile}.
Rank-deficient and zero inputs follow
Appendix~\ref{app:rank-conventions}.

Here, \emph{oracle} only means that the rule is evaluated exactly
with an SVD.
It does not choose the partition or the gain.
An oracle comparison fixes $M_t$, the rule, and $\rho$.
Both realizations therefore use the same threshold for
\textsc{BB-Adaptive} and the same rank for \textsc{BB-Static}.
The exact oracle matches $\|Q_t\|_F$, and the numerical realization
matches $\|\widehat Q_t\|_F$.
If $\widehat Q_t=Q_t$ and the orthonormal tracked basis spans the
selected spike subspace, the numerical update equals the oracle
(Appendix~\ref{app:numerical-norm-matching}).

For \textsc{BB-Adaptive}, the rank cap applies only to the numerical
tracker.
The exact target keeps every direction above the threshold.
Both sides use the same spectral rule, so a comparison at a common
input measures only numerical error: polar approximation, imperfect
subspace recovery, and, for \textsc{BB-Adaptive}, truncation by the
cap.

\subsection{Warm-Started Block Iteration}
\label{app:spike-tracking}

This subsection first gives the tracker for \textsc{BB-Adaptive} and
then the simpler tracker for \textsc{BB-Static}.
Both use the smaller-side Gram operator
$K_t\in\mathbb R^{d\times d}$ from \eqref{eq:gram-operator}.
Its action on a block $Y$ is
\[
K_tY=
\begin{cases}
M_t(M_t^\top Y), & m\leq n,\\
M_t^\top(M_tY), & n<m,
\end{cases}
\]
so the Gram matrix is never formed.
As in Section~\ref{sec:svd-free}, $\operatorname{orth}$ returns an
orthonormal basis of the column space of its argument.
An empty column space gives an empty basis.

\paragraph{Initialization and trial-space construction.}
Let $X_{t-1}$ be the spike basis that \textsc{BB-Adaptive} retained
at the previous step, with $\widehat r_{t-1}$ columns.
Before the first update, $X_0$ is empty and $\widehat r_0=0$.
For positive initial and refresh probe widths
$n_{\mathrm{init}}$ and $n_{\mathrm{probe}}$, set
\[
n_{\mathrm{fresh},t}
=
\begin{cases}
\min\{d,n_{\mathrm{init}}\},
& \widehat r_{t-1}=0,\\
\min\{d,n_{\mathrm{probe}}\},
& \widehat r_{t-1}>0.
\end{cases}
\]
The initial width is also used whenever the previous step selected
no directions.

Draw
$\Omega_t\in\mathbb R^{d\times n_{\mathrm{fresh},t}}$
with independent standard Gaussian entries and set
\[
Y_0=\operatorname{orth}([X_{t-1},\Omega_t]).
\]
The concatenated block may have more than $d$ columns, but its
orthonormal basis has at most $d$.
The previous basis gives a warm start, and the fresh probes let new
directions enter the trial space.
At the first step, $Y_0$ comes from the probes alone.
The first retained basis $X_1$ is obtained only after iteration and
threshold selection.

\paragraph{Block iteration and Ritz extraction.}
Within training step $t$, perform $q$ block-power iterations:
\begin{equation}
Y_{j+1}=\operatorname{orth}(K_tY_j),
\qquad j=0,\ldots,q-1.
\label{eq:block-power-tracking}
\end{equation}
Let $\ell_t$ be the number of columns of $Y_q$.
For a nonempty trial basis, form
\[
T_t=Y_q^\top K_tY_q=O_t\Lambda_tO_t^\top,
\qquad
\lambda_{t,1}\geq\cdots\geq\lambda_{t,\ell_t}\geq0,
\]
where
$\Lambda_t=\operatorname{diag}
(\lambda_{t,1},\ldots,\lambda_{t,\ell_t})$.
The Ritz directions are the columns of $Y_qO_t$,
with singular-value estimates $\sqrt{\lambda_{t,i}}$.
The eigenvalues are nonnegative in exact arithmetic
because $K_t$ is positive semidefinite.

For a nonnegative integer rank cap $r_{\max}$, let
$n_{\mathrm{cand},t}$ count the Ritz values above the threshold:
\begin{equation}
n_{\mathrm{cand},t}
=
\bigl|\{i\in[\ell_t]:
\lambda_{t,i}>\widehat\tau_t^{\,2}\}\bigr|,
\qquad
\widehat r_t=\min\{n_{\mathrm{cand},t},r_{\max}\},
\qquad
X_t=\operatorname{orth}\!\left(
Y_qO_{t,[\widehat r_t]}
\right).
\label{eq:ritz-spike-selection}
\end{equation}
Here $O_{t,[\widehat r_t]}$ holds the first $\widehat r_t$
eigenvectors in decreasing eigenvalue order.
If the cap binds, the largest estimates above the threshold are
kept.
An empty selection gives an empty spike basis.
If $\ell_t=0$, set $\widehat r_t=0$ and leave $X_t$ empty; no
eigendecomposition is needed.
The experiments use this hard threshold without hysteresis.

The selected Ritz block is reorthonormalized before it is stored.
The projector identities in
Appendix~\ref{app:numerical-norm-matching}
assume orthonormal columns, so in finite precision they hold only up
to orthonormalization error.

\paragraph{Experimental settings and random-number state.}
All \textsc{BB-Adaptive} training runs use
\[
q=6,
\qquad
n_{\mathrm{init}}=128,
\qquad
n_{\mathrm{probe}}=128,
\qquad
r_{\max}=256.
\]
Probe widths are limited by $d$.
The cap $r_{\max}$ is an absolute number of directions and does not
scale with $d$.
The retained rank satisfies
$\widehat r_t\leq\min\{\ell_t,r_{\max}\}$, with $\ell_t\leq d$.
The first retained basis has at most $\min\{d,128\}$ columns, below
the cap.
The cap applies only to the retained basis.
In these runs, the trial space has up to
$\min\{d,\widehat r_{t-1}+128\}$ columns, which can exceed the cap.

Gaussian probes are generated in float32 with a dedicated
pseudorandom-number generator for each parameter matrix, separate
from the training random-number stream.
Its state carries over from step to step and is never reseeded.

\paragraph{Fixed-rank tracking (\textsc{BB-Static}).}
\textsc{BB-Static} tracks a block of fixed width
$r=\lceil pd\rceil$.
The fraction $p$ is shared across matrices and fixed for the whole
run.
If a stored basis $X_{t-1}\in\mathbb R^{d\times r}$ exists, it is
the warm start, $Y_0=X_{t-1}$.
Otherwise, as at the first step, we draw a Gaussian block
$\Omega_t\in\mathbb R^{d\times r}$ and set
$Y_0=\operatorname{orth}(\Omega_t)$.
At each training step, we run
\[
Y_{j+1}=\operatorname{orth}(K_tY_j),
\qquad j=0,\ldots,q-1,
\qquad X_t=Y_q,
\]
with $q=6$ in the experiments.
The whole final basis enters \eqref{eq:svdfree-raw-direction} and is
stored for the next step.
Unlike the \textsc{BB-Adaptive} tracker, this tracker adds no fresh
probes, does no Ritz extraction, and applies no threshold or rank
cap.
The parameters $n_{\mathrm{init}}$, $n_{\mathrm{probe}}$, and
$r_{\max}$ therefore do not apply.
For example, $p=0.20$ and $d=1024$ give $r=205$, so initialization
and all six iterations use a $1024\times205$ block.

\subsection{Full Algorithm}
\label{app:positive-gain-realization}

Algorithm~\ref{alg:ser-muon} gives one \textsc{BB-Adaptive} step for
one matrix.
It combines noise calibration, spike tracking, and numerical norm
matching.
It returns the update direction before learning-rate scaling.
The clipping scalar $\chi_t$ is computed once from the full-model
gradient and passed in.
For \textsc{BB-Static}, the noise-estimate and threshold steps are
skipped, and the fixed-rank tracker of
Appendix~\ref{app:spike-tracking} replaces the adaptive one.

Before the first update, $B_0=0$, $X_0$ is empty,
and the variance state is uninitialized.
The algorithm returns $B_t$, $X_t$, and
$\overline s_t^{\,2}$ along with $\widehat D_t$.
The probe generator's state is kept separately for each matrix.
Since $\rho>0$, a nonzero $\widehat Q_t$ and an orthonormal $X_t$
give a nonzero raw direction in exact arithmetic, so the norm ratio
is defined.
Appendix~\ref{app:numerical-norm-matching} covers $\rho=0$ and
finite-precision effects.

\begin{algorithm}[t]
\caption{SVD-free \textsc{BB-Adaptive} update for one matrix}
\label{alg:ser-muon}

{\small
\noindent\textbf{Inputs:}
Mean gradients $G_t^{(1)},G_t^{(2)}\in\mathbb R^{m\times n}$ of the
two minibatch halves, before clipping;
global clipping scalar $\chi_t\in(0,1]$;
previous states $B_{t-1},X_{t-1}$;
previous variance state $\overline s_{t-1}^{\,2}$
(possibly uninitialized);
the matrix's probe generator;
step $t\geq1$.
Write $d=\min\{m,n\}$.
\par
\noindent\textbf{Parameters:}
Bulk-to-spike gain $\rho>0$; momentum $\beta\in[0,1)$;
smoothing $\gamma\in[0,1)$;
edge multiplier $\kappa_{\mathrm{edge}}>0$;
power iterations $q$;
initial probe width $n_{\mathrm{init}}$;
refresh probe width $n_{\mathrm{probe}}$;
absolute rank cap $r_{\max}$.
\par}
\medskip

\noindent
\begin{minipage}[t]{0.61\linewidth}
\small
\begin{algorithmic}[1]

\STATE $Z_t^{\mathrm{split}}
    \gets (G_t^{(1)}-G_t^{(2)})/2$
\STATE $\widehat s_t^{\,2}
    \gets \|Z_t^{\mathrm{split}}\|_F^2/(mn)$
\STATE $G_t^{\mathrm{raw}}
    \gets (G_t^{(1)}+G_t^{(2)})/2$
\STATE $G_t\gets\chi_tG_t^{\mathrm{raw}}$

\IF{the variance state is uninitialized}
    \STATE $\overline s_t^{\,2}
        \gets\chi_t^2\widehat s_t^{\,2}$
\ELSE
    \STATE $\overline s_t^{\,2}
        \gets\gamma\overline s_{t-1}^{\,2}
        +(1-\gamma)\chi_t^2\widehat s_t^{\,2}$
\ENDIF

\STATE $B_t\gets\beta B_{t-1}+(1-\beta)G_t$
\STATE $M_t\gets(1-\beta)G_t+\beta B_t$

\STATE Compute $\nu_t(\beta)$ from
\eqref{eq:nesterov-noise-factor}
\STATE $\widehat\tau_t\gets
\kappa_{\mathrm{edge}}\nu_t(\beta)
\sqrt{\overline s_t^{\,2}}(\sqrt m+\sqrt n)$

\STATE $\widehat Q_t
    \gets\operatorname{PolarExpress}(M_t)$

\STATE Define the Gram action using
\eqref{eq:gram-operator}

\IF{$X_{t-1}$ is empty}
    \STATE $f\gets\min\{d,n_{\mathrm{init}}\}$
\ELSE
    \STATE $f\gets\min\{d,n_{\mathrm{probe}}\}$
\ENDIF

\STATE Draw $\Omega_t\in\mathbb R^{d\times f}$
with independent $\mathcal N(0,1)$ entries
using this matrix's probe generator
\STATE $Y_0
    \gets\operatorname{orth}([X_{t-1},\Omega_t])$

\FOR{$j=0,\ldots,q-1$}
    \STATE $Y_{j+1}\gets\operatorname{orth}(K_tY_j)$
\ENDFOR

\STATE $\ell_t\gets$ number of columns of $Y_q$

\IF{$\ell_t=0$}
    \STATE $\widehat r_t\gets0$; set $X_t$ empty
\ELSE
    \STATE Compute
    $Y_q^\top K_tY_q=O_t\Lambda_tO_t^\top$
    \STATE Order eigenpairs by decreasing eigenvalue

    \STATE $n_{\mathrm{cand},t}\gets
    |\{i\in[\ell_t]:
    (\Lambda_t)_{ii}>\widehat\tau_t^{\,2}\}|$

    \STATE $\widehat r_t
        \gets\min\{n_{\mathrm{cand},t},r_{\max}\}$

    \STATE $X_t\gets
    \operatorname{orth}(Y_qO_{t,[\widehat r_t]})$
\ENDIF

\IF{$\|\widehat Q_t\|_F=0$}
    \STATE $\widehat D_t\gets0$
\ELSE
    \STATE Form $\widetilde D_t$ using
    \eqref{eq:svdfree-raw-direction}
    \STATE $\widehat D_t\gets
    \dfrac{\|\widehat Q_t\|_F}{\|\widetilde D_t\|_F}
    \widetilde D_t$
\ENDIF

\RETURN $\widehat D_t,B_t,X_t,\overline s_t^{\,2}$

\end{algorithmic}
\end{minipage}\hfill
\begin{minipage}[t]{0.35\linewidth}
\small
\raggedright
\setlength{\parindent}{0pt}

\textbf{Noise calibration.}
The half-batch difference estimates the average entrywise noise
variance of the full-minibatch mean gradient.
The clipping scalar rescales the gradient by $\chi_t$ and the
variance estimate by $\chi_t^2$.
The EMA averages variances; the square root and the Nesterov factor
are applied afterwards.
\par\medskip

\textbf{Spike tracking.}
Fresh Gaussian probes are added at every step:
$n_{\mathrm{init}}$ when the previous basis is empty and
$n_{\mathrm{probe}}$ otherwise.
A nonempty previous basis also gives a warm start.
The trial basis has at most $d$ columns.
Hard-threshold selection keeps at most $r_{\max}$ Ritz directions,
which are then reorthonormalized.
\par\medskip

\textbf{Persistent state.}
The momentum buffer, retained basis, and variance estimate carry
over between steps.
The probe generator advances when probes are drawn and is never
reseeded.
\par\medskip

\textbf{Norm matching.}
In exact arithmetic, $\|\widehat D_t\|_F=\|\widehat Q_t\|_F$.
At $\rho=1$, the update is the unreweighted direction
$\widehat Q_t$.
Norm matching fixes the magnitude only; it does not correct polar or
subspace approximation errors.
\par\medskip

\end{minipage}
\end{algorithm}

\section{Band Content Along \muon{} Trajectories}
\label{app:band-content}
\label{sec:band-content}

\emph{Does the bulk, whose singular values lie at or below the
noise edge, contribute meaningfully to reference-gradient
alignment under the flat profile?}
We examine this question along \muon{} trajectories of
continued pretraining, measuring band sizes, per-direction alignment,
and signed contributions across matrices.
The measured bulk is \emph{weak per direction but substantial
in aggregate}: it contains 93.7--97.1\% of pooled singular
directions and has positive signed aggregate alignment in every
analyzed matrix.
These observations motivate a broad-bulk/small-spike
partition and the fixed fractions tested in
Section~\ref{sec:experiments}.
Whether increasing the bulk's relative gain improves an
update is a separate question, which
Appendix~\ref{app:band-content-immediate} examines.

\subsection{Protocol and Frame-Consistent Calibration}
\label{app:band-content-protocol}

\paragraph{Training trajectories and checkpoints.}
For each of \FineMath{}, \CodeParrot{}, and \WikiAr{}, we
analyze one Pythia-70M \muon{} trajectory of continued
pretraining.
Each run uses approximately 100M tokens, with batches of
32 sequences of length 1024, momentum $\beta=0.95$,
global gradient clipping at norm 1, and Newton--Schulz
iterations for the polar map, as in the \muon{} baseline of
Section~\ref{sec:experiments}.
The matrix optimizer and the auxiliary AdamW optimizer both
use peak learning rate $5\times10^{-4}$, the shared
Pythia-70M rate of Section~\ref{sec:experiments}.
These diagnostics observe the \muon{} trajectories without
changing their updates.

We report four checkpoints as 0.0M, 25.0M, 50.0M, and
99.9M consumed training tokens.
The initial measurement is a transient probe before
continued pretraining; subsequent measurements precede
updates after 24,969,216, 49,971,200, and 99,942,400 tokens.
At each checkpoint, we analyze all 24 hidden weight matrices,
excluding embeddings and the language-model head.
For matrix $\ell$ with dimensions $m_\ell\times n_\ell$,
write $d_\ell=\min\{m_\ell,n_\ell\}$.
Here each matrix supplies 512 thin-SVD directions, giving
$N_{\mathrm{dir}}=\sum_\ell d_\ell=12{,}288$ directions per
corpus--checkpoint--frame combination.

\paragraph{An independent reference gradient.}
At fixed model parameters, construct the large-batch reference
\begin{equation}
H_t=\frac{1}{16}\sum_{j=1}^{16}G_t^{\mathrm{ref},j},
\label{eq:sec33-reference}
\end{equation}
where each $G_t^{\mathrm{ref},j}$ is an unclipped mean
gradient from an independently sampled full minibatch of
the same corpus.
The reference contains 512 sequences, or 524,288 input tokens;
$H_{t,\ell}$ denotes its block for matrix $\ell$.
It is neither clipped nor used for parameter updates.
Reference sampling uses a separate data random-number stream,
and the PyTorch random-number state is restored afterward.
Full thin SVDs are computed only at diagnostic checkpoints.

\paragraph{Raw-gradient and Nesterov frames.}
As in Section~\ref{sec:split-edge}, distinguish
the raw gradient $G_{t,\ell}^{\mathrm{raw}}$ from the clipped
gradient $G_{t,\ell}=\chi_tG_{t,\ell}^{\mathrm{raw}}$,
where $\chi_t$ is the full-model clipping scalar.
Using the actual previous momentum buffer, reconstruct
\begin{equation}
\begin{aligned}
B_{t,\ell}
&=\beta B_{t-1,\ell}+(1-\beta)G_{t,\ell},\\
M_{t,\ell}
&=(1-\beta)G_{t,\ell}+\beta B_{t,\ell}.
\end{aligned}
\label{eq:sec33-frames}
\end{equation}
We analyze $A^G_{t,\ell}=G_{t,\ell}^{\mathrm{raw}}$ in the
\emph{raw-gradient frame} and $A^M_{t,\ell}=M_{t,\ell}$ in
the \emph{Nesterov frame}, against the same $H_{t,\ell}$.
The first is the current minibatch gradient.
The second is the input to \muon{}'s polar map, which
combines the current gradient with the momentum buffer.

\paragraph{A noise edge matched to each frame.}
For each training minibatch, compute the raw mean gradients
of the two halves and form
\begin{equation}
G_{t,\ell}^{\mathrm{raw}}
=\frac{G_{t,\ell}^{(1)}+G_{t,\ell}^{(2)}}{2},
\qquad
Z_{t,\ell}^{\mathrm{split}}
=\frac{G_{t,\ell}^{(1)}-G_{t,\ell}^{(2)}}{2}.
\label{eq:sec33-split}
\end{equation}
Let $\widehat v_{t,\ell}
=\|Z_{t,\ell}^{\mathrm{split}}\|_F^2/(m_\ell n_\ell)$
be the raw half-batch variance estimate, which corresponds to
$\widehat s_t^{\,2}$ in Section~\ref{sec:split-edge}.
The diagnostics keep both a raw and a clipping-scaled
variance EMA:
\begin{equation}
\begin{aligned}
v_{t,\ell}^{\mathrm{raw}}
&=\gamma v_{t-1,\ell}^{\mathrm{raw}}
 +(1-\gamma)\widehat v_{t,\ell},\\
v_{t,\ell}^{\mathrm{clip}}
&=\gamma v_{t-1,\ell}^{\mathrm{clip}}
 +(1-\gamma)\chi_t^2\widehat v_{t,\ell}.
\end{aligned}
\label{eq:sec33-noise-ema}
\end{equation}
Initialize each EMA directly from its first observation,
then apply these recursions every training step with
$\gamma=0.95$.
The clipping-scaled state is $\overline s_t^{\,2}$ of
Section~\ref{sec:split-edge} for matrix $\ell$; the raw state
serves the raw-gradient frame.
The squared clipping factor is applied before smoothing.

Using the Nesterov variance multiplier
\begin{equation}
\nu_t^2(\beta)
=(1-\beta^2)^2+
\beta^4\frac{1-\beta}{1+\beta}
\bigl(1-\beta^{2(t-1)}\bigr),
\qquad t\geq1,
\label{eq:sec33-nesterov-factor}
\end{equation}
from Proposition~\ref{prop:edge-calibration}, set
\begin{equation}
\begin{aligned}
\tau_{t,\ell}^{G}
&=\kappa_{\mathrm{edge}}
\sqrt{v_{t,\ell}^{\mathrm{raw}}}
(\sqrt{m_\ell}+\sqrt{n_\ell}),\\
\tau_{t,\ell}^{M}
&=\kappa_{\mathrm{edge}}
\nu_t(\beta)\sqrt{v_{t,\ell}^{\mathrm{clip}}}
(\sqrt{m_\ell}+\sqrt{n_\ell}),
\end{aligned}
\label{eq:sec33-thresholds}
\end{equation}
with $\kappa_{\mathrm{edge}}=1$ and $\nu_t(\beta)\geq0$.
Thus the raw-gradient frame uses the raw noise scale, and
the Nesterov frame uses the clipping-scaled noise scale
adjusted for momentum.
In each frame, singular values strictly above the threshold
are spikes; the remaining directions form the bulk.
Selection uses the full computed spectrum, without a
tracker rank cap.
As in Section~\ref{sec:preliminaries}, these thresholds are
approximate references based on white noise.

The initial Nesterov measurement uses a transient first-step
probe with a zero buffer.
Direct initialization of the variance states gives
\[
M_{1,\ell}=(1-\beta^2)\chi_1G_{1,\ell}^{\mathrm{raw}},
\qquad
\tau^M_{1,\ell}=(1-\beta^2)\chi_1\tau^G_{1,\ell}.
\]
The matched positive rescaling preserves the initial
partition and polar direction, which explains the agreement
between frames at 0.0M tokens.
The probe leaves the training buffer and variance states
unchanged.

\paragraph{What per-direction alignment measures.}
For $F\in\{G,M\}$, write
$A^F_{t,\ell}=U^F_{t,\ell}\Sigma^F_{t,\ell}
(V^F_{t,\ell})^\top$ and define
\begin{equation}
\alpha^F_{t,\ell,i}
=(u^F_{t,\ell,i})^\top H_{t,\ell}v^F_{t,\ell,i}.
\label{eq:sec33-alpha}
\end{equation}
This is the signed projection of the reference gradient
onto the unit-Frobenius matrix
$u^F_{t,\ell,i}(v^F_{t,\ell,i})^\top$, without normalization
by $\|H_{t,\ell}\|_F$, so it is not a cosine similarity.
A positive value means that a step along
$-u^F_{t,\ell,i}(v^F_{t,\ell,i})^\top$ lowers the reference
objective to first order.
Under the flat profile, the scores sum to the polar alignment:
\begin{equation}
\sum_i\alpha^F_{t,\ell,i}
=\left\langle H_{t,\ell},
\operatorname{polar}(A^F_{t,\ell})\right\rangle_F.
\label{eq:sec33-polar-descent}
\end{equation}
These identities concern the exact polar targets of
Section~\ref{sec:preliminaries}, with the rank conventions
of Appendix~\ref{app:rank-conventions}.
Alignment with the raw input matrix instead weights each
$\alpha^F_{t,\ell,i}$ by its singular value.
The raw-gradient-frame polar direction is only a diagnostic
comparison; \muon{} uses the Nesterov frame.

\paragraph{Signed and absolute contributions within a matrix.}
Suppress the frame, time, and matrix indices temporarily.
For a nonempty band $\mathcal I$, its mean absolute
alignment is
$\overline a(\mathcal I)=|\mathcal I|^{-1}
\sum_{i\in\mathcal I}|\alpha_i|$.
Its signed and absolute shares are
\begin{equation}
S(\mathcal I)=
\frac{\sum_{i\in\mathcal I}\alpha_i}
     {\sum_{j=1}^{d_\ell}\alpha_j},
\qquad
S_{\mathrm{abs}}(\mathcal I)=
\frac{\sum_{i\in\mathcal I}|\alpha_i|}
     {\sum_{j=1}^{d_\ell}|\alpha_j|},
\label{eq:band-descent-shares}
\end{equation}
when the respective denominators are nonzero.
The signed share measures the net contribution after positive
and negative projections cancel; the absolute share measures
projection magnitude without cancellation.
A signed share can be negative or exceed one, and it can be
sensitive to a small total signed denominator.
For applied gains $w=(w_i)$, the contribution of direction $i$
becomes $w_i\alpha_i$.
A common positive normalizer cancels from these within-matrix
shares.
In general, it does not cancel from sums across matrices with
different normalizers.

\paragraph{How the tables pool matrices.}
The tables pool counts and alignment sums first and then form
ratios; they do not average matrix-wise shares.
Suppressing checkpoint and frame indices, define for
$b\in\{\mathrm{sp},\mathrm{bu}\}$
\begin{equation}
\begin{aligned}
N_b&=\sum_\ell|\mathcal I_{b,\ell}|,\\
A_b&=\sum_\ell\sum_{i\in\mathcal I_{b,\ell}}
\alpha_{\ell,i},
\qquad
B_b=\sum_\ell\sum_{i\in\mathcal I_{b,\ell}}
|\alpha_{\ell,i}|.
\end{aligned}
\label{eq:sec33-pooled-sums}
\end{equation}
We report the bulk fraction $N_{\mathrm{bu}}/N_{\mathrm{dir}}$,
the mean absolute per-direction alignments $a_b=B_b/N_b$,
and their spike-to-bulk ratio
$\Gamma=a_{\mathrm{sp}}/a_{\mathrm{bu}}$.
The pooled bulk shares are
\begin{equation}
S_{\mathrm{bu}}
=\frac{A_{\mathrm{bu}}}{A_{\mathrm{sp}}+A_{\mathrm{bu}}},
\qquad
S_{\mathrm{bu}}^{\mathrm{abs}}
=\frac{B_{\mathrm{bu}}}{B_{\mathrm{sp}}+B_{\mathrm{bu}}}.
\label{eq:sec33-shares}
\end{equation}
To quantify cancellation, we also report
\begin{equation}
C=\frac{|A_{\mathrm{all}}|}{B_{\mathrm{all}}},
\qquad
A_{\mathrm{all}}=A_{\mathrm{sp}}+A_{\mathrm{bu}},
\qquad
B_{\mathrm{all}}=B_{\mathrm{sp}}+B_{\mathrm{bu}}.
\label{eq:sec33-cancellation}
\end{equation}
Small $C$ means that opposing projections largely cancel.
All pooled denominators in the reported measurements are
positive.
No layer-specific learning-rate or shape multipliers enter
these sums, so they describe the alignment of polar directions
only.
Neither \textsc{BulkBoost} variant uses $H_t$ online.

\subsection{Results: A Broad Bulk with Positive Aggregate Alignment}
\label{app:band-content-results}

\paragraph{Many weak directions contribute collectively.}
Tables~\ref{tab:sec33-gradient} and
\ref{tab:sec33-momentum} report the raw-gradient and
Nesterov frames, respectively.
Across the three corpora and four checkpoints, the bulk
contains 93.7--97.1\% of pooled directions.
On average, spike directions have much larger absolute
reference-gradient projections: $\Gamma$ ranges from
33.2 to 92.3 in the raw-gradient frame and from 18.7 to
86.7 in the Nesterov frame.
Despite these per-direction differences, the bulk contributes
14.6--26.0\% of pooled signed polar alignment in the
raw-gradient frame and 10.0--63.5\% in the Nesterov frame.
Thus, \emph{weak average per-direction alignment does not imply
a negligible collective contribution}.

\begin{table}[t]
\centering
\small
\setlength{\tabcolsep}{6pt}
\begin{tabular}{@{}lrrrrr@{}}
\toprule
Corpus & Tokens (M) & Bulk (\%) & $\Gamma$ ($\times$)
& $S_{\mathrm{bu}}$ (\%) & $S_{\mathrm{bu}}^{\mathrm{abs}}$ (\%) \\
\midrule
\FineMath{}   &  0.0 & 95.3 & 66.4 & 22.6 & 23.5 \\
              & 25.0 & 97.1 & 85.1 & 25.8 & 28.0 \\
              & 50.0 & 96.6 & 92.3 & 19.7 & 23.6 \\
              & 99.9 & 96.2 & 33.2 & 18.6 & 43.5 \\
\midrule
\CodeParrot{} &  0.0 & 95.5 & 74.8 & 22.6 & 22.3 \\
              & 25.0 & 96.5 & 72.0 & 26.0 & 27.8 \\
              & 50.0 & 96.6 & 64.9 & 25.6 & 30.2 \\
              & 99.9 & 96.6 & 58.8 & 16.6 & 32.9 \\
\midrule
\WikiAr{}     &  0.0 & 93.7 & 86.7 & 14.6 & 14.6 \\
              & 25.0 & 94.8 & 64.8 & 21.4 & 22.0 \\
              & 50.0 & 94.5 & 47.0 & 25.5 & 26.9 \\
              & 99.9 & 94.6 & 40.7 & 18.1 & 30.3 \\
\bottomrule
\end{tabular}
\caption{\textbf{Raw-gradient-frame band content.}
Each row pools 24 matrices and 12,288 singular directions
from a Pythia-70M \muon{} trajectory.
Bulk (\%) is the pooled fraction of directions whose singular
values lie at or below the raw-gradient noise edge
$\tau^G_{t,\ell}$; $\Gamma$ is the ratio of mean absolute
per-direction alignment in spikes to bulk.
$S_{\mathrm{bu}}$ and $S_{\mathrm{bu}}^{\mathrm{abs}}$
are the signed and absolute shares in
\eqref{eq:sec33-shares}, computed from pooled sums.
Every analyzed matrix has positive signed bulk alignment.
Token labels are rounded.}
\label{tab:sec33-gradient}
\end{table}

\begin{table}[t]
\centering
\small
\setlength{\tabcolsep}{6pt}
\begin{tabular}{@{}lrrrrr@{}}
\toprule
Corpus & Tokens (M) & Bulk (\%) & $\Gamma$ ($\times$)
& $S_{\mathrm{bu}}$ (\%) & $S_{\mathrm{bu}}^{\mathrm{abs}}$ (\%) \\
\midrule
\FineMath{}   &  0.0 & 95.3 & 66.4 & 22.6 & 23.5 \\
              & 25.0 & 97.1 & 56.8 & 24.7 & 36.8 \\
              & 50.0 & 96.8 & 78.1 & 10.0 & 28.0 \\
              & 99.9 & 96.6 & 18.9 & 63.5 & 59.8 \\
\midrule
\CodeParrot{} &  0.0 & 95.5 & 74.8 & 22.6 & 22.3 \\
              & 25.0 & 96.5 & 60.7 & 19.4 & 31.3 \\
              & 50.0 & 96.5 & 37.4 & 20.5 & 42.4 \\
              & 99.9 & 96.3 & 18.7 & 45.8 & 58.5 \\
\midrule
\WikiAr{}     &  0.0 & 93.7 & 86.7 & 14.6 & 14.6 \\
              & 25.0 & 95.1 & 72.6 & 18.9 & 21.3 \\
              & 50.0 & 94.7 & 41.5 & 23.5 & 30.3 \\
              & 99.9 & 94.8 & 25.3 & 25.3 & 41.8 \\
\bottomrule
\end{tabular}
\caption{\textbf{Nesterov-frame band content.}
The trajectories, checkpoints, and reference gradients
are paired with Table~\ref{tab:sec33-gradient}, but the
singular frame and noise edge $\tau^M_{t,\ell}$ are those of
\muon{}'s Nesterov input $M_{t,\ell}$.
Each row pools 24 matrices and 12,288 directions, using the
same statistics as Table~\ref{tab:sec33-gradient}.
Matched scalar rescaling explains the agreement at 0.0M.
Every analyzed matrix has positive signed bulk alignment.}
\label{tab:sec33-momentum}
\end{table}

\paragraph{Positive alignment is present in every measured matrix.}
All 24 matrices have positive signed bulk sums at every
checkpoint in both frames and all three corpora:
$3\times4\times2\times24=576$ recorded combinations.
Positive bulk alignment is therefore spread across all layers.

\paragraph{The optimizer's frame changes the relative contributions.}
At 99.9M tokens, the Nesterov bulk supplies 63.5\% of
pooled signed alignment on \FineMath{}, 45.8\% on
\CodeParrot{}, and 25.3\% on \WikiAr{}.
The paired raw-gradient shares are 18.6\%, 16.6\%, and
18.1\%.
A larger share does not necessarily mean a larger absolute
bulk contribution: in Table~\ref{tab:sec33-cancellation},
the final Nesterov bulk sums are smaller than their
raw-gradient counterparts, but the spike sums decrease
by a greater proportion.
The share is also not monotone over training: on \FineMath{},
the Nesterov bulk share is 24.7\% at 25M, 10.0\% at
50M, and 63.5\% at 99.9M.
What holds in every case is positive net bulk alignment;
the share itself varies across corpora and checkpoints.

\paragraph{The contribution remains after signed cancellation.}
Table~\ref{tab:sec33-cancellation} reports raw signed
and absolute sums at the final checkpoint.
For example, the Nesterov bulk on \CodeParrot{}
contributes approximately $0.0957$ to total signed
alignment $0.2088$, while total absolute alignment is
$0.8246$, giving $C\approx0.253$.
The Nesterov cancellation ratios are $0.238$ on
\FineMath{} and $0.587$ on \WikiAr{}.
The positive bulk sums therefore survive the cancellation
between opposing directions.

\begin{table}[t]
\centering
\small
\setlength{\tabcolsep}{5pt}
\begin{tabular}{@{}llrrrrr@{}}
\toprule
Corpus & Frame & $A_{\mathrm{sp}}$ & $A_{\mathrm{bu}}$
& $A_{\mathrm{all}}$ & $B_{\mathrm{all}}$ & $C$ \\
\midrule
\FineMath{}
 & Gradient & 0.5946 & 0.1363 & 0.7308 & 1.2160 & 0.601 \\
 & Nesterov & 0.0749 & 0.1300 & 0.2049 & 0.8594 & 0.238 \\
\midrule
\CodeParrot{}
 & Gradient & 1.0155 & 0.2028 & 1.2183 & 1.6182 & 0.753 \\
 & Nesterov & 0.1132 & 0.0957 & 0.2088 & 0.8246 & 0.253 \\
\midrule
\WikiAr{}
 & Gradient & 0.9131 & 0.2013 & 1.1144 & 1.3890 & 0.802 \\
 & Nesterov & 0.4201 & 0.1422 & 0.5623 & 0.9577 & 0.587 \\
\bottomrule
\end{tabular}
\caption{\textbf{Signed contributions and cancellation at 99.9M tokens.}
$A_{\mathrm{sp}}$ and $A_{\mathrm{bu}}$ are pooled signed
alignments; $A_{\mathrm{all}}$ is their sum and
$B_{\mathrm{all}}$ sums absolute alignments.
$C=|A_{\mathrm{all}}|/B_{\mathrm{all}}$ measures the fraction
remaining after signed cancellation.
Every row pools 24 matrices, and every matrix has positive
signed bulk alignment.
Displayed sums are rounded independently.}
\label{tab:sec33-cancellation}
\end{table}

\subsection{Implications for \textsc{BulkBoost}}
\label{app:band-content-scope}

\paragraph{Motivation for the fixed-fraction sweep.}
The complementary spike band contains only 2.9--6.3\%
of the pooled singular directions in these measurements.
This broad-bulk/small-spike structure motivates testing
small fixed spike fractions, alongside larger ones, in
Table~\ref{tab:fracsweep}
(Appendix~\ref{app:omitted-ablation}).
The tested grid is $p\in\{0.015,0.02,0.05,0.15,0.20\}$:
5\% lies within the observed pooled range and is the
default for \textsc{BB-Static}, 1.5\% and 2\% examine
smaller fractions, and 15\% and 20\% examine larger spike
bands.
The diagnostics inform the grid; we did not convert the
noise edge into five calibrated fractions.
The loss comparisons in Table~\ref{tab:fracsweep} decide
which fractions perform best.

The 2.9--6.3\% range is computed after pooling matrices
within each corpus--checkpoint--frame combination.
Individual matrices can fall outside it, and it is not an
estimate of the best fixed rank.
\textsc{BB-Static} uses $r_{t,\ell}=\lceil p d_\ell\rceil$
with a fixed fraction for the whole run, whereas
\textsc{BB-Adaptive} selects the spike band separately for
each matrix and step.
The two rules give different partitions even when their
pooled band sizes are similar.

\paragraph{Link to the two-band design.}
Theorem~\ref{thm:exact-decomposition} shows that a shared bulk
gain needs only the polar factor plus a correction in the
spike subspace.
The band-content measurements show why the bulk matters in
this allocation: it carries positive reference alignment in
every matrix.
Appendix~\ref{app:post-update-loss} evaluates the post-update
effect of reweighting, and Section~\ref{sec:experiments}
evaluates training performance and computational cost.

\section{Pre-Update Alignment and Post-Update Loss}
\label{app:band-content-immediate}

This appendix connects the band-content measurements in
Appendix~\ref{app:band-content} to the effect of spectral
reweighting on validation loss.
We first compute the pre-update reference-alignment
preference and then evaluate the model produced by a
controlled update.
The single-update measurements show that finite
bulk-emphasizing reweighting can improve on the exact flat
update $Q_t$, even when paired pre-update predictions favor
less bulk weight.

\subsection{Pre-Update Reference Alignment}
\label{app:pre-update-alignment}

\paragraph{A norm-matched alignment score.}
Use the same singular frames, reference gradients, and
partitions as Appendix~\ref{app:band-content}, holding
them fixed while varying $\rho$.
Suppress the checkpoint and frame indices.
For matrix $\ell$, let $r_\ell=|\mathcal I_{\mathrm{sp},\ell}|$,
$k_\ell=d_\ell-r_\ell$, and $q_\ell=r_\ell/d_\ell$,
and set
$A_{b,\ell}=\sum_{i\in\mathcal I_{b,\ell}}\alpha_{\ell,i}$.
The exact norm-matched two-level direction has alignment
\begin{equation}
J_\ell(\rho)
=c_{\rho,r_\ell}
\bigl(A_{\mathrm{sp},\ell}+\rho A_{\mathrm{bu},\ell}\bigr),
\qquad
c_{\rho,r_\ell}
=\bigl[q_\ell+\rho^2(1-q_\ell)\bigr]^{-1/2},
\label{eq:sec33-immediate-score}
\end{equation}
whenever the profile is nonzero.
A larger score predicts a larger first-order decrease
in the reference objective under a common step size.
This gives the pre-update comparison for the loss
measurements below.

For nonempty bands, differentiating at $\rho=1$ gives
\begin{equation}
\left.\frac{dJ_\ell(\rho)}{d\rho}\right|_{\rho=1}
=\frac{r_\ell k_\ell}{d_\ell}
\left(
\overline\alpha_{\mathrm{bu},\ell}
-\overline\alpha_{\mathrm{sp},\ell}
\right),
\label{eq:sec33-immediate-criterion}
\end{equation}
where $\overline\alpha_{\mathrm{bu},\ell}
=A_{\mathrm{bu},\ell}/k_\ell$ and
$\overline\alpha_{\mathrm{sp},\ell}
=A_{\mathrm{sp},\ell}/r_\ell$.
This uses the same norm-matched differentiation as
Theorem~\ref{thm:continuation-decomposition}, with
pre-update reference projections in place of continuation
utilities.
The local preference therefore depends on the difference in
\emph{mean per-direction} alignment between bands.
Appendix~\ref{app:band-content} instead measures their
\emph{aggregate} contributions under the flat profile.

Normalize each matrix separately before pooling:
\begin{equation}
J(\rho)=\sum_\ell J_\ell(\rho),
\qquad
R_{\mathrm{mat}}(\rho)=\frac{J(\rho)}{J(1)}.
\label{eq:sec33-matrixwise-score-ratio}
\end{equation}
As in Appendix~\ref{app:band-content}, no learning-rate
or shape multipliers are included.
The reference $\rho=1$ is the flat profile in the same frame;
in the Nesterov frame, it gives \muon{}'s idealized update
$Q_t$.
The measured $J(1)$ is positive in every pooled comparison.

\paragraph{The pre-update score favors less bulk weight.}
The aggregate derivative $J'(1)$ is negative in all
24 corpus--checkpoint--frame combinations.
At the matrix level, the derivative is negative in 564
of 576 recorded combinations and positive in the remaining 12.
The pooled signed per-direction ratio
$(A_{\mathrm{bu}}/N_{\mathrm{bu}})/
(A_{\mathrm{sp}}/N_{\mathrm{sp}})$ ranges from
approximately $0.0037$ to $0.0617$.
Unlike $\Gamma$ in Appendix~\ref{app:band-content}, this
ratio keeps the signs of the projections.
The derivatives and reweighted scores are computed
matrix by matrix; the pooled ratio only summarizes them.
Across the recorded comparisons, $R_{\mathrm{mat}}(2)$
ranges from 0.556 to 0.828, and
$R_{\mathrm{mat}}(4)$ from 0.330 to 0.738.
Table~\ref{tab:sec33-criterion} reports the final checkpoint.

\begin{table}[t]
\centering
\small
\setlength{\tabcolsep}{5pt}
\begin{tabular}{@{}llrrr@{}}
\toprule
Corpus & Frame &
$\overline\alpha_{\mathrm{bu}}/\overline\alpha_{\mathrm{sp}}$
& $R_{\mathrm{mat}}(2)$ & $R_{\mathrm{mat}}(4)$ \\
\midrule
\FineMath{}   & Gradient & 0.009 & 0.601 & 0.397 \\
              & Nesterov & 0.062 & 0.828 & 0.738 \\
\CodeParrot{} & Gradient & 0.007 & 0.590 & 0.381 \\
              & Nesterov & 0.032 & 0.739 & 0.604 \\
\WikiAr{}     & Gradient & 0.012 & 0.602 & 0.395 \\
              & Nesterov & 0.019 & 0.638 & 0.450 \\
\bottomrule
\end{tabular}
\caption{\textbf{Pre-update reference-alignment trade-off
at 99.9M tokens.}
The per-direction ratio uses pooled signed means
$\overline\alpha_b=A_b/N_b$; the ratio $\Gamma$ in
Appendix~\ref{app:band-content} uses absolute alignments.
$R_{\mathrm{mat}}(\rho)$ applies each matrix's own
norm-matching factor before summing, as in
\eqref{eq:sec33-matrixwise-score-ratio}.
Values below one indicate lower alignment than the
flat profile in the same frame.
Each row uses the same 24 matrices and reference gradient
as the preceding diagnostics, without learning-rate or
shape multipliers.}
\label{tab:sec33-criterion}
\end{table}

\paragraph{Interpreting the norm-matched reweighting.}
Increasing $\rho$ moves the fixed direction norm from the
spikes toward the bulk.
For a matrix with 95\% bulk directions, the applied bulk
gains at $\rho=2$ and $4$ are approximately $1.019$ and
$1.024$, whereas the spike gains are $0.510$ and $0.256$.
In the pre-update score, the loss of spike alignment
outweighs the additional bulk alignment, which explains the
preference for smaller $\rho$.
Relative bulk emphasis and relative spike suppression
therefore describe the same norm-matched intervention.
The next subsection evaluates its effect directly on
post-update validation loss.

\subsection{Single-Update Loss Measurements}
\label{app:post-update-loss}

\paragraph{Question and controlled comparison.}
\emph{Can a finite increase in the bulk-to-spike gain lower
validation loss after one update, even when a pre-update
linearized prediction favors less bulk weight?}
We use the Pythia-70M \muon{} trajectories of
Appendix~\ref{app:band-content} on \FineMath{},
\CodeParrot{}, and \WikiAr{}, immediately before updates
763 and 1526, after 24,969,216 and 49,971,200 training
tokens, respectively.
The protocol uses batches of 32 sequences of length 1024,
$\beta=0.95$, $\gamma=0.95$, $\kappa_{\mathrm{edge}}=1$,
global clipping at norm 1, and peak learning rate
$5\times10^{-4}$ for both the matrix and auxiliary optimizers.
Validation uses 128 fixed held-out windows.

At each checkpoint, the branches share the current
Nesterov singular frames and the noise-edge partitions of
Appendix~\ref{app:band-content}.
We vary one common gain across the 24 hidden weight
matrices, using the exact-SVD target and separate
Frobenius norm matching for each matrix.
The tested gains are $\rho\in\{0.5,0.75,1,1.25,1.5\}$,
with two duplicate $\rho=1$ branches.
The comparator is the exact flat update $Q_t$ at the same
input, computed from the same SVD.
All branches compute their current update from the same
exact SVD, so they differ only in $\rho$.
We report $K=0$: validation immediately after the current
update, without additional training steps.
The duplicate $\rho=1$ branches have identical recorded
losses.

\paragraph{Direct improvement and allocation preference.}
Let $F_{t,K}(\rho)$ denote the expected evaluation loss,
given the current state, under a current common gain $\rho$
and $K$ additional steps following a fixed continuation
policy.
Write $\widehat F_{t,0}(\rho)$ for the measured post-update
validation loss on the fixed windows.
For $\epsilon\in\{0.25,0.5\}$, define
\begin{equation}
\widehat B_{\epsilon}
=\widehat F_{t,0}(1)-\widehat F_{t,0}(1+\epsilon),
\qquad
\widehat S_{\epsilon}
=\frac{\widehat F_{t,0}(1-\epsilon)
       -\widehat F_{t,0}(1+\epsilon)}{2\epsilon}.
\label{eq:post-update-measurements}
\end{equation}
The loss reduction $\widehat B_\epsilon$ directly measures
improvement over the flat update $\rho=1$: positive values
mean that the bulk-emphasizing update $\rho=1+\epsilon$
gives lower validation loss.
The raw central finite difference $\widehat S_\epsilon$
instead compares the bulk-emphasizing and bulk-suppressing
branches, $\rho=1+\epsilon$ and $\rho=1-\epsilon$.
A positive value indicates lower loss for the
bulk-emphasizing branch.
Together, these statistics give the benefit over the flat
update and the preference between opposing reweightings.
Table~\ref{tab:post-update-loss} reports both quantities
for all six corpus--checkpoint settings.

\begin{table}[t]
\centering
\small
\setlength{\tabcolsep}{4pt}
\begin{tabular}{@{}lrrrrr@{}}
\toprule
& & \multicolumn{2}{c}{Loss reduction $\widehat B_{\epsilon}$}
  & \multicolumn{2}{c}{Slope $\widehat S_{\epsilon}$} \\
& & \multicolumn{2}{c}{($10^{-4}$)}
  & \multicolumn{2}{c}{($10^{-3}$)} \\
\cmidrule(lr){3-4}\cmidrule(l){5-6}
Corpus & Tokens (M) & $\epsilon=0.25$ & $\epsilon=0.5$
& $\epsilon=0.25$ & $\epsilon=0.5$ \\
\midrule
\FineMath{}   & 25.0 & +2.524 & +3.725 & +1.709 & +2.774 \\
\CodeParrot{} & 25.0 & +1.842 & +2.574 & +1.350 & +2.320 \\
\WikiAr{}     & 25.0 & +6.593 & +10.226 & +4.082 & +5.944 \\
\midrule
\FineMath{}   & 50.0 & -0.604 & -1.191 & -0.177 & +0.032 \\
\CodeParrot{} & 50.0 & +0.844 & +1.262 & +0.554 & +0.866 \\
\WikiAr{}     & 50.0 & +1.530 & +2.318 & +0.978 & +1.471 \\
\bottomrule
\end{tabular}
\caption{\textbf{Single-update validation-loss improvement
($K=0$).}
Both $\rho=1.25$ and $\rho=1.5$ lower post-update validation
loss relative to the exact flat update $\rho=1$ in five of
six tested states.
Each row is a Pythia-70M corpus--checkpoint state with
a common gain across 24 hidden matrices and separate
per-matrix Frobenius norm matching.
$\widehat B_\epsilon$ measures the loss reduction relative
to $\rho=1$, whereas $\widehat S_\epsilon$ compares the
bulk-emphasizing and bulk-suppressing branches through
\eqref{eq:post-update-measurements}.
Loss reductions and raw central finite-difference slopes
are reported in units of $10^{-4}$ and $10^{-3}$,
respectively.
All six states are shown; token labels are rounded,
and the duplicate $\rho=1$ branches have identical
recorded losses.}
\label{tab:post-update-loss}
\end{table}

\paragraph{Bulk emphasis improves the post-update model.}
Both $\rho=1.25$ and $\rho=1.5$ lower validation loss
relative to $\rho=1$ in five of six corpus--checkpoint
settings: all three corpora at 25M tokens, and
\CodeParrot{} and \WikiAr{} at 50M tokens.
The positive loss reductions $\widehat B_\epsilon$ measure
these improvements directly.
\FineMath{} at 50M is the exception, where the flat update
has lower loss than both bulk-emphasizing updates.

The improvements occur despite paired pre-update
linearized predictions favoring less bulk weight at
both perturbation widths in every tested state.
These predictions are recorded for the intervention
branches and differ from the pooled score $J(\rho)$ in
Appendix~\ref{app:pre-update-alignment}.
Thus, finite, norm-matched spectral reweighting can improve
the post-update model even when the paired pre-update
approximation favors the opposite allocation.
The benefit is already visible at $K=0$, without
additional continuation steps.

\paragraph{Connection to Theorem~\ref{thm:continuation-decomposition}.}
The theorem varies the gains of one matrix, whereas this
experiment changes a common $\rho$ across matrices.
For matrix $\ell$, let $r_{t,\ell}$ be its spike rank,
$k_{t,\ell}=d_\ell-r_{t,\ell}$, and
$\Delta_{\ell,t,K}$ its bulk-minus-spike mean continuation
utility when every matrix uses the flat profile.
Each matrix's utilities are partial derivatives with the
other matrices' current gains held at one.
Define $\omega_{t,\ell}=r_{t,\ell}k_{t,\ell}/d_\ell$,
the matrix-specific version of the prefactor $rk/d$ in
the theorem.
Under the corresponding joint differentiability assumption,
the chain rule gives
\begin{equation}
S_{t,K}:=-F_{t,K}'(1)
=\sum_{\ell:\,0<r_{t,\ell}<d_\ell}
\omega_{t,\ell}\Delta_{\ell,t,K}.
\label{eq:joint-continuation-slope}
\end{equation}
Indeed, the derivative along the common-gain curve is
the sum of its matrix-wise partial derivatives, and
Theorem~\ref{thm:continuation-decomposition} gives
$\omega_{t,\ell}\Delta_{\ell,t,K}$ for each nonempty
partition.
Empty-band matrices have constant normalized targets and
contribute zero.
The identity applies to the joint loss, including
interactions across matrices.

When $\sum_\ell\omega_{t,\ell}>0$, the exact ratio
$S_{t,K}/\sum_\ell\omega_{t,\ell}$ is a weighted mean
of the matrix-wise continuation-utility gaps.
Its finite-difference counterpart is
$\widehat S_\epsilon/\sum_\ell\omega_{t,\ell}$.
Table~\ref{tab:post-update-loss} reports the raw
central finite differences $\widehat S_\epsilon$ in units
of $10^{-3}$.
The theorem characterizes the derivative at $\rho=1$;
$\widehat S_\epsilon$ is its finite-difference counterpart,
and $\widehat B_\epsilon$ directly tests improvement at the
chosen gains.
These diagnostics measure the joint response to a shared
gain.
They do not measure individual utilities or
$R_{\mathrm{bt}}$.

\paragraph{Continuation utility evaluates the updated model.}
Even at $K=0$, the evaluation point has changed.
For an additive intervention with fixed effective step
multiplier $\eta_{t,\ell}$, let $\theta_1^+$ be the complete
parameter state after the flat update, and let
$\mathcal E$ be the evaluation objective defining the loss.
Differentiation with respect to an applied singular-direction
gain gives
\begin{equation}
u_{\ell,i,t,0}
=\eta_{t,\ell}
\left\langle
\nabla_{W_\ell}\mathcal E(\theta_1^+),
u_{t,\ell,i}v_{t,\ell,i}^{\top}
\right\rangle_F.
\label{eq:post-update-utility-identity}
\end{equation}
The singular vectors remain those of the fixed
pre-intervention Nesterov input, but the evaluation
gradient is taken \emph{after} the flat update.
The chain rule uses
\[
\frac{\partial W_\ell^+}{\partial w_{\ell,i}}
=-\eta_{t,\ell}u_{t,\ell,i}v_{t,\ell,i}^{\top},
\]
and continuation utility is the negative loss derivative.
The multiplier $\eta_{t,\ell}$ includes the learning rate
and fixed shape scaling.
Thus, ``first order'' describes changes in the applied
gains around the flat update, with the actual
parameter-update step size held fixed.
This identifies the evaluation point relevant to local
allocation preferences and explains the difference from
pre-update reference alignment.

\paragraph{What the experiment establishes.}
\emph{The experiment provides direct evidence that finite,
norm-matched two-band reweighting can improve the model
produced by a single update, already at $K=0$}.
The measurements cover $\rho=1.25$ and $1.5$ at checkpoints
near 25M and 50M tokens; Appendix~\ref{app:band-content}
also measures band content at initialization and 99.9M.
Section~\ref{sec:experiments} separately evaluates the
full training rule with the shared default $\rho=4$.
Because norm matching couples relative bulk emphasis with
spike suppression, the measured benefit concerns their
combined reweighting.

\section{Proofs for Sections~\ref{sec:method} and~\ref{sec:theory}}
\label{app:theory-proofs}

This appendix proves Theorems~\ref{thm:continuation-decomposition}
and~\ref{thm:exact-decomposition}, and states and proves
Corollary~\ref{cor:exact-normalization}.
It also proves the half-batch and Nesterov noise identities used in
Section~\ref{sec:split-edge} (Proposition~\ref{prop:edge-calibration}),
and treats the Gaussian model behind the noise edge separately.
Appendix~\ref{app:numerical-norm-matching} derives numerical norm
matching.

\subsection{Proof of Theorem~\ref{thm:continuation-decomposition}}
\label{app:proof-continuation-decomposition}
\paragraph{Theorem~\ref{thm:continuation-decomposition} (restated).}
Fix the current training state, singular frame, nonempty two-band
partition, and continuation policy as in
Section~\ref{sec:theory-continuation}.
The intervention varies the applied gains of one matrix's current
update; the current updates of all other parameters and the rest of
the current step stay fixed.
Assume $\mathcal L^\pi_{t,K}$ is differentiable at $\mathbf 1_d$.
Along the two-level family,
\begin{equation}
-\left.\frac{d}{d\rho}
\mathcal L^\pi_{t,K}(w(\rho))\right|_{\rho=1}
=
\frac{rk}{d}\Delta.
\label{eq:two-band-continuation-derivative}
\end{equation}
Increasing $\rho$ from $1$ gives a strictly positive first-order
improvement if and only if $\Delta>0$.
If $\Delta=0$, the first-order test is inconclusive.

Small gain changes at the flat profile lie in the tangent space
$\mathcal T=\{\delta\in\mathbb R^d:
\mathbf 1_d^\top\delta=0\}$.
If every gain may change separately, the best first-order
improvement rate is
\begin{equation}
\max_{\substack{\delta\in\mathcal T\\\|\delta\|_2=1}}
u^\top\delta
=
\|u_\perp\|_2.
\label{eq:maximal-spectral-improvement}
\end{equation}
Thus the flat profile is first-order stationary in this
fixed-frame, norm-matched class if and only if $u_1=\cdots=u_d$.

If the change must be constant within each band, the best rate is
$\sqrt{rk/d}\,|\Delta|$.
For $u_\perp\neq0$, two bands therefore reach the fraction
\begin{equation}
R_{\mathrm{bt}}
=
\frac{\sqrt{rk/d}\,|\Delta|}{\|u_\perp\|_2}
\in[0,1]
\label{eq:between-band-capture-ratio}
\end{equation}
of the best rate.
Both rates are per unit Euclidean change in $w$, and
$R_{\mathrm{bt}}$ allows $\rho$ to move in either direction.

\paragraph{Between-band and within-band components.}
Use the fixed state, frame, partition, and utilities from
Section~\ref{sec:theory-continuation}. Define
\begin{equation}
\begin{aligned}
(u_{\mathrm{bt}})_i
&=\begin{cases}
\overline u_{\mathrm{sp}}-\overline u,
& i\in\mathcal I_{\mathrm{sp}},\\
\overline u_{\mathrm{bu}}-\overline u,
& i\in\mathcal I_{\mathrm{bu}},
\end{cases}
\\
(u_{\mathrm{wi}})_i
&=\begin{cases}
u_i-\overline u_{\mathrm{sp}},
& i\in\mathcal I_{\mathrm{sp}},\\
u_i-\overline u_{\mathrm{bu}},
& i\in\mathcal I_{\mathrm{bu}}.
\end{cases}
\end{aligned}
\label{eq:between-within-definitions}
\end{equation}
The decomposition $u_\perp=u_{\mathrm{bt}}+u_{\mathrm{wi}}$
is orthogonal and satisfies
\begin{equation}
\begin{aligned}
\|u_\perp\|_2^2
&=\frac{rk}{d}\Delta^2
+\sum_{i\in\mathcal I_{\mathrm{sp}}}
(u_i-\overline u_{\mathrm{sp}})^2
+\sum_{i\in\mathcal I_{\mathrm{bu}}}
(u_i-\overline u_{\mathrm{bu}})^2.
\end{aligned}
\label{eq:utility-energy-decomposition}
\end{equation}
In particular, $\|u_{\mathrm{bt}}\|_2=\sqrt{rk/d}\,|\Delta|$, so for
$u_\perp\neq0$ the capture ratio can also be written as
$R_{\mathrm{bt}}=\|u_{\mathrm{bt}}\|_2/\|u_\perp\|_2$.
Its square is the share of squared tangent utility that lies
between the bands; the rest lies within the bands.
This share says nothing directly about the size of a finite loss
reduction.
The proof below establishes these identities together with the
theorem.

\begin{proof}
Write $\mathcal L=\mathcal L^\pi_{t,K}$ for brevity.
As specified in Section~\ref{sec:theory-continuation},
$\mathcal L$ is defined on an open neighborhood of
$\mathbf 1_d$ in the positive orthant and is differentiable
at $\mathbf 1_d$, with
$u=-\nabla\mathcal L(\mathbf 1_d)$.
All gain comparisons below lie on the norm-matched sphere.
The singular frame, current training state, partition, and
continuation policy are fixed throughout the argument.
The policy is fixed as a rule: subsequent gradients and optimizer
states may respond to the current intervention.
We proceed in four steps.

\paragraph{Step 1: first-order rates on the norm-matched sphere.}
The feasible set is
\[
\mathcal S_d=\{w\in[0,\infty)^d:\|w\|_2^2=d\}.
\]
Differentiating the norm constraint along a differentiable
feasible curve through $\mathbf 1_d$ shows that its velocity
$\delta$ satisfies $\mathbf 1_d^\top\delta=0$.
Conversely, for any unit vector $\delta$ with
$\mathbf 1_d^\top\delta=0$, define
\[
w_\delta(\varepsilon)
=
\sqrt d\,
\frac{\mathbf 1_d+\varepsilon\delta}
     {\|\mathbf 1_d+\varepsilon\delta\|_2}.
\]
Since
\[
\|\mathbf 1_d+\varepsilon\delta\|_2^2=d+\varepsilon^2,
\]
this curve lies in $\mathcal S_d$ for sufficiently small
$|\varepsilon|$, has positive coordinates there, and satisfies
\[
w_\delta(0)=\mathbf 1_d,
\qquad
w_\delta'(0)=\delta.
\]
Thus the tangent space is
$\mathcal T=\{\delta:\mathbf 1_d^\top\delta=0\}$,
and every unit tangent direction is realized by such a curve.

Differentiability gives
\[
\mathcal L(w_\delta(\varepsilon))
=
\mathcal L(\mathbf 1_d)
-\varepsilon u^\top\delta+o(|\varepsilon|).
\]
Hence $u^\top\delta$ is the first-order improvement rate
per unit gain-vector displacement. For tangent $\delta$,
\[
u^\top\delta
=
(u-\overline u\mathbf 1_d)^\top\delta
=
u_\perp^\top\delta.
\]

\paragraph{Step 2: orthogonal between-band and within-band decomposition.}
For $i\in\mathcal I_{\mathrm{sp}}$,
\[
(u_{\mathrm{bt}})_i+(u_{\mathrm{wi}})_i
=
(\overline u_{\mathrm{sp}}-\overline u)
+(u_i-\overline u_{\mathrm{sp}})
=
u_i-\overline u.
\]
The same identity holds on $\mathcal I_{\mathrm{bu}}$, giving
$u_\perp=u_{\mathrm{bt}}+u_{\mathrm{wi}}$.
Moreover,
\[
\sum_{i\in\mathcal I_{\mathrm{sp}}}
(u_i-\overline u_{\mathrm{sp}})=0,
\qquad
\sum_{i\in\mathcal I_{\mathrm{bu}}}
(u_i-\overline u_{\mathrm{bu}})=0.
\]
Therefore $u_{\mathrm{wi}}$ is orthogonal to both band
indicator vectors. Since $u_{\mathrm{bt}}$ is in their span,
\[
u_{\mathrm{bt}}\perp u_{\mathrm{wi}},
\qquad
\|u_\perp\|_2^2
=
\|u_{\mathrm{bt}}\|_2^2+\|u_{\mathrm{wi}}\|_2^2.
\]
By definition,
\[
\|u_{\mathrm{wi}}\|_2^2
=
\sum_{i\in\mathcal I_{\mathrm{sp}}}
(u_i-\overline u_{\mathrm{sp}})^2
+
\sum_{i\in\mathcal I_{\mathrm{bu}}}
(u_i-\overline u_{\mathrm{bu}})^2.
\]
Let $\Delta=\overline u_{\mathrm{bu}}-\overline u_{\mathrm{sp}}$.
Using $d=r+k$ and
$\overline u=(r\overline u_{\mathrm{sp}}+
k\overline u_{\mathrm{bu}})/d$, we obtain
\[
\overline u_{\mathrm{sp}}-\overline u=-\frac{k}{d}\Delta,
\qquad
\overline u_{\mathrm{bu}}-\overline u=\frac{r}{d}\Delta.
\]
Consequently,
\begin{align*}
\|u_{\mathrm{bt}}\|_2^2
&=
r\frac{k^2}{d^2}\Delta^2
+k\frac{r^2}{d^2}\Delta^2\\
&=
\frac{rk(r+k)}{d^2}\Delta^2
=
\frac{rk}{d}\Delta^2.
\end{align*}
Combining these identities proves~\eqref{eq:utility-energy-decomposition}.

\paragraph{Step 3: maximal first-order rate and stationarity.}
For every unit tangent direction, Cauchy--Schwarz gives
\[
u^\top\delta=u_\perp^\top\delta\leq\|u_\perp\|_2.
\]
If $u_\perp\neq0$, equality holds at
\[
\delta^\star=\frac{u_\perp}{\|u_\perp\|_2},
\]
which is a feasible unit tangent direction because
$\mathbf 1_d^\top u_\perp=0$.
If $u_\perp=0$, every tangent direction has zero rate.
It follows that
\[
\max_{\substack{\delta\in\mathcal T\\\|\delta\|_2=1}}
u^\top\delta
=
\|u_\perp\|_2.
\]
This proves~\eqref{eq:maximal-spectral-improvement}.
In particular, the flat profile is first-order stationary
with respect to the fixed-frame, norm-matched interventions
if and only if
\[
u_\perp=0
\quad\Longleftrightarrow\quad
u\in\operatorname{span}\{\mathbf 1_d\}
\quad\Longleftrightarrow\quad
u_1=\cdots=u_d.
\]

\paragraph{Step 4: the two-level direction and capture ratio.}
A band-constant tangent direction has the form
$a\mathbf 1_{\mathrm{sp}}+b\mathbf 1_{\mathrm{bu}}$
with $ra+kb=0$. Since $r,k>0$, these directions form
the line spanned by
\[
v=
\frac{r}{d}\mathbf 1_{\mathrm{bu}}
-\frac{k}{d}\mathbf 1_{\mathrm{sp}},
\qquad
\|v\|_2^2=\frac{rk}{d}>0.
\]
For the two-level curve
$w(\rho)=c_{\rho,r}
(\mathbf 1_{\mathrm{sp}}+\rho\mathbf 1_{\mathrm{bu}})$,
the exact normalizer satisfies
\[
c_{\rho,r}=\sqrt{\frac{d}{r+k\rho^2}},
\qquad
c_{1,r}=1,
\qquad
\left.\frac{d}{d\rho}c_{\rho,r}\right|_{\rho=1}
=-\frac{k}{d}.
\]
Therefore
\begin{align*}
w'(1)
&=
-\frac{k}{d}
(\mathbf 1_{\mathrm{sp}}+\mathbf 1_{\mathrm{bu}})
+\mathbf 1_{\mathrm{bu}}\\
&=
-\frac{k}{d}\mathbf 1_{\mathrm{sp}}
+\frac{r}{d}\mathbf 1_{\mathrm{bu}}
=v.
\end{align*}
The chain rule now yields
\begin{align*}
-\left.\frac{d}{d\rho}\mathcal L(w(\rho))\right|_{\rho=1}
&=u^\top v\\
&=
\frac{r}{d}\sum_{i\in\mathcal I_{\mathrm{bu}}}u_i
-\frac{k}{d}\sum_{i\in\mathcal I_{\mathrm{sp}}}u_i\\
&=
\frac{rk}{d}
(\overline u_{\mathrm{bu}}-\overline u_{\mathrm{sp}})
=
\frac{rk}{d}\Delta.
\end{align*}
This proves~\eqref{eq:two-band-continuation-derivative}.
More explicitly, as $h\to0$,
\[
\mathcal L(w(1+h))
=
\mathcal L(\mathbf 1_d)
-h\frac{rk}{d}\Delta+o(|h|).
\]
If $\Delta>0$, every sufficiently small positive $h$
reduces the loss. If $\Delta<0$, every sufficiently small
positive $h$ increases it, and a sufficiently small negative
$h$ reduces it. When $\Delta=0$, the first derivative
vanishes and does not determine whether a higher-order
improvement occurs. Thus the strict sign criterion is
an if-and-only-if statement about first-order improvement.

For the capture ratio, observe that
\[
\operatorname{proj}_{\operatorname{span}\{v\}}u
=
\frac{u^\top v}{\|v\|_2^2}v
=
\Delta v
=
u_{\mathrm{bt}}.
\]
The two unit directions in this line are
$\pm v/\|v\|_2$, and hence
\[
\max_{\substack{\delta\in\operatorname{span}\{v\}\\
                 \|\delta\|_2=1}}
u^\top\delta
=
\frac{|u^\top v|}{\|v\|_2}
=
\sqrt{\frac{rk}{d}}\,|\Delta|
=
\|u_{\mathrm{bt}}\|_2.
\]
Both signs can be realized along the two-level family:
the curves $w(1\pm\varepsilon/\|v\|_2)$ have unit
velocities $\pm v/\|v\|_2$ at $\varepsilon=0$.
Thus, when $u_\perp\neq0$, the ratio of the restricted
and unrestricted unit-tangent optima is
\[
R_{\mathrm{bt}}
=
\frac{\sqrt{rk/d}\,|\Delta|}{\|u_\perp\|_2}
=
\frac{\|u_{\mathrm{bt}}\|_2}{\|u_\perp\|_2},
\qquad
R_{\mathrm{bt}}^2
=
\frac{\|u_{\mathrm{bt}}\|_2^2}
     {\|u_{\mathrm{bt}}\|_2^2+\|u_{\mathrm{wi}}\|_2^2}
\in[0,1].
\]
This proves~\eqref{eq:between-band-capture-ratio}.
The ratio is not defined when $u_\perp=0$, since both
first-order optima then vanish.
\end{proof}

\paragraph{What finer shaping can add.}
Two bands can use $u_{\mathrm{bt}}$, the differences between band
means.
They cannot use $u_{\mathrm{wi}}$, the variation within bands.
By Step~1, every differentiable curve of norm-matched gains through
$\mathbf 1_d$, however fine-grained, has its velocity in
$\mathcal T$.
By Step~3, its first-order rate per unit velocity is therefore at
most $\|u_\perp\|_2$.
The best two-band change reaches $\|u_{\mathrm{bt}}\|_2$.
Finer shaping can thus add at most
\[
\|u_\perp\|_2-\|u_{\mathrm{bt}}\|_2
=(1-R_{\mathrm{bt}})\|u_\perp\|_2
\leq\|u_{\mathrm{wi}}\|_2
\]
to the first-order rate.
In particular, if utilities are constant within each band, finer
shaping cannot improve on two bands to first order.
The discussion after Theorem~\ref{thm:continuation-decomposition}
covers the remaining caveats.

\subsection{Proof of Theorem~\ref{thm:exact-decomposition}}
\label{app:proof-exact-decomposition}

\begin{proof}
After a common reordering of the singular directions, write
\[
U=[U_{\mathrm{sp}}\;U_{\mathrm{bu}}],
\qquad
V=[V_{\mathrm{sp}}\;V_{\mathrm{bu}}].
\]
Then
\[
Q=UV^\top
=
U_{\mathrm{sp}}V_{\mathrm{sp}}^\top
+U_{\mathrm{bu}}V_{\mathrm{bu}}^\top.
\]
The prescribed unnormalized update is
\[
\widetilde D(a,\rho)
=
U_{\mathrm{sp}}\operatorname{diag}(a)V_{\mathrm{sp}}^\top
+\rho U_{\mathrm{bu}}V_{\mathrm{bu}}^\top.
\]
Substituting
$U_{\mathrm{bu}}V_{\mathrm{bu}}^\top
=Q-U_{\mathrm{sp}}V_{\mathrm{sp}}^\top$ gives
\begin{align*}
\widetilde D(a,\rho)
&=
U_{\mathrm{sp}}\operatorname{diag}(a)V_{\mathrm{sp}}^\top
+\rho\bigl(Q-U_{\mathrm{sp}}V_{\mathrm{sp}}^\top\bigr)\\
&=
\rho Q
+U_{\mathrm{sp}}
\bigl(\operatorname{diag}(a)-\rho I_r\bigr)
V_{\mathrm{sp}}^\top\\
&=
\rho Q
+U_{\mathrm{sp}}
\operatorname{diag}(a-\rho\mathbf 1_r)
V_{\mathrm{sp}}^\top.
\end{align*}
This is~\eqref{eq:theory-polar-decomposition}.
The correction factors through an $r$-dimensional space, so
\[
\operatorname{rank}\bigl(\widetilde D(a,\rho)-\rho Q\bigr)
\leq r.
\]
For $a=\mathbf 1_r$, it follows that
\[
\widetilde D_\rho
=
\rho Q-(\rho-1)U_{\mathrm{sp}}V_{\mathrm{sp}}^\top.
\]

Define
\[
P_{\mathrm{sp}}=U_{\mathrm{sp}}U_{\mathrm{sp}}^\top,
\qquad
R_{\mathrm{sp}}=V_{\mathrm{sp}}V_{\mathrm{sp}}^\top.
\]
Orthonormality within each singular-vector matrix gives
\begin{align*}
P_{\mathrm{sp}}Q
&=
U_{\mathrm{sp}}U_{\mathrm{sp}}^\top
\bigl(
U_{\mathrm{sp}}V_{\mathrm{sp}}^\top
+U_{\mathrm{bu}}V_{\mathrm{bu}}^\top
\bigr)\\
&=U_{\mathrm{sp}}V_{\mathrm{sp}}^\top,
\end{align*}
and
\begin{align*}
QR_{\mathrm{sp}}
&=
\bigl(
U_{\mathrm{sp}}V_{\mathrm{sp}}^\top
+U_{\mathrm{bu}}V_{\mathrm{bu}}^\top
\bigr)
V_{\mathrm{sp}}V_{\mathrm{sp}}^\top\\
&=U_{\mathrm{sp}}V_{\mathrm{sp}}^\top.
\end{align*}
Both identities hold for either rectangular orientation, giving
\begin{equation}
\widetilde D_\rho=
\begin{cases}
\rho Q-(\rho-1)P_{\mathrm{sp}}Q, & m\leq n,\\
\rho Q-(\rho-1)QR_{\mathrm{sp}}, & n<m.
\end{cases}
\label{eq:theory-two-level-decomposition}
\end{equation}
With time indices restored, multiplication by $c_{\rho,r_t}$
recovers the main-method identity
\eqref{eq:oracle-polar-correction}.
No approximation or random-matrix assumption is used.
\end{proof}

\subsection{Exact Frobenius Normalization}
\label{app:proof-exact-normalization}

Section~\ref{sec:theory-decomposition} uses the following
consequence of Theorem~\ref{thm:exact-decomposition}.

\begin{corollary}[Exact Frobenius normalization]
\label{cor:exact-normalization}
Under Theorem~\ref{thm:exact-decomposition},
\begin{equation}
\|\widetilde D(a,\rho)\|_F^2
=\|a\|_2^2+\rho^2(d-r).
\label{eq:exact-profile-norm}
\end{equation}
For $r+\rho^2(d-r)>0$, the normalizer $c_{\rho,r}$
from \eqref{eq:two-level-profile} gives
$\|c_{\rho,r}\widetilde D_\rho\|_F
=\|Q\|_F=\sqrt d$.
\end{corollary}

\begin{proof}
The rank-one singular modes are Frobenius-orthonormal:
\begin{align*}
\left\langle u_iv_i^\top,u_jv_j^\top\right\rangle_F
&=
\operatorname{tr}
\bigl(v_i u_i^\top u_j v_j^\top\bigr)\\
&=
(u_i^\top u_j)(v_i^\top v_j)
=
\mathbf 1\{i=j\}.
\end{align*}
Consequently, the squared Frobenius norm of a spectral
profile is the sum of its squared gains. In particular,
\[
\|\widetilde D(a,\rho)\|_F^2
=
\sum_{i=1}^r a_i^2+\rho^2(d-r)
=
\|a\|_2^2+\rho^2(d-r),
\]
which proves~\eqref{eq:exact-profile-norm}.
For $a=\mathbf 1_r$,
\[
\|\widetilde D_\rho\|_F^2=r+\rho^2(d-r),
\qquad
\|Q\|_F^2
=
\left\|\sum_{i=1}^d u_iv_i^\top\right\|_F^2=d.
\]
Whenever $r+\rho^2(d-r)>0$, the stated normalizer is
well-defined and satisfies
\[
\|c_{\rho,r}\widetilde D_\rho\|_F^2
=
\frac{d}{r+\rho^2(d-r)}
\bigl[r+\rho^2(d-r)\bigr]
=d
=\|Q\|_F^2.
\]
Taking nonnegative square roots proves the corollary.
If $r+\rho^2(d-r)=0$, the raw two-level direction is zero
and cannot be rescaled to norm $\sqrt d$; this case is
excluded by the corollary's hypothesis.
\end{proof}

\subsection{Online Calibration of the Noise Edge}
\label{app:theory-edge-calibration}

\textsc{BB-Adaptive} estimates the noise scale online from two
minibatch halves, without a large-batch reference gradient.
The proposition below proves the exact half-batch and Nesterov
covariance identities.
The Gaussian model behind the noise edge is treated separately after
the proof.

\begin{proposition}[Half-batch calibration of \muon{}'s spectral input]
\label{prop:edge-calibration}
Let $X_1,\ldots,X_b\in\mathbb R^{m\times n}$ be
independent per-example gradients with common mean
$\overline G$ and
\[
\operatorname{Cov}(\operatorname{vec}X_i)
=
s_{\mathrm{ex}}^2I_{mn},
\]
where $b$ is a positive even integer.
Let $G^{(1)},G^{(2)}$ be the two half-minibatch means,
and define
\[
G=\frac{G^{(1)}+G^{(2)}}{2},
\qquad
Z^{\mathrm{split}}
=
\frac{G^{(1)}-G^{(2)}}{2}.
\]
Then
\begin{equation}
\operatorname{Cov}(\operatorname{vec}(G-\overline G))
=
\operatorname{Cov}(\operatorname{vec}Z^{\mathrm{split}})
=
s_b^2I_{mn},
\qquad
s_b^2=\frac{s_{\mathrm{ex}}^2}{b},
\label{eq:split-noise-covariance}
\end{equation}
and the variance estimator satisfies
\begin{equation}
\widehat s_b^{\,2}
=
\frac{\|Z^{\mathrm{split}}\|_F^2}{mn},
\qquad
\mathbb E[\widehat s_b^{\,2}]=s_b^2.
\label{eq:split-variance-unbiased}
\end{equation}

Across training steps, suppose
$G_t=\overline G_t+\varepsilon_t$,
where the noises $\varepsilon_t$ are centered,
independent across time, and have the same covariance
$s_b^2I_{mn}$ after vectorization.
For $\beta\in[0,1)$ and $B_0=0$, define
\begin{equation}
B_t=\beta B_{t-1}+(1-\beta)G_t,
\qquad
M_t=(1-\beta)G_t+\beta B_t.
\label{eq:theory-nesterov-input}
\end{equation}
The component obtained by propagating the sampling noises
through these recursions is
\[
N_t^M
=
(1-\beta^2)\varepsilon_t
+
\beta^2(1-\beta)
\sum_{j=1}^{t-1}
\beta^{t-1-j}\varepsilon_j,
\]
where the sum is empty for $t=1$.
For every $t\geq1$,
\begin{equation}
\operatorname{Cov}(\operatorname{vec}N_t^M)
=
\nu_t^2(\beta)s_b^2I_{mn},
\label{eq:nesterov-noise-covariance}
\end{equation}
with
\begin{equation}
\nu_t^2(\beta)
=
(1-\beta^2)^2
+
\beta^4\frac{1-\beta}{1+\beta}
\bigl(1-\beta^{2(t-1)}\bigr).
\label{eq:theory-nesterov-variance-factor}
\end{equation}
Here $\nu_t(\beta)$ denotes the nonnegative square root
of this variance multiplier.
\end{proposition}

\paragraph{Interpretation.}
The half-batch difference removes the unknown mean gradient.
Its covariance equals that of the full-minibatch mean; each half
alone has twice this variance.
The factor $\nu_t^2(\beta)$ accounts for the current gradient
entering the Nesterov input both directly and through the updated
buffer.
For $\beta>0$, the EMA-only multiplier
$\frac{1-\beta}{1+\beta}\left(1-\beta^{2t}\right)$
would therefore understate the noise in $M_t$.
The implementation in Section~\ref{sec:split-edge} uses
$G_t=\chi_tG_t^{\mathrm{raw}}$ in the momentum recursion, keeps
$\overline s_t^{\,2}$ as an EMA of $\chi_t^2\widehat s_t^{\,2}$,
substitutes $\sqrt{\overline s_t^{\,2}}$ for $s_b$, and applies
$\kappa_{\mathrm{edge}}$, which gives
\eqref{eq:estimated-noise-edge}.
Appendix~\ref{app:method-calibration} lists the approximations
this involves.
Neither the covariance identities nor the noise edge show that the
resulting split is the best one for optimization.

\begin{proof}
We first establish the half-batch identities and then
propagate the noise through the momentum recursion.

\paragraph{Step 1: half-batch covariance and variance estimation.}
The half-minibatch means are
\[
G^{(1)}=\frac{2}{b}\sum_{i=1}^{b/2}X_i,
\qquad
G^{(2)}=\frac{2}{b}\sum_{i=b/2+1}^{b}X_i.
\]
By the proposition's assumptions, they are independent,
have common mean $\overline G$, and satisfy
\[
\operatorname{Cov}(\operatorname{vec}G^{(1)})
=
\operatorname{Cov}(\operatorname{vec}G^{(2)})
=
\frac{2s_{\mathrm{ex}}^2}{b}I_{mn}.
\]
Let
\[
g_1=\operatorname{vec}(G^{(1)}-\overline G),
\qquad
g_2=\operatorname{vec}(G^{(2)}-\overline G).
\]
Then
\[
\operatorname{vec}(G-\overline G)=\frac{g_1+g_2}{2},
\qquad
\operatorname{vec}Z^{\mathrm{split}}=\frac{g_1-g_2}{2}.
\]
The cross-covariances vanish by independence, so
\begin{align*}
\operatorname{Cov}(\operatorname{vec}(G-\overline G))
&=
\frac14\bigl(\operatorname{Cov}(g_1)
             +\operatorname{Cov}(g_2)\bigr)\\
&=
\frac{s_{\mathrm{ex}}^2}{b}I_{mn}
=s_b^2I_{mn},
\end{align*}
and the identical calculation for the difference gives
\[
\operatorname{Cov}(\operatorname{vec}Z^{\mathrm{split}})
=s_b^2I_{mn}.
\]
This proves \eqref{eq:split-noise-covariance}.
Since $Z^{\mathrm{split}}$ has zero mean,
\begin{align*}
\mathbb E[\widehat s_b^{\,2}]
&=
\frac{1}{mn}\mathbb E\|Z^{\mathrm{split}}\|_F^2\\
&=
\frac{1}{mn}
\operatorname{tr}
\bigl(\operatorname{Cov}(\operatorname{vec}Z^{\mathrm{split}})\bigr)\\
&=
\frac{1}{mn}\operatorname{tr}(s_b^2I_{mn})
=s_b^2.
\end{align*}
This proves \eqref{eq:split-variance-unbiased}.

For completeness, the average and difference also satisfy
\[
\operatorname{Cov}
\bigl(\operatorname{vec}(G-\overline G),
      \operatorname{vec}Z^{\mathrm{split}}\bigr)
=
\frac14\bigl(\operatorname{Cov}(g_1)
             -\operatorname{Cov}(g_2)\bigr)
=0.
\]
If $(g_1,g_2)$ is jointly Gaussian, the average and difference
are therefore independent. Without Gaussianity, zero
cross-covariance does not by itself imply independence;
that independence is not needed for the variance identity.

\paragraph{Step 2: propagation through the Nesterov spectral input.}
Write $G_t=\overline G_t+\varepsilon_t$ and consider the
component obtained by linearly propagating the centered
sampling noises through the recursions in
\eqref{eq:theory-nesterov-input}.
The EMA noise component satisfies
\[
N_t^B=\beta N_{t-1}^B+(1-\beta)\varepsilon_t,
\qquad N_0^B=0.
\]
Unrolling the recursion gives, for $t\geq1$,
\[
N_{t-1}^B
=
(1-\beta)\sum_{j=1}^{t-1}
\beta^{t-1-j}\varepsilon_j,
\]
with an empty sum when $t=1$.
Because the noises are independent across time and have
the same covariance $s_b^2I_{mn}$,
\begin{align*}
\operatorname{Cov}(\operatorname{vec}N_{t-1}^B)
&=
(1-\beta)^2
\sum_{j=1}^{t-1}\beta^{2(t-1-j)}s_b^2I_{mn}\\
&=
\frac{1-\beta}{1+\beta}
\bigl(1-\beta^{2(t-1)}\bigr)s_b^2I_{mn}.
\end{align*}
The spectral-input noise is
\begin{align*}
N_t^M
&=(1-\beta)\varepsilon_t+\beta N_t^B\\
&=(1-\beta)\varepsilon_t
  +\beta\bigl(\beta N_{t-1}^B+(1-\beta)\varepsilon_t\bigr)\\
&=(1-\beta^2)\varepsilon_t+\beta^2N_{t-1}^B\\
&=(1-\beta^2)\varepsilon_t
  +\beta^2(1-\beta)\sum_{j=1}^{t-1}
    \beta^{t-1-j}\varepsilon_j.
\end{align*}
The pair $(\varepsilon_t,N_{t-1}^B)$ is independent, so
\begin{align*}
\operatorname{Cov}(\operatorname{vec}N_t^M)
&=
(1-\beta^2)^2
\operatorname{Cov}(\operatorname{vec}\varepsilon_t)
+
\beta^4\operatorname{Cov}(\operatorname{vec}N_{t-1}^B)\\
&=
\left[
(1-\beta^2)^2
+
\beta^4\frac{1-\beta}{1+\beta}
\bigl(1-\beta^{2(t-1)}\bigr)
\right]s_b^2I_{mn}\\
&=\nu_t^2(\beta)s_b^2I_{mn}.
\end{align*}
This proves~\eqref{eq:nesterov-noise-covariance}
and~\eqref{eq:theory-nesterov-variance-factor}.
The current gradient appears twice, directly and through the
buffer, so its two coefficients are combined before taking
covariances.

Finally,
\begin{align*}
\nu_t^2(\beta)
&=(1-\beta^2)^2
 +\beta^4\frac{1-\beta}{1+\beta}
  \bigl(1-\beta^{2(t-1)}\bigr)\\
&=(1-\beta)^2
\left[
(1+\beta)^2
+
\frac{\beta^4\bigl(1-\beta^{2(t-1)}\bigr)}{1-\beta^2}
\right],
\end{align*}
which is the method-section factor in
\eqref{eq:nesterov-noise-factor}.
The entrywise standard deviation of $N_t^M$ is
$s_b\nu_t(\beta)$, where $\nu_t(\beta)$ is the nonnegative
square root of the variance multiplier.
\end{proof}

\paragraph{The noise edge under the Gaussian model.}
The covariance calculation above does not need Gaussian entries,
and it does not determine a spectral edge by itself.
Assume, in addition, that the entries of each $\varepsilon_t$ are
independent centered Gaussian variables with variance $s_b^2$, and
that the noise matrices are independent across time.
Then the linear combination for $N_t^M$ has independent centered
Gaussian entries with variance $s_b^2\nu_t^2(\beta)$.

For fixed $t$ and $\beta\in[0,1)$, assume first that $s_b>0$
and set $E_{m,n}=N_t^M/(s_b\nu_t(\beta))$.
This is a standard Gaussian matrix.
Its largest singular value converges to the upper edge of the
Marchenko--Pastur law \citep{marchenko1967,geman1980limit}:
\[
\frac{\|E_{m,n}\|_2}{\sqrt m+\sqrt n}
\xrightarrow{\mathrm{a.s.}}1
\qquad
\text{as }m,n\to\infty,
\quad m/n\to\zeta\in(0,\infty).
\]
Rescaling gives
$\|N_t^M\|_2/\bigl(s_b\nu_t(\beta)(\sqrt m+\sqrt n)\bigr)
\to1$ almost surely, so the noise edge of $N_t^M$ is
\begin{equation}
\tau_t^M
=
s_b\nu_t(\beta)(\sqrt m+\sqrt n).
\label{eq:theory-mp-edge}
\end{equation}
If $s_b=0$, the centered noise vanishes almost surely and
the edge is zero.
This step needs the Gaussian model and an asymptotic result.
Isotropic covariance alone does not give it, and it is not a
finite-sample bound.
\newpage

\section{Omitted experiment results}
\label{app:experiment-result}

In this section, we give the detailed experimental results.

\subsection {Corpora adaptation.} \label{app:corpus-adaptation}

\begin{table}[t]
\centering\small
\setlength{\tabcolsep}{3.5pt}
\scalebox{0.8}{%
\begin{tabular}{@{}lrrrrrrrrrr@{}}
\toprule
 & \multicolumn{2}{c}{14M}
 & \multicolumn{2}{c}{31M}
 & \multicolumn{2}{c}{70M}
 & \multicolumn{2}{c}{160M}
 & \multicolumn{2}{c}{410M} \\
\cmidrule(lr){2-3}\cmidrule(lr){4-5}\cmidrule(lr){6-7}\cmidrule(lr){8-9}\cmidrule(lr){10-11}
Corpus & $\ell_0$ & $\Delta$ & $\ell_0$ & $\Delta$ & $\ell_0$ & $\Delta$ & $\ell_0$ & $\Delta$ & $\ell_0$ & $\Delta$ \\
\midrule
\WikiDe{} & \cellcolor{LossBlue!60}4.530 & \cellcolor{DeltaRed!60}$-1.247{\scriptstyle\,\pm\,0.001}$ & \cellcolor{LossBlue!52}4.009 & \cellcolor{DeltaRed!51}$-1.050{\scriptstyle\,\pm\,0.002}$ & \cellcolor{LossBlue!47}3.699 & \cellcolor{DeltaRed!50}$-1.019{\scriptstyle\,\pm\,0.003}$ & \cellcolor{LossBlue!39}3.171 & \cellcolor{DeltaRed!36}$-0.719{\scriptstyle\,\pm\,0.002}$ & \cellcolor{LossBlue!31}2.626 & \cellcolor{DeltaRed!24}$-0.445{\scriptstyle\,\pm\,0.001}$ \\
\WikiZh{} & \cellcolor{LossBlue!48}3.715 & \cellcolor{DeltaRed!47}$-0.970{\scriptstyle\,\pm\,0.001}$ & \cellcolor{LossBlue!44}3.465 & \cellcolor{DeltaRed!48}$-0.994{\scriptstyle\,\pm\,0.001}$ & \cellcolor{LossBlue!42}3.329 & \cellcolor{DeltaRed!54}$-1.118{\scriptstyle\,\pm\,0.005}$ & \cellcolor{LossBlue!38}3.090 & \cellcolor{DeltaRed!50}$-1.035{\scriptstyle\,\pm\,0.001}$ & \cellcolor{LossBlue!30}2.570 & \cellcolor{DeltaRed!35}$-0.691{\scriptstyle\,\pm\,0.002}$ \\
\FineMath{} & \cellcolor{LossBlue!45}3.572 & \cellcolor{DeltaRed!26}$-0.486{\scriptstyle\,\pm\,0.002}$ & \cellcolor{LossBlue!41}3.278 & \cellcolor{DeltaRed!26}$-0.502{\scriptstyle\,\pm\,0.001}$ & \cellcolor{LossBlue!38}3.088 & \cellcolor{DeltaRed!28}$-0.537{\scriptstyle\,\pm\,0.002}$ & \cellcolor{LossBlue!33}2.753 & \cellcolor{DeltaRed!25}$-0.481{\scriptstyle\,\pm\,0.003}$ & \cellcolor{LossBlue!28}2.421 & \cellcolor{DeltaRed!21}$-0.375{\scriptstyle\,\pm\,0.005}$ \\
\WikiAr{} & \cellcolor{LossBlue!36}2.936 & \cellcolor{DeltaRed!46}$-0.930{\scriptstyle\,\pm\,0.000}$ & \cellcolor{LossBlue!33}2.761 & \cellcolor{DeltaRed!44}$-0.906{\scriptstyle\,\pm\,0.001}$ & \cellcolor{LossBlue!33}2.747 & \cellcolor{DeltaRed!51}$-1.045{\scriptstyle\,\pm\,0.000}$ & \cellcolor{LossBlue!31}2.620 & \cellcolor{DeltaRed!49}$-1.012{\scriptstyle\,\pm\,0.000}$ & \cellcolor{LossBlue!24}2.150 & \cellcolor{DeltaRed!35}$-0.698{\scriptstyle\,\pm\,0.001}$ \\
\FineMathN{} & \cellcolor{LossBlue!36}2.930 & \cellcolor{DeltaRed!17}$-0.290{\scriptstyle\,\pm\,0.002}$ & \cellcolor{LossBlue!32}2.692 & \cellcolor{DeltaRed!15}$-0.261{\scriptstyle\,\pm\,0.002}$ & \cellcolor{LossBlue!30}2.566 & \cellcolor{DeltaRed!16}$-0.269{\scriptstyle\,\pm\,0.004}$ & \cellcolor{LossBlue!27}2.320 & \cellcolor{DeltaRed!15}$-0.254{\scriptstyle\,\pm\,0.002}$ & \cellcolor{LossBlue!22}2.050 & \cellcolor{DeltaRed!11}$-0.159{\scriptstyle\,\pm\,0.002}$ \\
\CodeParrot{} & \cellcolor{LossBlue!27}2.381 & \cellcolor{DeltaRed!25}$-0.479{\scriptstyle\,\pm\,0.004}$ & \cellcolor{LossBlue!22}2.040 & \cellcolor{DeltaRed!21}$-0.391{\scriptstyle\,\pm\,0.001}$ & \cellcolor{LossBlue!19}1.845 & \cellcolor{DeltaRed!20}$-0.361{\scriptstyle\,\pm\,0.003}$ & \cellcolor{LossBlue!14}1.508 & \cellcolor{DeltaRed!14}$-0.238{\scriptstyle\,\pm\,0.003}$ & \cellcolor{LossBlue!10}1.229 & \cellcolor{DeltaRed!9}$-0.118{\scriptstyle\,\pm\,0.002}$ \\
\bottomrule
\end{tabular}%
}

\caption{Corpus adaptation across five model scales.
$\ell_0$ is the held-out loss before continued pretraining.
$\Delta=\ell_{\mathrm{final}}-\ell_0$ is the change
after about 100M tokens with \muon{}, reported as mean $\pm$ standard
deviation over three seeds.
Negative values indicate improvement.
$\ell_0$ has no spread, because all seeds start from the same
checkpoint and use the same held-out windows.
Rows are ordered by decreasing $\ell_0$, which gives the same order at
every scale. Corpora are grouped into mathematics, code, and
multilingual Wikipedia.
}
\label{tab:datasets}
\end{table}

\paragraph{Models and datasets.} We conduct continued pretraining by adapting \emph{converged}, publicly available Pythia checkpoints
 on new corpora with the cross-entropy loss.
We use five scales, Pythia-14M, 31M, 70M, 160M, and 410M, with
hidden widths $\dmodel=128$, 256, 512, 768, and 1024, respectively.
For the matrix methods, \adamw{} updates the embeddings and the
language-model head.
Table~\ref{tab:models} in Appendix~\ref{app:experiment-result} lists
the widths, batch settings, and training budgets.
Pythia was pretrained on the Pile~\citep{Piledata}, a mainly English
mixture that includes GitHub code, English Wikipedia, and arXiv.
We adapt it to six corpora in three groups.
For mathematics, \FineMath{} and \FineMathN{}~\citep{finemath} come
from Common Crawl with different math-quality thresholds.
For code, \CodeParrot{} is CodeParrot-clean~\citep{codeparrotclean},
a filtered and deduplicated Python corpus from GitHub.
For multilingual text, \WikiDe{}, \WikiZh{}, and \WikiAr{} are
German, Chinese, and Arabic Wikipedia.

\paragraph{Initial loss reflects corpus shift.}
Table~\ref{tab:datasets} reports the pre-adaptation held-out loss
$\ell_0$ and the change $\Delta = \ell_{\mathrm{final}}-\ell_0$ after 100M tokens of
continued pretraining with \muon{}.
We use $\ell_0$ as a proxy for how far each corpus is from the Pythia
pretraining distribution.
A lower $\ell_0$ means that the pretrained model already assigns
higher probability to the corpus.
A higher $\ell_0$ suggests a larger distribution shift.
The corpora keep the same order at every scale.
\CodeParrot{} has the lowest $\ell_0$, as GitHub code is already part
of the Pile.
\WikiDe{} and \WikiZh{} have the highest $\ell_0$, as they are far
from the mainly English Pile.
The two math corpora, which overlap with mathematical text in the
Pile, lie in between.
\WikiAr{} also lies in between, below \FineMath{} and above
\FineMathN{} at every scale.

Continued pretraining with \muon{} lowers the validation loss in every setting, but a larger $\ell_0$ does not always give a larger
reduction. For example, \WikiAr{} starts below \FineMath{} yet gains about twice as much at every scale. \FineMath{} has higher $\ell_0$ than \FineMathN{} and gains more at every scale. \CodeParrot{} starts with the lowest $\ell_0$ and has the smallest
reductions at 160M and 410M, with $\Delta=-0.479$ at
14M and $-0.118$ at 410M. Among the multilingual corpora, \WikiDe{} starts with the highest $\ell_0$. \WikiZh{} and \WikiAr{} keep large gains through 160M and still improve by about $0.7$ at 410M.
Thus, $\ell_0$ describes how well the pretrained model already fits a corpus, while $\Delta$ records the change achieved
within the adaptation budget.

\subsection{ Ablation studies}
\label{app:omitted-ablation}

\paragraph{Fixed fraction versus noise calibration.}
Both variants use the same norm-matched two-level rule with $\rho=4$.
They differ in how they choose the split and in how they track the
spike subspace (Appendix~\ref{app:spike-tracking}).
Table~\ref{tab:fracsweep} shows that the best fixed fraction varies
across models and corpora.
\algoadaptive{} instead sets the split for each matrix and step from
the noise edge.
On average, it keeps $4.6\%$ of $\dmodel$ directions as spikes
(5th--95th percentile: $1.3$--$7.7\%$; per-corpus means:
$2.8$--$5.8\%$).
Without any fraction search, it performs on par with the shared
\algostatic{}$_{(5\%)}$.
Their aggregate losses are close (Table~\ref{tab:tokensaved}), and
\algoadaptive{} has lower loss in 10 settings, higher loss in 6, and
equal loss in 14.
\algostatic{}$_{(5\%)}$ has the lower per-step cost
(Table~\ref{tab:cost}).
Selecting the fraction for each setting gives further gains:
\algostatic{}$^{\dagger}$ has lower loss than \algoadaptive{} in 20
settings, higher loss in 4, and equal loss in 6
(Table~\ref{tab:main}).
Noise calibration thus removes the fraction search, a shared fraction
lowers the per-step cost, and per-setting selection lowers the loss
further at an extra tuning cost.

\paragraph{Adding a third band.}
We also test whether a finer split improves on two bands.
The three-band variant of \algostatic{} splits the singular directions of each matrix into the leading 5\%, the next 15\%, and the remaining 80\%. Its leading band is the same as in \algostatic{}$_{(5\%)}$, so the
comparison isolates the effect of splitting the bulk into two parts.
Averaged over all $30$ model--corpus settings, the three-band variant ends with a final loss $0.002\pm0.001$ higher than
\algostatic{}$_{(5\%)}$. The extra band therefore does not help, and two bands remain the better choice.
This is consistent with Theorem~\ref{thm:continuation-decomposition}, under which a finer split can improve on two bands only through utility differences within the bands.

\paragraph{Single-update allocation experiment.}
On Pythia-70M \muon{} trajectories for \FineMath{}, \CodeParrot{}, and
\WikiAr{}, we compare exact-SVD updates with $\rho=1.25$ and $1.5$
against the exact flat update $\rho=1$ near 25M and 50M tokens.
Each comparison fixes the current singular frames and noise-edge
partitions, applies a common $\rho$ across 24 hidden weight matrices,
and matches each matrix's Frobenius direction norm separately.
Both bulk-emphasizing gains lower the validation loss immediately after
the update in five of six corpus--checkpoint settings, despite paired
pre-update linearized predictions that favor less bulk weight.
These $K=0$ measurements illustrate the difference between pre-update
alignment and the post-update loss response studied in
Theorem~\ref{thm:continuation-decomposition}.
Appendix~\ref{app:post-update-loss} gives the protocol, and
Table~\ref{tab:post-update-loss} reports all six settings.

\paragraph{Gain sensitivity and numerical accuracy.}
For \algoadaptive{}, we sweep $\rho\in\{2,4,6,8,10\}$ across 30
model--corpus settings using 20M tokens, one seed, and the shared
learning rates (Table~\ref{tab:rho}).
Within a setting, the loss changes by a median of $0.009$ across the
tested gains, with larger changes at 70M.
The best tested gain improves on $\rho=4$ by about $0.0045$ on average.
The smallest gain, $\rho=2$, is best in 20 settings.
We keep $\rho=4$ as the shared default in all 100M-token runs.
The sweep covers bulk-emphasizing gains, and
Appendix~\ref{app:post-update-loss} studies gains near $\rho=1$.
At $\rho=4$, the per-scale median effective gains $\rho_{\mathrm{eff}}$
lie in $[3.963,4.005]$, and the per-scale median Rayleigh--Ritz
residuals stay below $10^{-5}$ (Table~\ref{tab:rhoeff}).
These diagnostics check the realized gain and the accuracy of the
retained spike subspaces.

\begin{table}[p]
\centering\footnotesize
\setlength{\tabcolsep}{2pt}
\scalebox{0.9}{
\begin{tabular}{@{}llrrrrrr@{}}
\toprule
Model & Method & \FineMath{} & \FineMathN{} & \CodeParrot{} & \WikiDe{} & \WikiZh{} & \WikiAr{} \\
\midrule
\multirow{5}{*}{14M} & \algostatic{}$_{(1.5\%)}$ & $3.088{\scriptstyle\,\pm\,0.003}$ & $2.642{\scriptstyle\,\pm\,0.002}$ & $1.903{\scriptstyle\,\pm\,0.003}$ & $\mathbf{3.278}{\scriptstyle\,\pm\,0.001}$ & $2.741{\scriptstyle\,\pm\,0.000}$ & $1.999{\scriptstyle\,\pm\,0.001}$ \\
 & \algostatic{}$_{(2\%)}$ & $3.086{\scriptstyle\,\pm\,0.002}$ & $2.639{\scriptstyle\,\pm\,0.001}$ & $1.899{\scriptstyle\,\pm\,0.003}$ & $3.279{\scriptstyle\,\pm\,0.000}$ & $2.739{\scriptstyle\,\pm\,0.001}$ & $\mathbf{1.998}{\scriptstyle\,\pm\,0.001}$ \\
 & \algostatic{}$_{(5\%)}$ & $3.085{\scriptstyle\,\pm\,0.003}$ & $2.639{\scriptstyle\,\pm\,0.002}$ & $1.900{\scriptstyle\,\pm\,0.003}$ & $\mathbf{3.278}{\scriptstyle\,\pm\,0.001}$ & $\mathbf{2.738}{\scriptstyle\,\pm\,0.001}$ & $2.000{\scriptstyle\,\pm\,0.000}$ \\
 & \algostatic{}$_{(15\%)}$ & $\mathbf{3.083}{\scriptstyle\,\pm\,0.002}$ & $2.639{\scriptstyle\,\pm\,0.002}$ & $\mathbf{1.898}{\scriptstyle\,\pm\,0.003}$ & $3.279{\scriptstyle\,\pm\,0.001}$ & $2.742{\scriptstyle\,\pm\,0.001}$ & $2.006{\scriptstyle\,\pm\,0.001}$ \\
 & \algostatic{}$_{(20\%)}$ & $\mathbf{3.083}{\scriptstyle\,\pm\,0.002}$ & $\mathbf{2.638}{\scriptstyle\,\pm\,0.002}$ & $1.899{\scriptstyle\,\pm\,0.004}$ & $3.282{\scriptstyle\,\pm\,0.002}$ & $2.745{\scriptstyle\,\pm\,0.001}$ & $2.009{\scriptstyle\,\pm\,0.001}$ \\
\midrule
\multirow{5}{*}{31M} & \algostatic{}$_{(1.5\%)}$ & $2.776{\scriptstyle\,\pm\,0.001}$ & $2.433{\scriptstyle\,\pm\,0.002}$ & $1.650{\scriptstyle\,\pm\,0.002}$ & $2.956{\scriptstyle\,\pm\,0.001}$ & $\mathbf{2.463}{\scriptstyle\,\pm\,0.001}$ & $\mathbf{1.848}{\scriptstyle\,\pm\,0.001}$ \\
 & \algostatic{}$_{(2\%)}$ & $2.776{\scriptstyle\,\pm\,0.002}$ & $2.433{\scriptstyle\,\pm\,0.002}$ & $1.649{\scriptstyle\,\pm\,0.001}$ & $2.956{\scriptstyle\,\pm\,0.002}$ & $2.464{\scriptstyle\,\pm\,0.001}$ & $\mathbf{1.848}{\scriptstyle\,\pm\,0.001}$ \\
 & \algostatic{}$_{(5\%)}$ & $2.773{\scriptstyle\,\pm\,0.002}$ & $2.431{\scriptstyle\,\pm\,0.003}$ & $\mathbf{1.648}{\scriptstyle\,\pm\,0.002}$ & $\mathbf{2.953}{\scriptstyle\,\pm\,0.001}$ & $2.464{\scriptstyle\,\pm\,0.001}$ & $1.849{\scriptstyle\,\pm\,0.001}$ \\
 & \algostatic{}$_{(15\%)}$ & $2.773{\scriptstyle\,\pm\,0.002}$ & $\mathbf{2.428}{\scriptstyle\,\pm\,0.002}$ & $\mathbf{1.648}{\scriptstyle\,\pm\,0.002}$ & $2.956{\scriptstyle\,\pm\,0.001}$ & $2.469{\scriptstyle\,\pm\,0.001}$ & $1.858{\scriptstyle\,\pm\,0.000}$ \\
 & \algostatic{}$_{(20\%)}$ & $\mathbf{2.772}{\scriptstyle\,\pm\,0.001}$ & $2.429{\scriptstyle\,\pm\,0.002}$ & $1.650{\scriptstyle\,\pm\,0.003}$ & $2.960{\scriptstyle\,\pm\,0.001}$ & $2.472{\scriptstyle\,\pm\,0.001}$ & $1.862{\scriptstyle\,\pm\,0.001}$ \\
\midrule
\multirow{5}{*}{70M} & \algostatic{}$_{(1.5\%)}$ & $2.549{\scriptstyle\,\pm\,0.004}$ & $2.295{\scriptstyle\,\pm\,0.006}$ & $1.478{\scriptstyle\,\pm\,0.003}$ & $2.677{\scriptstyle\,\pm\,0.002}$ & $2.204{\scriptstyle\,\pm\,0.000}$ & $\mathbf{1.690}{\scriptstyle\,\pm\,0.001}$ \\
 & \algostatic{}$_{(2\%)}$ & $2.547{\scriptstyle\,\pm\,0.002}$ & $2.294{\scriptstyle\,\pm\,0.004}$ & $1.478{\scriptstyle\,\pm\,0.002}$ & $2.678{\scriptstyle\,\pm\,0.002}$ & $2.205{\scriptstyle\,\pm\,0.002}$ & $\mathbf{1.690}{\scriptstyle\,\pm\,0.000}$ \\
 & \algostatic{}$_{(5\%)}$ & $2.541{\scriptstyle\,\pm\,0.004}$ & $2.290{\scriptstyle\,\pm\,0.003}$ & $1.472{\scriptstyle\,\pm\,0.003}$ & $2.673{\scriptstyle\,\pm\,0.002}$ & $\mathbf{2.203}{\scriptstyle\,\pm\,0.003}$ & $\mathbf{1.690}{\scriptstyle\,\pm\,0.001}$ \\
 & \algostatic{}$_{(15\%)}$ & $\mathbf{2.537}{\scriptstyle\,\pm\,0.002}$ & $\mathbf{2.287}{\scriptstyle\,\pm\,0.007}$ & $\mathbf{1.469}{\scriptstyle\,\pm\,0.003}$ & $\mathbf{2.670}{\scriptstyle\,\pm\,0.002}$ & $2.204{\scriptstyle\,\pm\,0.002}$ & $1.695{\scriptstyle\,\pm\,0.001}$ \\
 & \algostatic{}$_{(20\%)}$ & $\mathbf{2.537}{\scriptstyle\,\pm\,0.004}$ & $\mathbf{2.287}{\scriptstyle\,\pm\,0.003}$ & $\mathbf{1.469}{\scriptstyle\,\pm\,0.004}$ & $2.674{\scriptstyle\,\pm\,0.001}$ & $2.205{\scriptstyle\,\pm\,0.002}$ & $1.700{\scriptstyle\,\pm\,0.001}$ \\
\midrule
\multirow{5}{*}{160M} & \algostatic{}$_{(1.5\%)}$ & $\mathbf{2.281}{\scriptstyle\,\pm\,0.003}$ & $\mathbf{2.077}{\scriptstyle\,\pm\,0.002}$ & $\mathbf{1.280}{\scriptstyle\,\pm\,0.004}$ & $2.458{\scriptstyle\,\pm\,0.002}$ & $\mathbf{2.051}{\scriptstyle\,\pm\,0.001}$ & $\mathbf{1.597}{\scriptstyle\,\pm\,0.002}$ \\
 & \algostatic{}$_{(2\%)}$ & $\mathbf{2.281}{\scriptstyle\,\pm\,0.002}$ & $\mathbf{2.077}{\scriptstyle\,\pm\,0.003}$ & $1.281{\scriptstyle\,\pm\,0.003}$ & $\mathbf{2.457}{\scriptstyle\,\pm\,0.003}$ & $\mathbf{2.051}{\scriptstyle\,\pm\,0.000}$ & $1.598{\scriptstyle\,\pm\,0.002}$ \\
 & \algostatic{}$_{(5\%)}$ & $2.283{\scriptstyle\,\pm\,0.003}$ & $\mathbf{2.077}{\scriptstyle\,\pm\,0.003}$ & $1.283{\scriptstyle\,\pm\,0.004}$ & $2.459{\scriptstyle\,\pm\,0.002}$ & $2.054{\scriptstyle\,\pm\,0.001}$ & $1.606{\scriptstyle\,\pm\,0.001}$ \\
 & \algostatic{}$_{(15\%)}$ & $2.285{\scriptstyle\,\pm\,0.002}$ & $2.079{\scriptstyle\,\pm\,0.002}$ & $1.283{\scriptstyle\,\pm\,0.004}$ & $2.466{\scriptstyle\,\pm\,0.002}$ & $2.066{\scriptstyle\,\pm\,0.000}$ & $1.625{\scriptstyle\,\pm\,0.001}$ \\
 & \algostatic{}$_{(20\%)}$ & $2.284{\scriptstyle\,\pm\,0.003}$ & $2.080{\scriptstyle\,\pm\,0.002}$ & $1.284{\scriptstyle\,\pm\,0.003}$ & $2.470{\scriptstyle\,\pm\,0.003}$ & $2.072{\scriptstyle\,\pm\,0.001}$ & $1.632{\scriptstyle\,\pm\,0.001}$ \\
\midrule
\multirow{5}{*}{410M} & \algostatic{}$_{(1.5\%)}$ & $\mathbf{2.043}{\scriptstyle\,\pm\,0.005}$ & $\mathbf{1.890}{\scriptstyle\,\pm\,0.002}$ & $\mathbf{1.111}{\scriptstyle\,\pm\,0.002}$ & $\mathbf{2.174}{\scriptstyle\,\pm\,0.001}$ & $\mathbf{1.870}{\scriptstyle\,\pm\,0.001}$ & $\mathbf{1.446}{\scriptstyle\,\pm\,0.001}$ \\
 & \algostatic{}$_{(2\%)}$ & $\mathbf{2.043}{\scriptstyle\,\pm\,0.005}$ & $1.891{\scriptstyle\,\pm\,0.001}$ & $\mathbf{1.111}{\scriptstyle\,\pm\,0.002}$ & $\mathbf{2.174}{\scriptstyle\,\pm\,0.001}$ & $1.871{\scriptstyle\,\pm\,0.001}$ & $1.447{\scriptstyle\,\pm\,0.001}$ \\
 & \algostatic{}$_{(5\%)}$ & $\mathbf{2.043}{\scriptstyle\,\pm\,0.005}$ & $\mathbf{1.890}{\scriptstyle\,\pm\,0.001}$ & $\mathbf{1.111}{\scriptstyle\,\pm\,0.002}$ & $2.176{\scriptstyle\,\pm\,0.000}$ & $1.875{\scriptstyle\,\pm\,0.001}$ & $1.453{\scriptstyle\,\pm\,0.001}$ \\
 & \algostatic{}$_{(15\%)}$ & $2.044{\scriptstyle\,\pm\,0.005}$ & $\mathbf{1.890}{\scriptstyle\,\pm\,0.002}$ & $1.112{\scriptstyle\,\pm\,0.002}$ & $2.182{\scriptstyle\,\pm\,0.001}$ & $1.888{\scriptstyle\,\pm\,0.001}$ & $1.468{\scriptstyle\,\pm\,0.001}$ \\
 & \algostatic{}$_{(20\%)}$ & $2.044{\scriptstyle\,\pm\,0.005}$ & $\mathbf{1.890}{\scriptstyle\,\pm\,0.002}$ & $1.112{\scriptstyle\,\pm\,0.001}$ & $2.185{\scriptstyle\,\pm\,0.001}$ & $1.894{\scriptstyle\,\pm\,0.001}$ & $1.473{\scriptstyle\,\pm\,0.001}$ \\
\bottomrule
\end{tabular}
}

\caption{Fixed-fraction sweep for \algostatic{}: final validation loss
after about 100M training tokens, reported as mean $\pm$ standard
deviation over three seeds.
For $p\in\{0.015,0.02,0.05,0.15,0.20\}$, the leading band contains
$r_t=\lceil pd\rceil$ singular directions, where $d=\min\{m,n\}$ and
$p$ is held fixed within each run.
All fractions use $\rho=4$ and the shared learning rate of their scale,
with the same training budget, schedule, and seeds as
Table~\ref{tab:main}.
For each model--corpus setting, the fraction with the lowest
three-seed mean gives the \algostatic{}$^{\dagger}$ entry in
Table~\ref{tab:main}.
Table~\ref{tab:tokensaved} summarizes the shared configuration
$p=5\%$, and Figure~\ref{fig:grid-wiki_zh} shows $p=1.5\%$.
\textbf{Bold} marks the lowest mean among the fractions in each
setting, including ties at the printed precision. %
}
\label{tab:fracsweep}
\end{table}

\begin{table}[!ht]
\centering\small
\setlength{\tabcolsep}{2.5pt}
\begin{tabular}{@{}llrrrrrr@{}}
\toprule
Arm & Model & \FineMath{} & \FineMathN{} & \CodeParrot{} & \WikiDe{} & \WikiZh{} & \WikiAr{} \\
\midrule
\multirow{5}{*}{\freon{}} & 14M & $8.3{\scriptstyle\,\pm\,0.7}$ & $3.5{\scriptstyle\,\pm\,2.9}$ & $7.1{\scriptstyle\,\pm\,1.6}$ & $4.9{\scriptstyle\,\pm\,2.6}$ & $4.4{\scriptstyle\,\pm\,2.6}$ & -- \\
 & 31M & $7.8{\scriptstyle\,\pm\,1.8}$ & $8.7{\scriptstyle\,\pm\,3.5}$ & $3.3{\scriptstyle\,\pm\,2.9}$ & $4.4{\scriptstyle\,\pm\,1.7}$ & $4.7{\scriptstyle\,\pm\,0.7}$ & -- \\
 & 70M & $11.1{\scriptstyle\,\pm\,1.4}$ & $10.6{\scriptstyle\,\pm\,3.2}$ & $12.5{\scriptstyle\,\pm\,0.7}$ & $9.8{\scriptstyle\,\pm\,2.6}$ & $10.9{\scriptstyle\,\pm\,3.3}$ & $10.2{\scriptstyle\,\pm\,1.5}$ \\
 & 160M & -- & -- & -- & -- & -- & -- \\
 & 410M & $9.9{\scriptstyle\,\pm\,3.1}$ & $9.6{\scriptstyle\,\pm\,3.5}$ & -- & $6.0{\scriptstyle\,\pm\,5.1}$ & -- & -- \\
\midrule
\multirow{5}{*}{\algostatic{}$_{(5\%)}$} & 14M & $4.6{\scriptstyle\,\pm\,0.4}$ & $3.4{\scriptstyle\,\pm\,2.9}$ & $6.1{\scriptstyle\,\pm\,2.1}$ & $6.3{\scriptstyle\,\pm\,2.1}$ & $10.5{\scriptstyle\,\pm\,2.4}$ & $12.3{\scriptstyle\,\pm\,2.2}$ \\
 & 31M & $5.2{\scriptstyle\,\pm\,0.6}$ & -- & $3.1{\scriptstyle\,\pm\,2.8}$ & $8.2{\scriptstyle\,\pm\,2.2}$ & $9.5{\scriptstyle\,\pm\,0.6}$ & $9.8{\scriptstyle\,\pm\,2.1}$ \\
 & 70M & $8.4{\scriptstyle\,\pm\,2.9}$ & $9.2{\scriptstyle\,\pm\,3.3}$ & $11.2{\scriptstyle\,\pm\,1.6}$ & $7.0{\scriptstyle\,\pm\,1.7}$ & $7.7{\scriptstyle\,\pm\,1.7}$ & $13.2{\scriptstyle\,\pm\,0.6}$ \\
 & 160M & -- & -- & -- & -- & $2.6{\scriptstyle\,\pm\,0.8}$ & $4.6{\scriptstyle\,\pm\,2.1}$ \\
 & 410M & $10.0{\scriptstyle\,\pm\,2.9}$ & $6.9{\scriptstyle\,\pm\,5.5}$ & -- & $10.6{\scriptstyle\,\pm\,3.1}$ & $9.3{\scriptstyle\,\pm\,1.6}$ & -- \\
\midrule
\multirow{5}{*}{\algoadaptive{}} & 14M & $5.2{\scriptstyle\,\pm\,1.8}$ & $4.4{\scriptstyle\,\pm\,1.6}^{\,2/3}$ & $6.5{\scriptstyle\,\pm\,2.6}$ & $7.3{\scriptstyle\,\pm\,2.8}$ & $11.3{\scriptstyle\,\pm\,2.6}$ & $11.5{\scriptstyle\,\pm\,2.0}$ \\
 & 31M & $4.5{\scriptstyle\,\pm\,2.2}$ & $7.5{\scriptstyle\,\pm\,1.3}^{\,2/3}$ & $4.6{\scriptstyle\,\pm\,3.9}$ & $8.7{\scriptstyle\,\pm\,1.2}$ & $10.0{\scriptstyle\,\pm\,0.2}$ & $9.6{\scriptstyle\,\pm\,2.4}$ \\
 & 70M & $5.6{\scriptstyle\,\pm\,2.9}$ & $8.0{\scriptstyle\,\pm\,4.8}$ & $10.7{\scriptstyle\,\pm\,1.4}$ & $5.8{\scriptstyle\,\pm\,1.2}$ & $6.7{\scriptstyle\,\pm\,1.6}$ & $13.1{\scriptstyle\,\pm\,0.3}$ \\
 & 160M & -- & -- & -- & -- & $3.6{\scriptstyle\,\pm\,0.8}$ & $6.8{\scriptstyle\,\pm\,1.2}$ \\
 & 410M & $10.7{\scriptstyle\,\pm\,3.0}$ & $7.2{\scriptstyle\,\pm\,4.6}$ & -- & $11.3{\scriptstyle\,\pm\,1.8}$ & $9.7{\scriptstyle\,\pm\,1.9}$ & -- \\
\midrule
\multirow{5}{*}{\algostatic{}$^{\dagger}$} & 14M & $8.1{\scriptstyle\,\pm\,0.8}$ & $4.4{\scriptstyle\,\pm\,3.9}$ & $7.8{\scriptstyle\,\pm\,2.7}$ & $6.3{\scriptstyle\,\pm\,2.1}$ & $10.5{\scriptstyle\,\pm\,2.4}$ & $13.7{\scriptstyle\,\pm\,1.6}$ \\
 & 31M & $6.3{\scriptstyle\,\pm\,0.7}$ & $10.0{\scriptstyle\,\pm\,1.7}$ & $3.1{\scriptstyle\,\pm\,2.8}$ & $8.2{\scriptstyle\,\pm\,2.2}$ & $10.7{\scriptstyle\,\pm\,0.7}$ & $12.6{\scriptstyle\,\pm\,2.8}$ \\
 & 70M & $11.1{\scriptstyle\,\pm\,3.5}$ & $11.7{\scriptstyle\,\pm\,6.2}$ & $13.1{\scriptstyle\,\pm\,1.9}$ & $8.0{\scriptstyle\,\pm\,4.5}$ & $7.7{\scriptstyle\,\pm\,1.7}$ & $13.2{\scriptstyle\,\pm\,0.6}$ \\
 & 160M & -- & -- & -- & -- & $7.5{\scriptstyle\,\pm\,1.1}$ & $15.3{\scriptstyle\,\pm\,2.4}$ \\
 & 410M & $10.1{\scriptstyle\,\pm\,2.9}$ & $8.5{\scriptstyle\,\pm\,4.5}$ & -- & $13.9{\scriptstyle\,\pm\,1.6}$ & $15.3{\scriptstyle\,\pm\,1.2}$ & $12.8{\scriptstyle\,\pm\,1.0}$ \\
\bottomrule
\end{tabular}

\caption{Token savings against \muon{}, as a percentage of the
100M-token budget, for each method and model--corpus setting.
\muon{} is the reference, so it has no block of its own.
A cell is filled only where the method's three-seed mean final loss is
below \muon{}'s in Table~\ref{tab:main}, and ``--'' marks settings
where \muon{} is at least as good.
A filled cell reports $1-D^{\star}/100\mathrm{M}$, where $D^{\star}$
is the number of tokens at which the method's curve first reaches
\muon{}'s 100M-token loss in the same seed.
Values are mean $\pm$ standard deviation over the seeds in which the
method ends below \muon{}.
A superscript $k/3$ marks a setting won in only $k$ of the three seeds.
\deva{} and \spectra{} never beat \muon{}, so they are omitted.
The last block, \algostatic{}$^{\dagger}$, uses the best of the five
fractions in each setting (Table~\ref{tab:fracsweep}).
Because this fraction is chosen on the same runs, its savings are
optimistic.
Separately, intermediate points of a cosine schedule are evaluated
before the learning rate has fully decayed.
}
\label{tab:tokensaved-detail}
\end{table}

\paragraph{Effect of the spike fraction in \algostatic.}
Table~\ref{tab:fracsweep} presents the results of tuning the fixed fraction $p$ used by \algostatic{}. The optimal value varies substantially across model scales and corpora. Thus, no single value of $p$ is optimal across all experiments, and selecting one requires an expensive tuning of parameter $p$. In contrast, \algoadaptive{} determines the rank directly from the spectrum of each matrix using the Marchenko--Pastur noise edge. It selects, on average, $4.6\%$ of $\dmodel$ (5th--95th percentile: $1.3\%$--$7.7\%$; per-corpus means: $2.8\%$--$5.8\%$) and, without tuning, closely matches the best fixed fraction across scales. The shared $p=5\%$ configuration achieves aggregate
validation loss comparable to \algoadaptive{}
(Table~\ref{tab:tokensaved}).
Table~\ref{tab:main} instead reports
\algostatic{}$^{\dagger}$, selected separately for each
model--corpus setting from the five fractions in
Table~\ref{tab:fracsweep}. \emph{The benefit of the adaptive rule is therefore not improved accuracy over a tuned fixed fraction, but achieving comparable accuracy without the tuning}.

\paragraph{Frequency of improvement and token savings.}
Table~\ref{tab:tokensaved-detail} shows which settings each method wins
and its token savings in each, complementing the aggregate comparison
in Table~\ref{tab:tokensaved}.
\algostatic{}$_{(5\%)}$, \algoadaptive{}, and \algostatic{}$^{\dagger}$
have lower three-seed mean final loss than \muon{} in 23, 24, and 25 of
the 30 settings, respectively, compared with 19 for \freon{}.
These counts match the filled cells and exclude ties.
Within each winning setting, the savings are averaged over the seeds
that also finish below \muon{}'s same-seed final loss.
A superscript marks a setting won in fewer than three seeds, and an
entry without one is won in all three seeds.
For example, on Pythia-70M \CodeParrot{}, \algostatic{}$_{(5\%)}$ and
\algoadaptive{} save $11.2\pm1.6\%$ and $10.7\pm1.4\%$ of the budget,
and \algostatic{}$^{\dagger}$ saves $13.1\pm1.9\%$.
Table~\ref{tab:cost} reports the per-step cost.

\paragraph{Learning-rate tuning and the shared-rate policy.}
We distinguish each method's best multiplier in the 20M-token sweep
from the shared multiplier used in the 100M-token runs.
Figure~\ref{fig:lrsweep} and Table~\ref{tab:lrsweep} report the sweep
on \FineMath{}, run with one seed over multiples of the reference rate
of Section~\ref{sec:experiments}.
The swept methods are \muon{}, \freon{}, \deva{}, \spectra{},
\algostatic{}, and \algoadaptive{}.
In the sweep, \algostatic{} uses $p=20\%$, and both \bulkboost{}
variants use $\rho=4$.
Each method's best multiplier at a scale is the one with the lowest
validation loss after 20M tokens.

The best multipliers agree across all six methods at 14M, 31M, and
410M, and across five of the six at 70M and 160M.
At 70M, \muon{} does best at $0.1\times$, while the other methods do
best at $0.25\times$.
At 160M, \spectra{} does best at $0.05\times$, while the other methods
do best at $0.1\times$.
The shared multipliers are therefore $0.1$, $0.1$, $0.25$, $0.1$, and
$0.05$ for Pythia-14M, 31M, 70M, 160M, and 410M, and each is the best
multiplier for most of the six methods.
All 100M-token runs use the resulting peak learning rates for every
method, including \adamw{}, on all six corpora.
In particular, the 70M \muon{} runs use $5\times10^{-4}$ and the 160M
\spectra{} runs use $10^{-4}$, instead of their own best sweep rates.
In the sweep, this raises the loss from 2.779 to 2.783 for 70M
\muon{}, and from 2.552 to 2.553 for 160M \spectra{}.
Table~\ref{tab:lrsweep} marks each method's best multiplier, and its
caption lists the shared multipliers.

\label{sec:exp:tuning}
\begin{figure}[t]
\centering
\includegraphics[width=\linewidth]{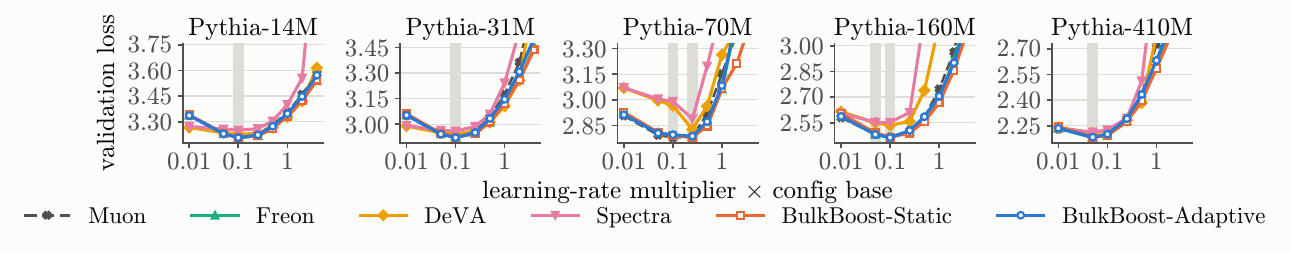}
\caption{Learning-rate sweeps on \FineMath{} with a 20M-token budget and one seed. Curves show final validation loss against multiples of each scale's reference rate. Grey bands mark each method's best multiplier. \algostatic{} uses $p=20\%$ in the sweep. Table~\ref{tab:lrsweep} lists the full results.} %

\label{fig:lrsweep}
\end{figure}

\begin{table}[t]
\centering\footnotesize
\setlength{\tabcolsep}{4pt}
\begin{tabular}{@{}llrrrrrrrrc@{}}
\toprule
Model & Method & 0.01$\times$ & 0.05$\times$ & 0.1$\times$ & 0.25$\times$ & 0.5$\times$ & 1$\times$ & 2$\times$ & 4$\times$ & best \\
\midrule
\multirow{6}{*}{14M} & \muon & 3.335 & 3.232 & \textbf{3.209} & \underline{3.227} & 3.281 & 3.354 & 3.459 & 3.581 & 0.1$\times$ \\
 & \freon & 3.338 & 3.231 & \textbf{3.207} & \underline{3.219} & 3.265 & 3.335 & 3.427 & 3.540 & 0.1$\times$ \\
 & \deva & 3.266 & 3.235 & \textbf{3.226} & \underline{3.232} & 3.267 & 3.328 & 3.421 & 3.615 & 0.1$\times$ \\
 & \spectra & 3.272 & \underline{3.257} & \textbf{3.250} & 3.259 & 3.304 & 3.400 & 3.555 & 4.458 & 0.1$\times$ \\
 & \algostatic{}$_{(20\%)}$ & 3.340 & 3.232 & \textbf{3.207} & \underline{3.217} & 3.259 & 3.332 & 3.422 & 3.543 & 0.1$\times$ \\
 & \algoadaptive{} & 3.336 & 3.230 & \textbf{3.207} & \underline{3.223} & 3.273 & 3.347 & 3.448 & 3.572 & 0.1$\times$ \\
\midrule
\multirow{6}{*}{31M} & \muon & 3.051 & \underline{2.945} & \textbf{2.922} & 2.958 & 3.042 & 3.176 & 3.363 & 3.733 & 0.1$\times$ \\
 & \freon & 3.055 & \underline{2.945} & \textbf{2.924} & 2.948 & 3.016 & 3.125 & 3.259 & 3.450 & 0.1$\times$ \\
 & \deva & 2.989 & \underline{2.950} & \textbf{2.942} & 2.963 & 3.011 & 3.106 & 3.257 & 3.747 & 0.1$\times$ \\
 & \spectra & 2.993 & \underline{2.965} & \textbf{2.960} & 2.987 & 3.060 & 3.242 & 3.547 & 4.848 & 0.1$\times$ \\
 & \algostatic{}$_{(20\%)}$ & 3.061 & 2.947 & \textbf{2.925} & \underline{2.944} & 3.012 & 3.122 & 3.261 & 3.436 & 0.1$\times$ \\
 & \algoadaptive{} & 3.053 & \underline{2.943} & \textbf{2.922} & 2.952 & 3.034 & 3.146 & 3.308 & 3.504 & 0.1$\times$ \\
\midrule
\multirow{6}{*}{70M} & \muon & 2.908 & 2.794 & \textbf{2.779} & \underline{2.783} & 2.919 & 3.152 & 3.386 & 3.986 & 0.1$\times$ \\
 & \freon & 2.918 & 2.808 & \underline{2.792} & \textbf{2.785} & 2.848 & 3.064 & 3.473 & 3.674 & 0.25$\times$ \\
 & \deva & 3.070 & 2.996 & 2.966 & \textbf{2.827} & \underline{2.963} & 3.263 & 3.415 & 4.610 & 0.25$\times$ \\
 & \spectra & 3.071 & 3.005 & \underline{2.991} & \textbf{2.890} & 3.198 & 3.547 & 5.466 & 5.936 & 0.25$\times$ \\
 & \algostatic{}$_{(20\%)}$ & 2.928 & 2.809 & \underline{2.789} & \textbf{2.782} & 2.845 & 3.065 & 3.212 & 3.467 & 0.25$\times$ \\
 & \algoadaptive{} & 2.913 & 2.805 & \underline{2.796} & \textbf{2.787} & 2.872 & 3.083 & 3.484 & 3.634 & 0.25$\times$ \\
\midrule
\multirow{6}{*}{160M} & \muon & 2.580 & \underline{2.483} & \textbf{2.466} & 2.503 & 2.579 & 2.747 & 2.964 & 3.434 & 0.1$\times$ \\
 & \freon & 2.590 & \underline{2.484} & \textbf{2.463} & 2.498 & 2.569 & 2.687 & 2.918 & 3.196 & 0.1$\times$ \\
 & \deva & 2.615 & \underline{2.550} & \textbf{2.535} & 2.558 & 2.739 & 3.130 & 3.545 & 4.937 & 0.1$\times$ \\
 & \spectra & 2.607 & \textbf{2.552} & \underline{2.553} & 2.610 & 3.182 & 3.940 & 5.445 & 6.425 & 0.05$\times$ \\
 & \algostatic{}$_{(20\%)}$ & 2.603 & \underline{2.490} & \textbf{2.466} & \underline{2.490} & 2.558 & 2.666 & 2.857 & 3.107 & 0.1$\times$ \\
 & \algoadaptive{} & 2.587 & \underline{2.481} & \textbf{2.463} & 2.505 & 2.585 & 2.706 & 2.901 & 3.180 & 0.1$\times$ \\
\midrule
\multirow{6}{*}{410M} & \muon & 2.232 & \textbf{2.182} & \underline{2.199} & 2.290 & 2.429 & 2.728 & 3.047 & 3.204 & 0.05$\times$ \\
 & \freon & 2.236 & \textbf{2.180} & \underline{2.196} & 2.278 & 2.412 & 2.598 & 2.799 & 2.989 & 0.05$\times$ \\
 & \deva & 2.239 & \textbf{2.208} & \underline{2.218} & 2.282 & 2.384 & 2.779 & 3.263 & 3.884 & 0.05$\times$ \\
 & \spectra & 2.235 & \textbf{2.212} & \underline{2.226} & 2.294 & 2.512 & 3.228 & 3.529 & 3.963 & 0.05$\times$ \\
 & \algostatic{}$_{(20\%)}$ & 2.246 & \textbf{2.182} & \underline{2.193} & 2.273 & 2.397 & 2.584 & 2.783 & 2.990 & 0.05$\times$ \\
 & \algoadaptive{} & 2.235 & \textbf{2.183} & \underline{2.200} & 2.291 & 2.433 & 2.632 & 2.849 & 3.060 & 0.05$\times$ \\
\bottomrule
\end{tabular}

\caption{Final held-out loss on \FineMath{} after 20M tokens with one
seed, at multiples of each scale's reference rate.
\textbf{Bold} marks each method's lowest loss, and \underline{underline}
marks its second lowest.
\algostatic{} uses $p=20\%$.
The 100M-token runs use one shared multiplier per scale, the best for
most methods: $0.1$, $0.1$, $0.25$, $0.1$, and $0.05$ from 14M to
410M.}%
\label{tab:lrsweep}
\end{table}

\begin{table}[p]
\centering\footnotesize
\begin{tabular}{@{}llrrrrrcr@{}}
\toprule
Model & Corpus & $\rho{=}2$ & $\rho{=}4$ & $\rho{=}6$ & $\rho{=}8$ & $\rho{=}10$ & best & range \\
\midrule
\multirow{6}{*}{14M} & \FineMath{} & \underline{3.209} & \underline{3.209} & \textbf{3.209} & \underline{3.209} & 3.210 & 6 & 0.000 \\
 & \FineMathN{} & 2.737 & 2.737 & \underline{2.736} & 2.737 & \textbf{2.735} & 10 & 0.002 \\
 & \CodeParrot{} & \underline{2.046} & \textbf{2.046} & 2.048 & 2.048 & 2.048 & 4 & 0.002 \\
 & \WikiDe{} & \textbf{3.518} & \underline{3.522} & 3.523 & 3.526 & 3.528 & 2 & 0.010 \\
 & \WikiZh{} & \textbf{2.960} & \underline{2.962} & 2.964 & 2.965 & 2.967 & 2 & 0.007 \\
 & \WikiAr{} & \textbf{2.184} & \underline{2.188} & 2.190 & 2.193 & 2.194 & 2 & 0.011 \\
\midrule
\multirow{6}{*}{31M} & \FineMath{} & \textbf{2.923} & \underline{2.925} & 2.927 & 2.927 & 2.928 & 2 & 0.005 \\
 & \FineMathN{} & \textbf{2.530} & \underline{2.534} & 2.535 & 2.536 & -- & 2 & 0.006 \\
 & \CodeParrot{} & \textbf{1.796} & \underline{1.797} & 1.798 & 1.800 & 1.802 & 2 & 0.006 \\
 & \WikiDe{} & \textbf{3.203} & \underline{3.205} & 3.208 & 3.210 & 3.212 & 2 & 0.009 \\
 & \WikiZh{} & \textbf{2.702} & \underline{2.705} & 2.709 & 2.711 & 2.712 & 2 & 0.011 \\
 & \WikiAr{} & \textbf{2.046} & \underline{2.051} & 2.053 & 2.056 & 2.058 & 2 & 0.012 \\
\midrule
\multirow{6}{*}{70M} & \FineMath{} & \textbf{2.780} & \underline{2.795} & 2.806 & 2.818 & 2.837 & 2 & 0.057 \\
 & \FineMathN{} & \underline{2.458} & \textbf{2.457} & 2.481 & 2.482 & 2.492 & 4 & 0.035 \\
 & \CodeParrot{} & \textbf{1.677} & \underline{1.691} & 1.714 & 1.712 & 1.731 & 2 & 0.054 \\
 & \WikiDe{} & \textbf{2.985} & \underline{3.020} & 3.030 & 3.034 & 3.043 & 2 & 0.057 \\
 & \WikiZh{} & \textbf{2.481} & \underline{2.493} & 2.503 & 2.516 & 2.528 & 2 & 0.047 \\
 & \WikiAr{} & \textbf{1.915} & \underline{1.919} & 1.923 & 1.926 & 1.931 & 2 & 0.016 \\
\midrule
\multirow{6}{*}{160M} & \FineMath{} & \underline{2.466} & \textbf{2.466} & 2.467 & \underline{2.466} & 2.468 & 4 & 0.002 \\
 & \FineMathN{} & 2.205 & \textbf{2.201} & 2.206 & \underline{2.203} & 2.210 & 4 & 0.009 \\
 & \CodeParrot{} & 1.408 & \textbf{1.406} & \underline{1.407} & \underline{1.407} & 1.409 & 4 & 0.002 \\
 & \WikiDe{} & \textbf{2.698} & \underline{2.702} & 2.704 & 2.706 & 2.707 & 2 & 0.010 \\
 & \WikiZh{} & \textbf{2.328} & \underline{2.332} & 2.337 & 2.339 & 2.340 & 2 & 0.012 \\
 & \WikiAr{} & \textbf{1.859} & \underline{1.863} & 1.864 & 1.866 & 1.870 & 2 & 0.011 \\
\midrule
\multirow{6}{*}{410M} & \FineMath{} & \underline{2.181} & \underline{2.181} & \textbf{2.180} & \underline{2.181} & \underline{2.181} & 6 & 0.001 \\
 & \FineMathN{} & 1.979 & 1.978 & 1.977 & \textbf{1.976} & \underline{1.976} & 8 & 0.002 \\
 & \CodeParrot{} & 1.194 & 1.192 & \underline{1.191} & \underline{1.191} & \textbf{1.190} & 10 & 0.003 \\
 & \WikiDe{} & \textbf{2.348} & \underline{2.348} & \underline{2.348} & 2.349 & 2.349 & 2 & 0.001 \\
 & \WikiZh{} & \textbf{2.093} & \underline{2.097} & 2.099 & 2.100 & 2.101 & 2 & 0.008 \\
 & \WikiAr{} & \textbf{1.668} & \underline{1.674} & 1.676 & 1.678 & 1.679 & 2 & 0.012 \\
\bottomrule
\end{tabular}

\caption{$\rho$ tuning across all 30 model--corpus pairs, measured by final held-out validation loss after 20M tokens. The best loss in each row is \textbf{bold}, and the second-best is \underline{underlined}. The \emph{range} column is the difference between the largest and
smallest loss, computed before rounding; its median is $0.009$. At each model scale, the sweep uses the same peak learning rate as the main 100M-token \algoadaptive{} runs (Table~\ref{tab:lrsweep}), held fixed across $\rho$ values and corpora. The sweep uses a different seed from the learning-rate sweep, so its
$\rho=4$ entries on \FineMath{} differ from the \algoadaptive{}
entries of Table~\ref{tab:lrsweep} by at most $0.008$.} %

\label{tab:rho}
\end{table}

\begin{table}[!ht]
\centering\small
\begin{tabular}{@{}lrrrrr@{}}
\toprule
Model & matrices & median $\rho_{\mathrm{eff}}$ & p5--p95 & $\rho_{\mathrm{eff}}-\rho$ & Ritz residual \\
\midrule
Pythia-14M & 144 & 3.980 & 3.88--4.07 & $-0.020$ & 4.1$\times10^{-7}$ \\
Pythia-31M & 144 & 3.980 & 3.91--4.04 & $-0.020$ & 4.7$\times10^{-7}$ \\
Pythia-70M & 144 & 4.005 & 3.95--4.06 & $+0.005$ & 6.2$\times10^{-7}$ \\
Pythia-160M & 288 & 3.963 & 3.92--4.05 & $-0.037$ & 1.5$\times10^{-6}$ \\
Pythia-410M & 576 & 3.965 & 3.93--4.02 & $-0.035$ & 8.2$\times10^{-6}$ \\
\bottomrule
\end{tabular}

\caption{Numerical diagnostics for \algoadaptive{} at
the target gain $\rho=4$, summarized by model scale.
``matrices'' counts the instrumented two-dimensional
weight matrices across the six corpora.
We report the median effective gain $\rho_{\mathrm{eff}}$,
its 5th--95th percentile interval, and the difference
between the median and the target gain.
The final column reports the median Rayleigh--Ritz
residual of the retained spike subspaces.}
\label{tab:rhoeff}
\end{table}

\paragraph{Computation-time overhead.}
Table~\ref{tab:cost} reports per-step runtimes relative
to \muon{} at each model scale, averaged over the six
corpora.
The measured fixed-rank configurations with $p=1.5\%$
and $5\%$ require $1.19$--$1.59\times$ \muon{}'s
step time, compared with $1.36$--$2.13\times$ for
\algoadaptive{}.
At 70M--410M, these ranges narrow to
$1.19$--$1.27\times$ and $1.36$--$1.50\times$,
respectively.
The two measured fixed fractions have nearly identical
runtimes and are faster per step than the noise-calibrated
variant at every scale.
For the other methods, \adamw{} and \spectra{} take
$0.92$--$0.98\times$ and $0.98$--$1.00\times$
\muon{}'s step time, respectively, while \freon{} and
\deva{} incur overheads of up to $20\%$ and $34\%$.
These timings compare the implemented variants, which
differ in both noise calibration and subspace tracking, so they do
not isolate the cost of rank selection.
The overhead of the adaptive variant is largest at 14M ($113\%$) and
lies between $36\%$ and $66\%$ at the larger scales.

\begin{table}[p]
\centering\small
\setlength{\tabcolsep}{4pt}
\begin{tabular}{@{}lrrrrr@{}}
\toprule
Arm & 14M & 31M & 70M & 160M & 410M \\
\midrule
\muon{} (ms/step) & $173{\scriptstyle\,\pm\,5}$ & $308{\scriptstyle\,\pm\,1}$ & $617{\scriptstyle\,\pm\,0}$ & $635{\scriptstyle\,\pm\,1}$ & $1748{\scriptstyle\,\pm\,10}$ \\
\midrule
\adamw{} & $0.93{\scriptstyle\,\pm\,0.05}\times$ & $0.92{\scriptstyle\,\pm\,0.01}\times$ & $0.98{\scriptstyle\,\pm\,0.01}\times$ & $0.98{\scriptstyle\,\pm\,0.00}\times$ & $0.98{\scriptstyle\,\pm\,0.01}\times$ \\
\freon{} & $1.17{\scriptstyle\,\pm\,0.01}\times$ & $1.10{\scriptstyle\,\pm\,0.00}\times$ & $1.06{\scriptstyle\,\pm\,0.00}\times$ & $1.17{\scriptstyle\,\pm\,0.00}\times$ & $1.20{\scriptstyle\,\pm\,0.00}\times$ \\
\deva{} & $1.13{\scriptstyle\,\pm\,0.00}\times$ & $1.10{\scriptstyle\,\pm\,0.00}\times$ & $1.11{\scriptstyle\,\pm\,0.00}\times$ & $1.27{\scriptstyle\,\pm\,0.00}\times$ & $1.34{\scriptstyle\,\pm\,0.00}\times$ \\
\spectra{} & $0.98{\scriptstyle\,\pm\,0.01}\times$ & $1.00{\scriptstyle\,\pm\,0.00}\times$ & $1.00{\scriptstyle\,\pm\,0.00}\times$ & $1.00{\scriptstyle\,\pm\,0.00}\times$ & $1.00{\scriptstyle\,\pm\,0.00}\times$ \\
\algostatic{}$_{(1.5\%)}$ & $1.59{\scriptstyle\,\pm\,0.05}\times$ & $1.32{\scriptstyle\,\pm\,0.01}\times$ & $1.19{\scriptstyle\,\pm\,0.00}\times$ & $1.27{\scriptstyle\,\pm\,0.00}\times$ & $1.27{\scriptstyle\,\pm\,0.01}\times$ \\
\algostatic{}$_{(5\%)}$ & $1.59{\scriptstyle\,\pm\,0.04}\times$ & $1.33{\scriptstyle\,\pm\,0.01}\times$ & $1.19{\scriptstyle\,\pm\,0.00}\times$ & $1.27{\scriptstyle\,\pm\,0.00}\times$ & $1.27{\scriptstyle\,\pm\,0.01}\times$ \\
\algoadaptive{} & $2.13{\scriptstyle\,\pm\,0.05}\times$ & $1.66{\scriptstyle\,\pm\,0.00}\times$ & $1.36{\scriptstyle\,\pm\,0.00}\times$ & $1.50{\scriptstyle\,\pm\,0.00}\times$ & $1.48{\scriptstyle\,\pm\,0.01}\times$ \\
\bottomrule
\end{tabular}

\caption{Per-step wall-clock time at each model scale.
For each run, the per-step time is its total wall-clock time divided
by its number of steps.
Entries are mean $\pm$ standard deviation over the runs at that
scale, across six corpora and three seeds.
The \muon{} row reports milliseconds per step;
all other rows report runtime relative to \muon{}
at the same scale.
Fixed-rank timings cover $p=1.5\%$ and $5\%$.
Measurements were collected on shared cluster nodes
and are subject to scheduling variability. The ratios within a scale are stable. %
}
\label{tab:cost}
\end{table}

\begin{table}[p]
\centering\small
\begin{tabular}{@{}lrcrrr@{}}
\toprule
Model & $\dmodel$ & micro-batch $\times$ accum & steps & tokens & \muon{} ms/step \\
\midrule
Pythia-14M & 128 & 32$\times$1 & 3051 & 100M & $173{\scriptstyle\,\pm\,5}$ \\
Pythia-31M & 256 & 32$\times$1 & 3051 & 100M & $308{\scriptstyle\,\pm\,1}$ \\
Pythia-70M & 512 & 32$\times$1 & 3051 & 100M & $617{\scriptstyle\,\pm\,0}$ \\
Pythia-160M & 768 & 32$\times$1 & 3051 & 100M & $635 {\scriptstyle\,\pm\,1}$ \\
Pythia-410M & 1024 & 32$\times$1 & 3051 & 100M & $1748{\scriptstyle\,\pm\,10}$ \\
\bottomrule
\end{tabular}

\caption{Pythia model configurations and training budgets.
$\dmodel$ denotes the hidden width, and
``micro-batch $\times$ accum'' gives the number of
sequences per microbatch and the number of gradient
accumulation steps.
All scales use an effective batch of 32 sequences of
length 1024, giving 32,768 tokens per step and
approximately 100M tokens over 3,051 steps.
The \muon{} timing entries reproduce those in
Table~\ref{tab:cost}: mean $\pm$ standard deviation over six corpora
and three seeds at each scale. %
}
\label{tab:models}
\end{table}

\begin{figure}[t]
\centering
\includegraphics[width=\linewidth]{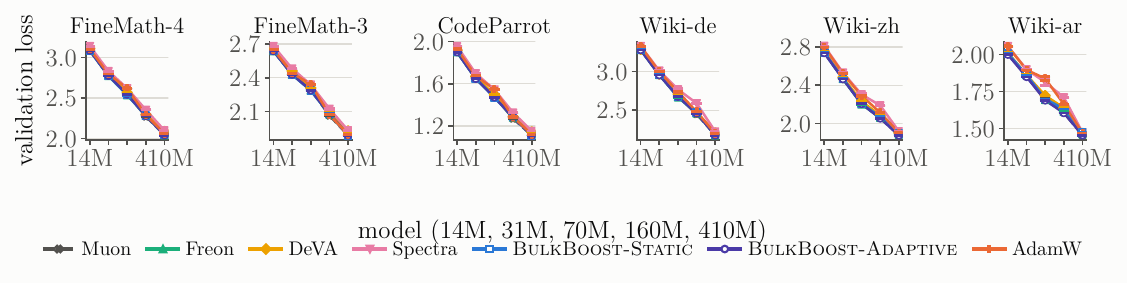}
\caption{Final held-out validation loss as a function of model size, with one panel per corpus and all seven optimization methods shown. Across all corpora, final loss decreases consistently as model size increases, matching the expected scaling trend. The rate of decrease differs by corpus, reflecting differences in adaptation difficulty, but the overall monotonic improvement with scale remains stable. The \algostatic{} curve uses $p=20\%$ (Table~\ref{tab:fracsweep}); the
other curves match Table~\ref{tab:main}.}
\label{fig:bycorpus}
\end{figure}

\clearpage
\subsection{Additional learning curves}
\label{app:curves}
\label{sec:exp:curvegrid}

\paragraph{Final loss against model size.}
Figure~\ref{fig:bycorpus} plots final loss against model size, with one
panel per corpus and all seven methods shown on an absolute loss axis.
The curves match Table~\ref{tab:main}, except that \algostatic{} is
shown at $p=20\%$ (Table~\ref{tab:fracsweep}).
This view makes the scaling trend clear: final loss decreases steadily
with model size on every corpus.

Figures~\ref{fig:grid}--\ref{fig:grid-wiki_ar} show learning curves
for the five remaining corpora---FineMath-4, FineMath-3, CodeParrot,
Wiki-de, and Wiki-ar.
Those figures use the same layout as Figure~\ref{fig:grid-wiki_zh}
(Wiki-zh).
All fixed-rank curves use $p=1.5\%$.
\begin{figure}[p]
\centering
\includegraphics[width=\linewidth]{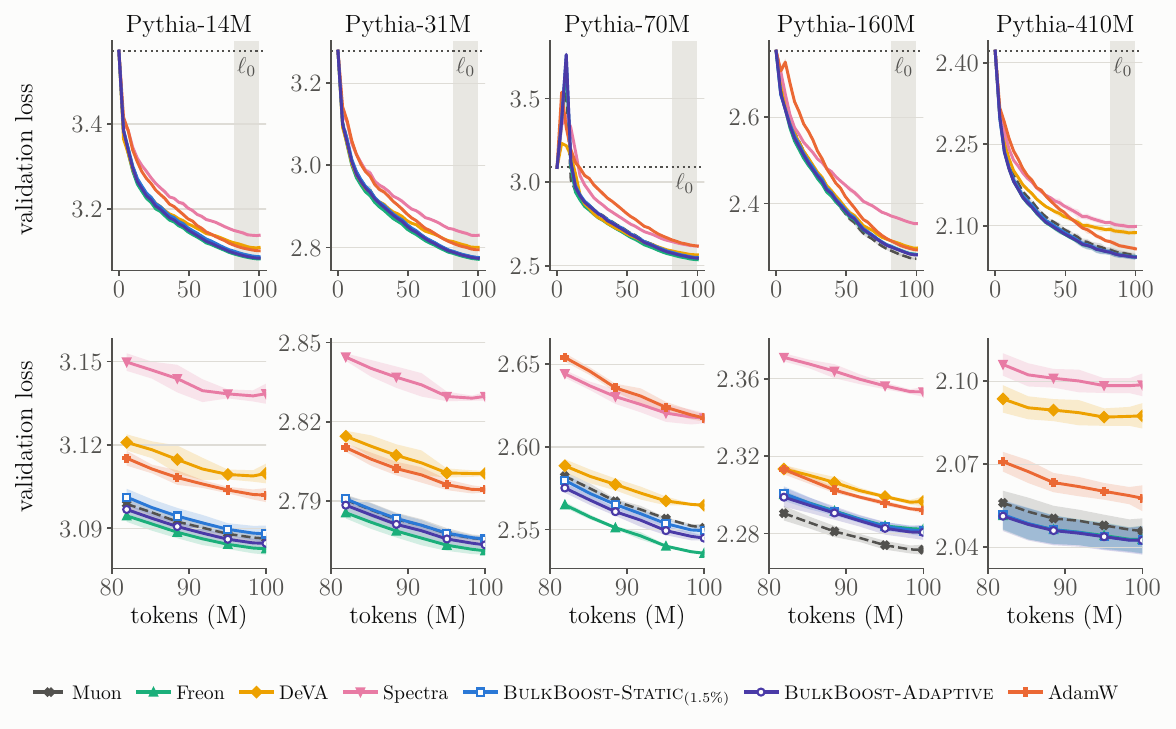}
\caption{Held-out loss vs.\ tokens on \FineMath{}, one column per scale. Top: the
full run, with the pre-adaptation loss $\ell_0$ dotted and the zoom window shaded.
Bottom: the shaded window, rescaled. Note the $y$-ranges differ per panel.}
\label{fig:grid}
\end{figure}

\begin{figure}[p]
\centering
\includegraphics[width=\linewidth]{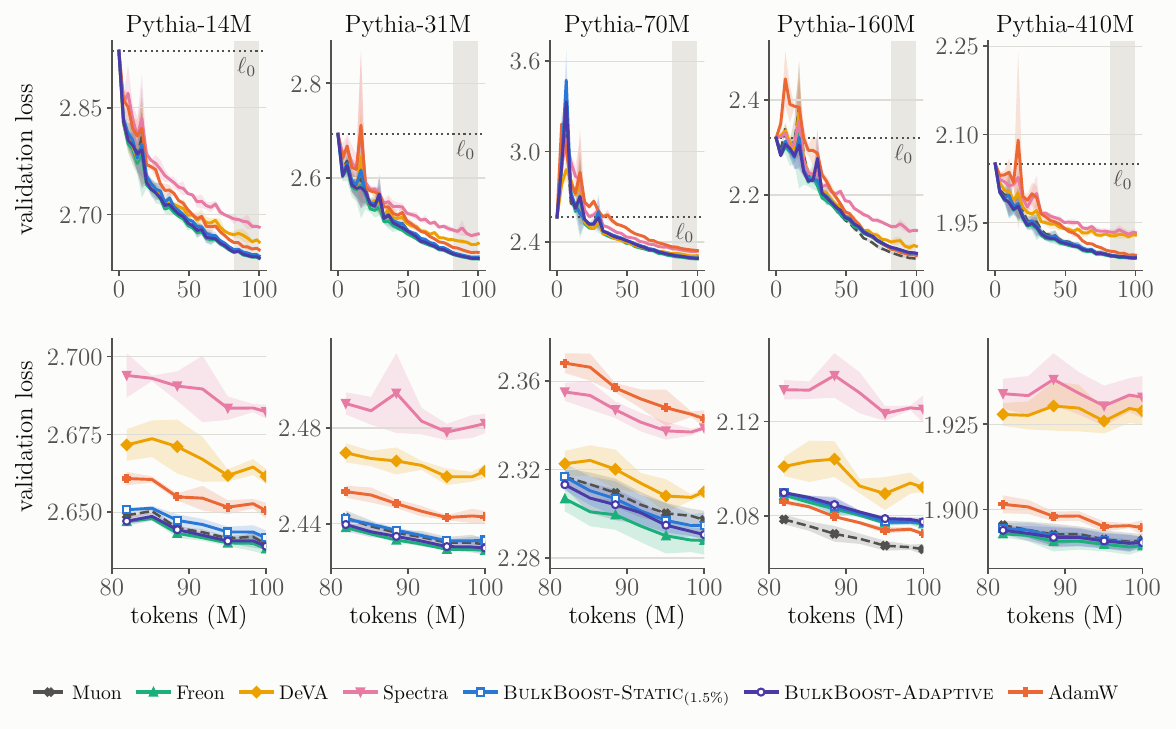}
\caption{Held-out loss vs.\ tokens on \FineMathN{}, one column per scale. Top: the
full run, with the pre-adaptation loss $\ell_0$ dotted and the zoom window shaded.
Bottom: the shaded window, rescaled. Note the $y$-ranges differ per panel.}
\label{fig:grid-finemath3}
\end{figure}

\begin{figure}[p]
\centering
\includegraphics[width=\linewidth]{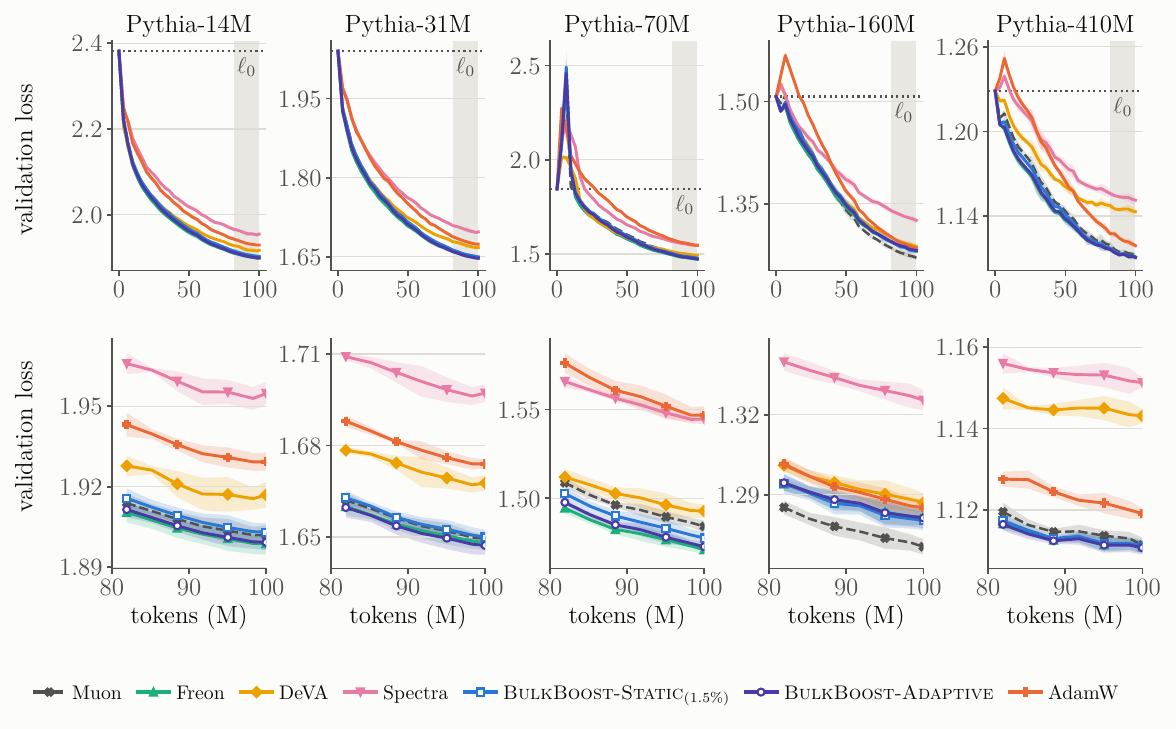}
\caption{Held-out loss vs.\ tokens on \CodeParrot{}, one column per scale. Top: the
full run, with the pre-adaptation loss $\ell_0$ dotted and the zoom window shaded.
Bottom: the shaded window, rescaled. Note the $y$-ranges differ per panel.}
\label{fig:grid-code_py}
\end{figure}

\begin{figure}[p]
\centering
\includegraphics[width=\linewidth]{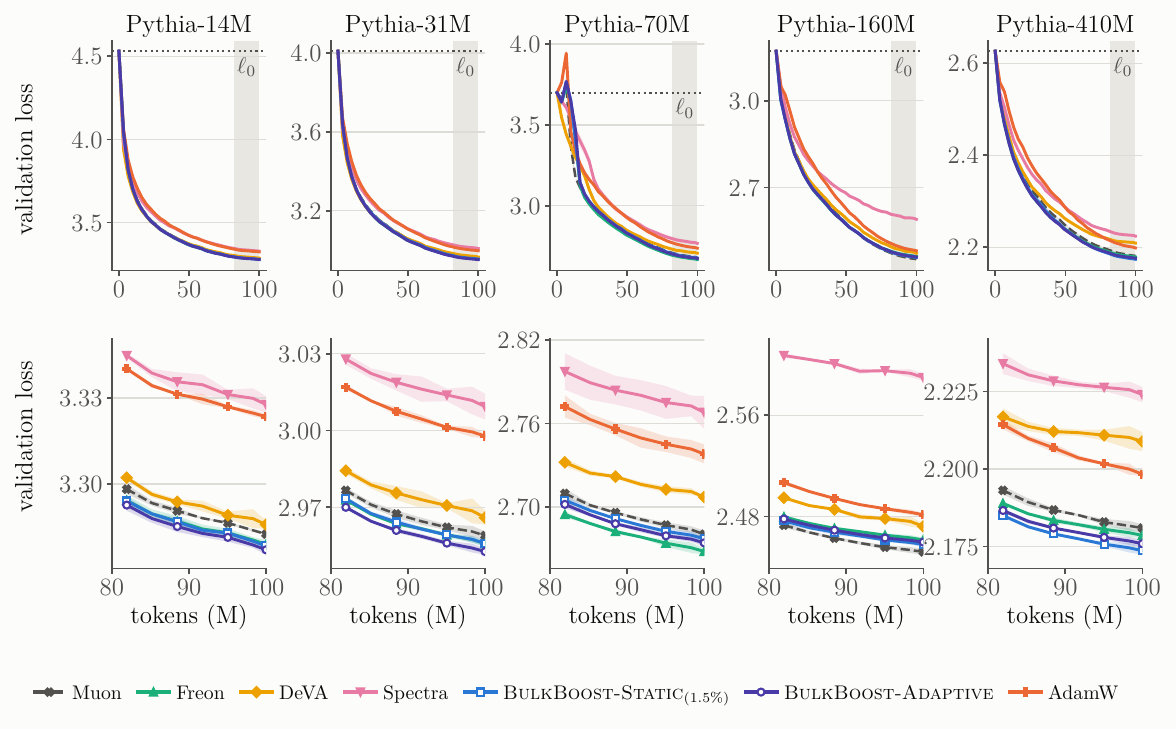}
\caption{Held-out loss vs.\ tokens on \WikiDe{}, one column per scale. Top: the
full run, with the pre-adaptation loss $\ell_0$ dotted and the zoom window shaded.
Bottom: the shaded window, rescaled. Note the $y$-ranges differ per panel.}
\label{fig:grid-wiki_de}
\end{figure}

\begin{figure}[p]
\centering
\includegraphics[width=\linewidth]{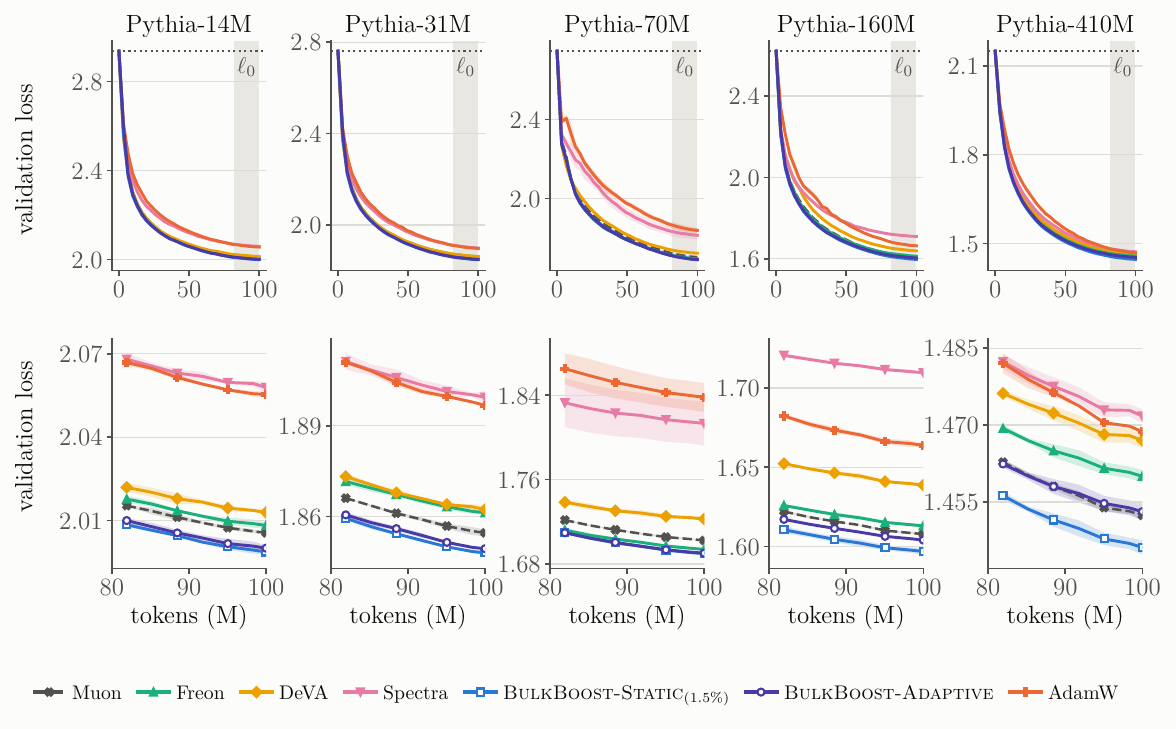}
\caption{Held-out loss vs.\ tokens on \WikiAr{}, one column per scale. Top: the
full run, with the pre-adaptation loss $\ell_0$ dotted and the zoom window shaded.
Bottom: the shaded window, rescaled. Note the $y$-ranges differ per panel.}
\label{fig:grid-wiki_ar}
\end{figure}

\end{document}